%% file: main.tex
\documentclass{article}

 \usepackage[preprint]{neurips_2026}

\usepackage[utf8]{inputenc} 
\usepackage[T1]{fontenc}    
\usepackage{hyperref}       
\usepackage{url}            
\usepackage{booktabs}       
\usepackage{amsfonts}       
\usepackage{nicefrac}       
\usepackage{microtype}      
\usepackage{xcolor}         
\usepackage{multirow}
\usepackage{multicol}
\usepackage{kotex}
\usepackage{amsmath}
\usepackage{amssymb}
\usepackage{graphicx} 
\usepackage{subcaption}
\usepackage{wrapfig}
\usepackage{relsize}
\usepackage{amsthm}
\usepackage{mathtools}
\definecolor{tabfirst}{rgb}{1, 0.7, 0.7}
\definecolor{tabsecond}{rgb}{1, 0.85, 0.7}
\definecolor{tabthird}{rgb}{1, 1, 0.7}
\definecolor{lightyellow}{RGB}{255, 255, 204}
\definecolor{softgray}{RGB}{230, 230, 230}
\newtheorem{theorem}{Theorem}
\newtheorem{proposition}{Proposition}
\newtheorem{lemma}{Lemma}
\newtheorem{definition}{Definition}
\usepackage[dvipsnames]{xcolor}
\usepackage[x11names]{xcolor}
\usepackage[table,xcdraw]{xcolor}
\definecolor{babyblue}{rgb}{0.32, 0.52, 1.0}
\usepackage{amsthm}     
\usepackage{enumitem}   
\usepackage{pifont}

\usepackage{placeins}
\definecolor{blue}{rgb}{0.0, 0.0, 1.0}

\usepackage{booktabs}
\usepackage{multirow}
\usepackage{longtable}
\usepackage{array}
\usepackage{xltabular}
\usepackage{makecell}
\newcolumntype{Y}{>{\raggedright\arraybackslash}X}

\newcommand\blfootnote[1]{%
  \begingroup
  \renewcommand\thefootnote{}%
  \NoHyper\footnote{#1}\endNoHyper%
  \addtocounter{footnote}{-1}%
  \endgroup
}
\title{Strong Helps Weak: Directional Cross-Modal Alignment Transfer in Multi-modal LLMs}

\author{%
\textbf{Hoigi Seo}$^{1*}$ \qquad \textbf{Byung Hyun Lee}$^{1*}$ \qquad \textbf{Minjun Kim}$^{1*}$\\
\textbf{Dohyun Mah}$^1$ \qquad \textbf{Jongho Lee}$^{1,2}$ \qquad \textbf{Se Young Chun}$^{1,2,3\dag}$\\
$^1$Dept.\ of ECE \quad $^2$INMC \& AIIS \quad $^3$IPAI, Seoul National University, Republic of Korea\\
{\small\texttt{\{seohoiki3215, ldlqudgus756, minjoony, dohyun0505\}}}\\[-1pt]
{\small\texttt{\{jonghoyi, sychun\}@snu.ac.kr}}
}

\begin{document}

\maketitle

\begin{abstract}
Multi-modal large language models (MLLMs) achieve strong modality understanding by pairing a large language model (LLM) with an encoder for a target modality such as vision, video, or audio. However, improving an MLLM's capability for a given modality typically requires additional training on large modality-specific datasets, incurring substantial data collection and compute costs. Model merging offers an alternative, but it is often infeasible for data-scarce, large per-sample size, or domain-specific modalities (\textit{e.g.}, audio and video), where same-modality model variants are rarely available. In this work, we characterize an intriguing asymmetric phenomenon: merging a well-aligned, data-rich source-modality MLLM into a data-scarce target-modality MLLM substantially improves the target on its own benchmarks. Our theoretical and empirical analyses show that this gain stems from enhanced alignment between modality-specific and textual tokens, induced by the stronger donor modality. Specifically, we derive a mutual-information lower bound that is monotonic in alignment-related quantities and strongly correlated with downstream MLLM performance. Building on this principle, we propose Directional Cross-modal Alignment Transfer (DCAT), a novel framework that transfers textual alignment from a strong, well-aligned source (donor) modality to a weak target (recipient) modality, boosting target-modality performance without further fine-tuning. We further show that the alignment-enhancing objective admits a closed-form weight-space solution computed from only a small calibration set. DCAT outperforms existing model-merging methods, offering an efficient path toward cross-modal alignment transfer. Project page with code is available at \url{https://seohoiki3215.github.io/DCAT_project_page}
\end{abstract}

\section{Introduction}
\blfootnote{* equal contribution, $\dagger$ corresponding author.}
Large language models (LLMs)~\cite{touvron2023llama, vaswani2017attention, liu2024deepseek, radford2019language, bai2023qwen} achieve strong performance across textual tasks through scaling in data and model size.
Most existing multi-modal LLMs (MLLMs)~\cite{liu2023visual, chu2023qwen, xu2024pointllm} have advanced visual reasoning, video understanding, or audio recognition by pairing a strong LLM backbone with a single modality encoder such as vision~\cite{radford2021learning, tschannen2025siglip, zhai2023sigmoid}, video~\cite{maaz2024video, lin2024video, li2024llava} or audio~\cite{radford2023robust, shi2025sam, schneider2019wav2vec} encoders, and a projector~\cite{li2023blip, liu2023visual} that aligns the modality with the textual space. However, enhancing the target-modality capability of an MLLM typically requires additional training on large modality-specific data, incurring substantial data collection and compute costs.

Model merging~\cite{ilharco2022editing, yadav2023ties, marczak2025no, gargiulo2025task, cheng2025whoever, wei2025optmerge} offers a promising alternative for improving MLLMs by combining same-architecture models with little data.
For vision MLLMs, multiple same-backbone variants and abundant modality data are often available and can be merged to enhance visual capabilities.
However, such data abundance and same-modality variants are much scarcer for other domain-specific modalities, where training multiple variants is costly or infeasible. While existing studies have mostly focused on merging within a single modality, merging different-modality MLLMs remains largely unexplored. We argue that MLLMs trained for different modalities may still share the same base LLM, making cross-modality weight merging possible. Thus, we ask a more targeted question: can a source-modality MLLM (donor) transfer useful alignment to a target-modality MLLM (recipient), thereby improving the recipient's target-modality performance without additional training?

\begin{wrapfigure}{r}{0.4\linewidth}
  \centering
  \includegraphics[width=\linewidth]{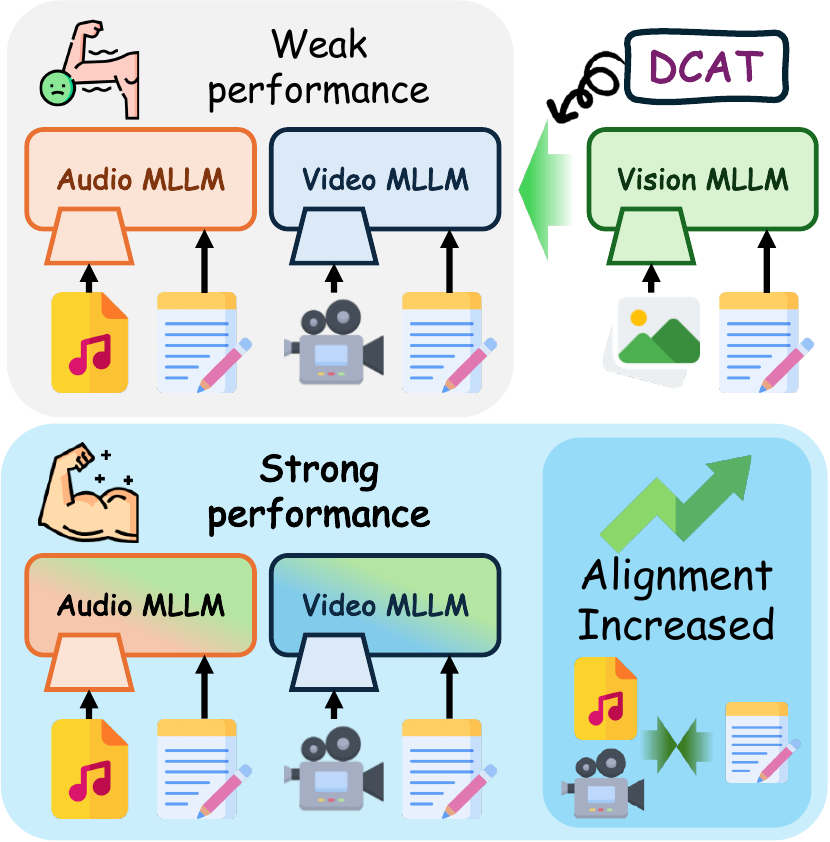}
    \caption{\textbf{Conceptual diagram of DCAT.} DCAT transfers alignment from a strong donor modality model to a relatively weak recipient modality model, thereby improving the recipient model’s modality-specific performance.}
  \label{fig:intro}
  \vspace{-1.5em}
\end{wrapfigure}
In this paper, we identify and study an intriguing and recurring phenomenon: merging a donor and a recipient via existing merging methods can yield substantial performance gains over the original recipient on its associated benchmark. We investigate this phenomenon along two complementary axes. On the theoretical side, we analyze it through the lens of mutual information (MI)~\cite{kreer1957question, shannon1948mathematical}. Under a mild assumption, we prove that a lower bound on the MI between text and the other modality is monotonic in the product of two quantities: the Subspace Energy Overlap (SEO), which measures the extent to which each modality's token energy lies within the text-aligned subspace, and the Spectral Diversity (SD)~\cite{roy2007effective}, which measures how uniformly that energy is distributed within this subspace. We refer to this product as the \textit{Alignment} score. On the empirical side, across hundreds of model checkpoints spanning multiple architectures, we conduct a series of statistically tested analyses and find that the Alignment score is strongly and significantly correlated with MLLM performance, and that merged models exhibiting performance gains also show a statistically significant increase in the Alignment. Taken together, these theoretical and empirical results support our MI-lower-bound-based Alignment score, as a plausible explanatory mechanism for the observed benefits of model merging.

Building on these analyses, we propose Directional Cross-modal Alignment Transfer (DCAT), a novel framework that transfers alignment from a high-alignment donor MLLM to a recipient MLLM using only a few hundred calibration samples, thereby improving the recipient's target-modality performance. To our knowledge, we are the first to formulate and analyze this phenomenon; accordingly, the most natural baselines for DCAT are the existing model-merging methods through which the phenomenon is initially observed. As illustrated in Fig.~\ref{fig:intro}, we instantiate this directional transfer by using vision MLLMs, which generally exhibit high alignment, as donors and audio or video MLLMs as recipients. Across a diverse suite of benchmarks, including MMAU~\cite{sakshi2024mmau}, AIR-Bench~\cite{yang2024air}, Video-MME~\cite{fu2025video}, and MVBench~\cite{li2024MVBench}, DCAT achieves 29.55\% improvement for audio tasks and 24.92\% for video tasks on average over the original recipient models.

Our contributions can be summarized as follows:
\begin{itemize}
    \item We investigate an asymmetric phenomenon: model merging between MLLMs trained on different modalities improves target-modality performance. We introduce a theory-motivated alignment metric that explains this phenomenon, supported by strong empirical analysis. 
    \item Building on this analysis, we propose Directional Cross-modal Alignment Transfer (DCAT), a method that leverages a small amount of data to perform a closed-form weight update, thereby transferring alignment from a well-aligned modality to a relatively weaker one.
    \item We conduct extensive experiments across multiple MLLMs on various benchmarks. DCAT consistently outperforms existing model merging methods and achieves average performance gains of 29.55\% over the base model on audio tasks and 24.92\% on video tasks.
\end{itemize}

\section{Related Work}
\paragraph{Multi-modal large language models.}
Recent advances in LLMs have spurred extensive work on MLLMs, which attach modality-specific encoders and projectors to an LLM to generate text from inputs such as vision~\cite{liu2023visual, seo2025epistemic, li2025llava, an2025llava, zhu2024minigpt, chen2023shikra, maaz2024video, li2024llava, lin2024video, wang2024qwen2, bai2025qwen3} and audio~\cite{hu2024wavllm, chu2023qwen, chu2024qwen2, xu2025qwen3, huang2025step, hurst2024gpt}. Training these models is compute-intensive, often requiring optimization of the LLM together with newly added components. Moreover, modality-specific data, such as video, requires far more storage than text and is costly or difficult to collect, making modality-specific performance gains expensive. These challenges are especially severe for rare or high-volume modalities with limited supervision, motivating training-free and data-efficient alternatives for improving MLLM capabilities.

\paragraph{Model merging.}
As neural networks scale, deploying separate task specialists becomes prohibitive, motivating model merging as a training-free way to consolidate independently fine-tuned models.
TA~\cite{ilharco2022editing} enables model editing via task vectors, defined as weight differences between fine-tuned and pretrained models.
Subsequent methods mitigate interference among task-specific updates: TIES~\cite{yadav2023ties} trims redundant changes and resolves sign conflicts; ISO~\cite{marczak2025no} separates task-specific subspaces; TSV~\cite{gargiulo2025task} decomposes updates into singular directions for low-dimensional merging; and WUDI~\cite{cheng2025whoever} identifies interfering vectors for robust merging.
OptMerge~\cite{wei2025optmerge} extends this paradigm across heterogeneous tasks and modalities. However, existing studies have not systematically examined cross-modal merging from a directional transfer perspective, where a source-modality MLLM is used to improve a designated target-modality MLLM on its own modality-specific benchmarks.

\section{Analysis}
\begin{table*}[t]
    \centering
    \caption{\textbf{Quantitative comparison between modality-specific models and merged models.} We report \textbf{(Left)} the MMAU~\cite{sakshi2024mmau} performance obtained by merging an audio-only model with an audio-vision merged model (A+V), and \textbf{(Right)} the Video-MME~\cite{fu2025video} performance obtained by merging a video-only model with a video-vision merged model (Vd+V). Results are reported for Task Arithmetic (TA), ISO-C (ISO), TSV, and OptMerge (Opt.). In all cases, merging is applied only to the LLM and the Alignment score was measured at the 21st layer of the Vicuna model. We set $\gamma=0.8$.}
    \label{tab:motivation}
    \setlength{\tabcolsep}{2.5pt}
    \renewcommand{\arraystretch}{1.05}
    \noindent
    \begin{minipage}[t]{0.49\textwidth}
        \centering
        \small
        \begin{tabular}{@{}lccccc@{}}
            \toprule
            & Music & Sound & Speech & Overall & Align$_\gamma$ \\
            \midrule
            Audio only & 54.79 & 60.66 & 31.23 & 48.90 & 0.0091 \\
            \midrule
            A+V (TA)   & 55.69 & 64.56 & 39.34 & 53.20 & 0.0103 \\
            A+V (ISO)  & 58.68 & 66.67 & 34.83 & 53.40 & 0.0950 \\
            A+V (TSV)  & 55.99 & 66.67 & 35.44 & 52.70 & 0.0100 \\
            A+V (Opt.) & 59.58 & 64.26 & 35.44 & 53.10 & 0.0099 \\
            \bottomrule
        \end{tabular}
    \end{minipage}%
    \hfill
    \begin{minipage}[t]{0.49\textwidth}
        \centering
        \small
        \begin{tabular}{@{}lccccc@{}}
            \toprule
            & Short & Medium & Long & Overall & Align$_\gamma$ \\
            \midrule
            Video only & 39.89 & 34.22 & 30.89 & 35.00 & 0.0248 \\
            \midrule
            Vd+V (TA)   & 44.67 & 34.78 & 31.89 & 37.11 & 0.0287 \\
            Vd+V (ISO)  & 43.67 & 38.22 & 33.44 & 38.44 & 0.0269 \\
            Vd+V (TSV)  & 46.44 & 36.67 & 34.00 & 39.04 & 0.0278 \\
            Vd+V (Opt.) & 44.33 & 37.00 & 33.56 & 38.30 & 0.0267 \\
            \bottomrule
        \end{tabular}
    \end{minipage}
\end{table*}
\subsection{Model merging increases modality performance}
While task-wise model merging, which combines LLMs fine-tuned on different downstream tasks into one model, has been extensively studied, comparatively little attention has been paid to merging MLLMs trained for different input modalities.
Tab.~\ref{tab:motivation} reports MMAU~\cite{sakshi2024mmau} scores for vision--audio MLLM merges and Video-MME~\cite{fu2025video} scores for vision--video MLLM merges. All checkpoints are from OptMerge~\cite{wei2025optmerge} which are Vicuna~\cite{chiang2023vicuna}-based.
Despite modality heterogeneity, merging with a vision MLLM improves audio and video benchmark performance by up to 18.45\% over the original modality-specific MLLM.
This raises a natural question: what drives this striking performance gain?

\subsection{On the lower bound of Mutual Information}
We interpret this phenomenon through the lens of the mutual information (MI) between text and the other modality since it provides a principled way to quantify how much modality-specific information is accessible from text representations. As a first step toward making this perspective concrete, we introduce a quantity that we call the Subspace Energy Overlap (SEO), defined as follows.
\begin{definition}[Subspace Energy Overlap]
Let $Y_{\mathrm{text}}\in\mathbb{R}^{N_t\times d}$ denote the text-token activation matrix, and let $U_k\in\mathbb{R}^{k\times d}$ be the top-$k$ right singular vectors of the centered text activations. For modality-token activations $Y_{\mathrm{mod}}=\{y_i\}_{i=1}^{N_m}$, the \emph{Subspace Energy Overlap} (SEO) is defined as
\begin{equation}
\text{SEO}
= \frac{ \|Y_{\mathrm{mod}} U_k^\top\|_F^2}{ \|Y_{\mathrm{mod}}\|_F^2}.
\end{equation}
\end{definition}
SEO quantifies the fraction of modality-token energy captured by the principal text subspace, with higher values indicating that modality representations are more concentrated within this subspace.

Next, we introduce a second quantity, which we call the Spectral Diversity (SD), defined as follows.
\begin{definition}[Spectral Diversity]
Let $Z=Y_{\mathrm{mod}}U_k^\top \in\mathbb{R}^{N_m\times k}$
and $\Sigma_Z$ be the empirical covariance matrix of the projected representations. Denote the eigenvalues of $\Sigma_Z$ by $\{\lambda_j\}_{j=1}^k$, and define $p_j=\frac{\lambda_j}{\sum_{\ell=1}^{k}\lambda_\ell}.$
The \emph{Spectral Diversity} (SD) is defined as
\begin{equation}
\text{SD}
=
\frac{1}{k}
\exp\left(
-\sum_{j=1}^{k} p_j\log p_j
\right).
\end{equation}
\end{definition}

Intuitively, $k\cdot \mathrm{SD}$ corresponds to the effective rank~\cite{roy2007effective} in text-subspace for the projected modality representation. Thus, the effective-rank view makes SD directly interpretable in the text-subspace: higher SD indicates that modality information is broadly distributed within that subspace. Then, the following theorem lower-bounds the mutual information between modality tokens and text.

\begin{theorem}[Weighted Alignment-MI Lower Bound]\label{thm:lower_bound}
    Let $z = y U_k^\top$ where $U_k$ contains the top-$k$ right singular vectors of the centered text activation matrix. Suppose that $z\sim\mathcal{N}(\mu_Z,\Sigma_Z)$, $z|x \sim \mathcal{N}(\mu_z(x), \Sigma_{z|x})$, and      $p_j \geq \frac{\delta}{k}$
     for some sufficiently large $ \delta\in(0,1]$. Then, for any $\gamma \in (0,1)$:
    \begin{equation}
    \label{eq:mi_lower_bound}
    I(y; x) \;\geq\; \kappa(\gamma) \cdot \log\!\left(\text{SEO}^\gamma \cdot \text{SD}^{1-\gamma}\right) + C_{\delta},
    \quad \kappa(\gamma) = \frac{k}{2\min(\gamma, 1-\gamma)}
    \end{equation}
    with $C_{\delta} = \frac{k}{2}\log(\text{tr}(\Sigma_y)/k) - \frac{1}{2}\log\det(\bar\Sigma_{z|x}) +
    \frac{k}{2}\log\delta$.
\end{theorem}
The proof of Theorem~\ref{thm:lower_bound} is in Sec.~\ref{sec:append_thm_proof} and we empirically verify that $\delta$ is not vanishingly small in Sec.~\ref{sec:empirical_validation_tau}. 
Although the lower bound in this theorem is maximized at $\gamma=0.5$, the most suitable value of $\gamma$ may vary across models and modalities; we therefore consider a more flexible, looser bound.

\begin{figure}[t]
    \centering
    \includegraphics[width=\textwidth]{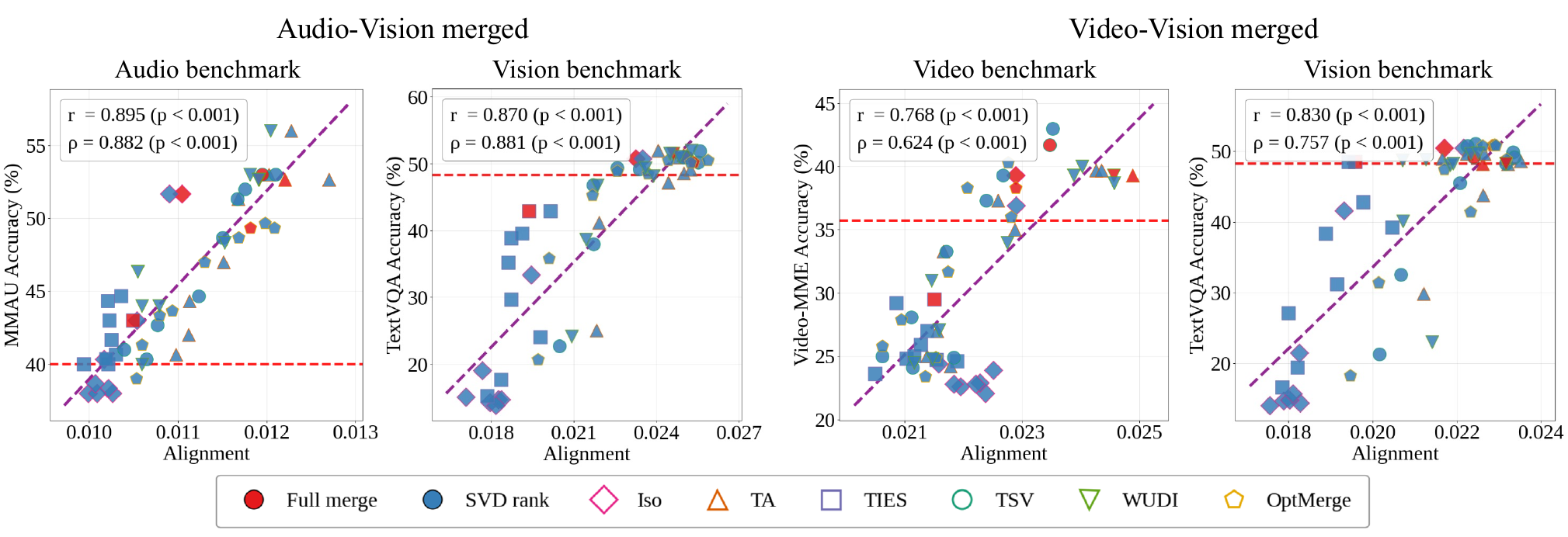}
    \caption{\textbf{Empirical analysis of our alignment metric and performance.} We analyze the correlation between the Alignment score and modality performance across a total of $N=114$ model checkpoints. We use MMAU~\cite{sakshi2024mmau} for audio, Video-MME~\cite{fu2025video} for video, and TextVQA~\cite{singh2019towards} for vision performance. Checkpoints are obtained from a range of merging methods~\cite{ilharco2022editing, yadav2023ties, marczak2025no, gargiulo2025task, cheng2025whoever, wei2025optmerge}, combined with several variations of the merged weights. The red dotted line refers to the performance of original model. For all experiments, we chose $k=32$ and $\gamma=0.8$, and measure performance on 300 questions per benchmark across categories. Text embeddings are extracted from the questions.}
    \label{fig:correlation}
\end{figure}

It is widely recognized that increasing the MI between modalities is essential for improving performance in multi-modal models~\cite{su2023towards, almudevar2025aligning, oord2018representation}. This implies that maximizing the lower bound established in Theorem~\ref{thm:lower_bound} provides a principled means of enhancing performance. 
Motivated by the bound in Eq.~\ref{eq:mi_lower_bound}, we define the Alignment score as a weighted geometric mean of SEO and SD:
\begin{equation}
    \mathrm{Align}_{\gamma}
    :=
    \mathrm{SEO}^{\gamma}\cdot \mathrm{SD}^{1-\gamma}.
\end{equation}
That is, it measures how much modality-token energy enters the text principal subspace (SEO) and how broadly that energy is distributed across the effective directions of the text subspace (SD).
In the following section, we leverage the proposed $\text{Align}_\gamma$ to analyze the effect of various weight choices and present a statistically grounded empirical study examining whether the performance gains from model merging can be attributed to an increase in the MI lower bound for each modality.

\subsection{Empirical analysis on Alignment score}
We measure the $\text{Align}_\gamma$ for each model before and after merging, with the Align. column of Tab.~\ref{tab:motivation}. The alignment increases significantly under model merging ($p\ll10^{-10}$, Wilcoxon signed-rank test), suggesting that merging may improve performance by raising the MI lower bound.

To test whether this relationship is global, we evaluate 114 Vicuna-based MLLM checkpoints generated by multiple merging methods, low-rank approximations of weight differences, and module-restricted merging (\textit{e.g.}, attention or FFN).
For each checkpoint, we plot modality-wise performance against $\text{Align}_\gamma$ at the 21st layer (Fig.~\ref{fig:correlation}), where Pearson ($r$) and Spearman ($\rho$) coefficients show significant ($p<10^{-3}$) strong positive correlations across all modalities.
Audio and video models improve clearly, while the vision model shows only marginal improvement or degradation; thus, we refer to relatively degraded models as \textit{donor models} and improved models as \textit{recipient models}.

\begin{wrapfigure}{r}{0.5\linewidth}
  \centering
  \vspace{-1.0em}
  \includegraphics[width=\linewidth]{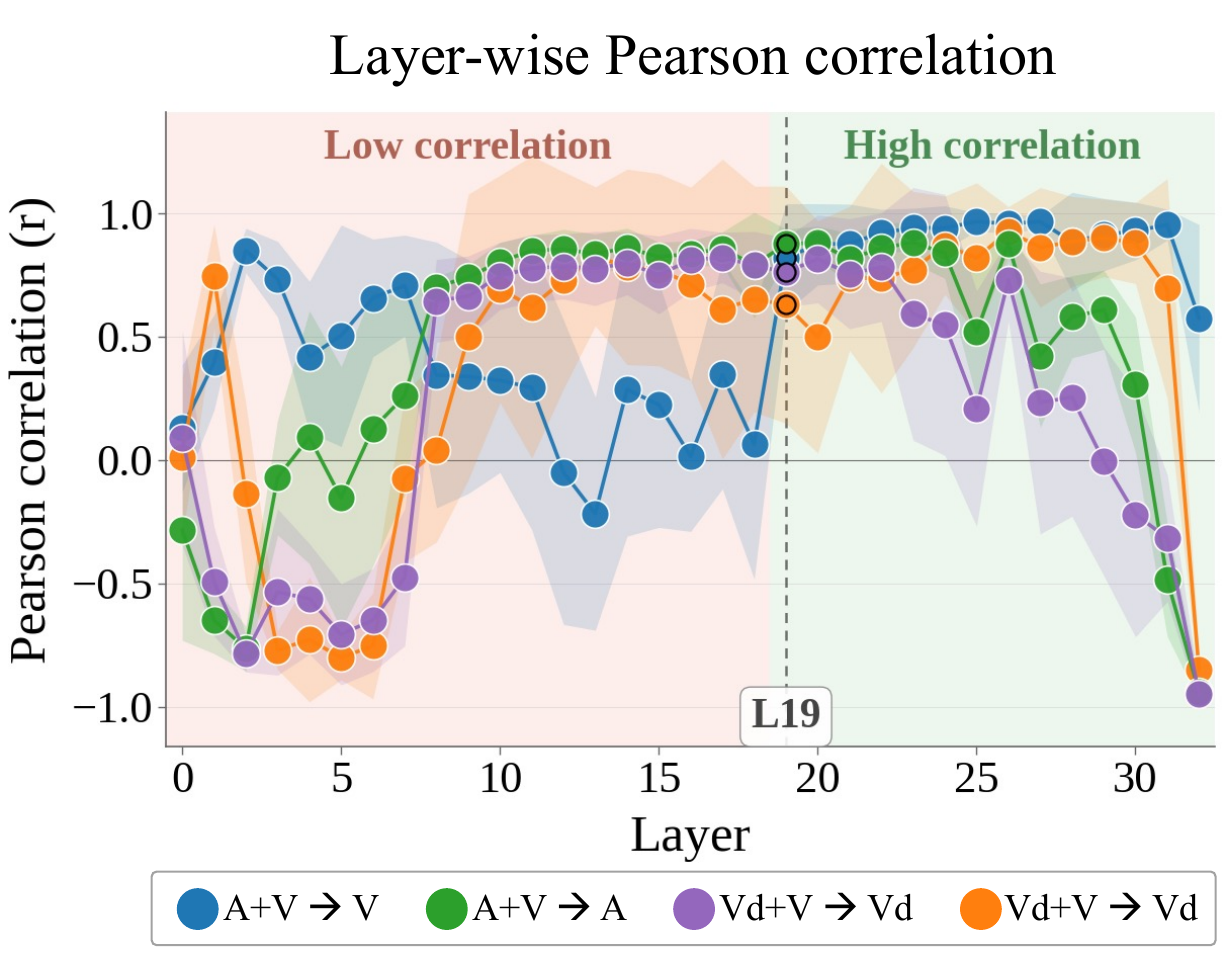}
    \caption{\textbf{Layer-wise Pearson correlation between Alignment score and performance across MLLM combinations and modalities.} We measure alignment in a layer-wise manner for each modality across multiple modality pairs. A consistently strong positive correlation between the Alignment score and performance emerges from the middle layers onward.}
  \label{fig:layerwise_alignment}
  \vspace{-1.5em}
\end{wrapfigure}
We selected the 21st layer, the mid-to-late layer, to see this trend. In Fig.~\ref{fig:layerwise_alignment}, we plot the Pearson coefficient between $\text{Align}_\gamma$ and performance at each layer. Beyond the 19th layer, we consistently observe a high correlation across all modality merge pairs and performance. This finding is consistent with prior works~\cite{venhoff2025visual, wu2026vision, jiang2025devils, yoon2025visual, hsu2026anatomy, seo2025epistemic, skean2025layer, ye2025atp} showing that MLLMs exhibit high, meaningful attention from the middle layers onward, which has been interpreted as the point where alignment between text and the corresponding modality begins. In other words, the phenomenon captured here parallels observations made in earlier studies using different metrics, suggesting that our $\text{Align}_\gamma$ metric not only functions as a lower bound on MI but also quantifies the alignment between textual and non-textual modalities. To ensure robustness, we adopt the 21st layer in our analysis, as it exhibits the lowest standard deviation. Experimental details and broader analysis of alignment across different models and benchmarks are provided in the Appendix Sec.~\ref{sec:append_details} and Sec.~\ref{sec:append_analysis}.

\section{Method}
Building on these analyses, we propose \textbf{Directional Cross-modal Alignment Transfer (DCAT)}, which transfers donor alignment to recipient models via a closed-form solution using minimal calibration inputs. DCAT exploits cross-modal merging benefits by targeting the alignment-related terms in Eq.~\eqref{eq:mi_lower_bound}. We further incorporate $C_{\delta}$ to preserve the overall modality activation scale.
\paragraph{Notation.}
Specifically, let $\theta_0$ be the pretrained base weight (\textit{e.g.}, base LLM weight) for a given layer, and let $\tau_r=\theta_r-\theta_0$ and $\tau_d=\theta_d-\theta_0$ denote the recipient and donor task vectors.
We denote recipient-modality calibration inputs and text calibration inputs as $X_{\mathrm{mod}}$ and $X_{\mathrm{text}}$, respectively.
For a candidate task vector $\tau$, the corresponding modality-token and text-token activations are given by $X_{\mathrm{mod}}(\theta_0+\tau)$ and $X_{\mathrm{text}}(\theta_0+\tau)$.
Then recipient activations are computed as:
\begin{equation}
Y_{\mathrm{text}} = X_{\mathrm{text}}(\theta_0+\tau_r), \quad 
Y_{\mathrm{mod}} = X_{\mathrm{mod}}(\theta_0+\tau_r).
\end{equation}
We also define the orthogonal projector of $U_k$ and its orthogonal complement projector as $P_k (=U_k^\top U_k )$ and $P_k^\perp (=I-P_k)$, respectively.
Here, $U_k$ represents the principal text subspace, so $P_k$ and $P_k^\perp$ decompose an activation into the text-subspace component and its residual component.

\paragraph{Cross-modal alignment improvement.}
Motivated by the former analysis, DCAT first introduces a donor-guided cross-modal alignment term on text inputs:
\begin{equation}
\label{eq:text_regularization}
\mathcal{L}_{\mathrm{cross}}(\tau) = \beta\cdot\|X_{\mathrm{text}} (\tau - \tau_r) \|_F^2 + (1-\beta)\cdot\|X_{\mathrm{text}} (\tau - \tau_d) \|_F^2 .
\end{equation}
The recipient anchor preserves the recipient model’s original behavior and prevents the update from overwriting modality-specific capabilities.
The donor anchor is complementary: from the perspective of linear mode connectivity~\cite{frankle2020linear,wu2023pi}, donor weights on text inputs define directions along which DCAT can improve modality-text alignment while remaining in a text-compatible region of weight space.
Thus, the donor-side term does not merely pull the solution toward the donor model; it transfers donor-induced cross-modal alignment directions through text space, allowing the recipient model to benefit from the donor while suppressing unexpected drift of $C_{\delta}$.

\paragraph{Activation-space target construction.}
Beyond donor-guided cross-modal alignment transfer, DCAT further improves the recipient model's self-alignment by constructing an activation-space target for the recipient modality.
Since $U_k$ defines the principal text subspace, we project the recipient representation as $X_{\mathrm{mod}}(\theta_0+\tau_r) U_k^T$.
We compute its SVD as $A\;\mathrm{diag}(\sigma_1,\ldots,\sigma_k) B^\top$ and define $\rho=\|X_{\mathrm{mod}}(\theta_0 + \tau_r )\|_F/\sqrt{k}$, which preserves the total output energy across all $k$ text-subspace directions.
DCAT then reshapes the original spectrum into a more balanced projected target $T_\alpha$:
\begin{equation}
\label{eq:projected_target}
T_\alpha = c_\alpha \cdot A\;\mathrm{diag}\left(\sigma_j^\alpha \rho^{1-\alpha}\right)_{j=1}^{k}B^\top U_k,
\qquad
c_\alpha = \frac{\|X_{\mathrm{mod}}(\theta_0+\tau_r) \|_F}{\left(\sum_{j=1}^{k}\sigma_j^{2\alpha}\rho^{2(1-\alpha)}\right)^{1/2}}.
\end{equation}
Note that $T_\alpha$ reflects favorable properties for increasing the lower bound in Theorem~\ref{thm:lower_bound} such that:
\begin{equation}
\label{eq:target_properties}
T_\alpha P_k^\perp=0,
\qquad
\|T_\alpha\|_F=\|X_{\mathrm{mod}}(\theta_0+\tau_r)\|_F .
\end{equation}
Derivation for Eq.~\ref{eq:target_properties} is provided in Appendix Sec.~\ref{sec:append_properties}.
The first property makes $T_\alpha$ contained in the principal text subspace, while the second prevents the target from changing the scale of the recipient activations.
Thus, $T_\alpha$ improves self-alignment by pulling the recipient modality toward its own text subspace while preserving the recipient's activation energy, which prevents increasing text-subspace overlap by merely shrinking the modality activation, undesirable for the trace component in $C_{\delta}$.

The role of $\alpha$ is to control the trade-off between preserving the recipient representation within the text subspace and increasing spectral diversity, while the target remains norm-preserving for all $\alpha$.
Specifically, keeping the original singular values $\{\sigma_j\}$ would preserve the recipient geometry but may leave the projected representation spectrally concentrated.
On the other hand, replacing them with a fully isotropic spectrum would maximize diversity but may distort the recipient representation.

\paragraph{Unified self-alignment matching.}
Given the target $T_\alpha$, DCAT improves self-alignment by matching the recipient-modality activations to this target:
\begin{align}
\mathcal{L}_{\mathrm{self}}(\tau) 
&= \underbrace{\|X_{\mathrm{mod}}(\theta_0 + \tau) P_k^\perp\|_F^2}_{\text{\textbf{Encouraging SEO}}} 
+ \underbrace{\|X_{\mathrm{mod}}(\theta_0 + \tau)P_k-T_\alpha\|_F^2}_{\text{\textbf{Encouraging SD}}} \nonumber \\
&= \|X_{\mathrm{mod}}(\theta_0 + \tau)-T_\alpha\|_F^2,
\label{eq:alignment_loss}
\end{align}
where the equality follows from the orthogonality of $P_k$ (derived in Appendix Sec.~\ref{sec:append_align}).
Note that the first term reduces activation energy outside the principal text subspace, while the second shapes the projected representation within that subspace.
Therefore, once the target $T_\alpha$ is constructed inside the text subspace with a balanced spectrum, $\mathcal{L}_{\mathrm{self}}(\tau)$ simultaneously encourages {SEO} and {SD}.

\paragraph{Closed-form DCAT update.}
Combining the cross-modal alignment improvement term $\mathcal{L}_{\mathrm{cross}}(\tau)$ and the self-alignment matching term $\mathcal{L}_{\mathrm{self}}(\tau)$, the final objective is:
\begin{equation}
\label{eq:DCAT_objective}
\min_{\tau} \; \mathcal{L}_{\mathrm{self}}(\tau) + \eta \cdot\mathcal{L}_{\mathrm{cross}}(\tau).
\end{equation}
That is, $\mathcal{L}_{\mathrm{cross}}(\tau)$ transfers donor-induced cross-modal alignment directions while preserving the recipient model's original text behavior, and $\mathcal{L}_{\mathrm{self}}(\tau)$ further reshapes the recipient-modality activations toward the norm-preserving target $T_\alpha$.
The donor enters through the cross-modal alignment improvement term, which biases the closed-form solution toward donor-compatible text-aligned directions without overwriting the recipient modality representation.

Since Eq.~\ref{eq:DCAT_objective} is quadratic in $\tau$, DCAT admits a closed-form layer-wise solution without gradient-based optimization (see Appendix Sec.~\ref{sec:append_closed}).
To avoid mismatch from applying all layer-wise solutions at once, since $T_\alpha$ is constructed using pre-update recipient activations, DCAT progressively merges each layer’s solution and recomputes calibration activations for the next layer using the partially updated model.
This reduces accumulated activation drift while preserving closed-form efficiency.

\paragraph{Scheduling $\alpha$.}
The degree of alignment can vary across layers, as well as across the linear projections within each layer. Accordingly, $\alpha$ should adaptively balance SEO and SD by adjusting the spectrum of the target $T_\alpha$ inside the principal text subspace. To this end, we apply a simple yet effective adaptive schedule for $\alpha$ as follows:
\begin{equation}
\label{eq:alpha_schedule}
\alpha = \mathrm{Sigmoid} \left(\frac{\log \mathrm{SEO}_r}{   \log \mathrm{SD}_r + \epsilon }  \right),
\end{equation}

where $\epsilon > 0$ is for stability.
This schedule makes the target more conservative when the recipient has weak SEO, and more isotropic when the projected representation is spectrally concentrated.
Accordingly, larger $\alpha$ preserves more of the recipient's projected spectrum, while smaller $\alpha$ encourages SD.
To support this formulation, we analyze the effective rank of textual and modality-specific token representations across models and modalities; details are provided in Appendix Sec.~\ref{sec:append_analysis}.

\section{Experiment}
\subsection{Experimental settings}\label{sec:experimental_settings}
\paragraph{Baselines.} Although first observed in model merging, DCAT has a distinct goal: transferring stronger donor alignment to a recipient rather than equipping one MLLM with multiple tasks or modalities. Because both settings combine models at the weight level, we compare DCAT with representative merging baselines: Task Arithmetic~\cite{ilharco2022editing} (TA), TIES-merging~\cite{yadav2023ties} (TIES), ISO-C~\cite{marczak2025no} (ISO), Task Singular Vector~\cite{gargiulo2025task} (TSV), WUDI~\cite{cheng2025whoever}, and OptMerge~\cite{wei2025optmerge}. Experiments cover Audio+Vision and Video+Vision settings on two widely used backbones, Vicuna~\cite{chiang2023vicuna} and LLaMA 3.1~\cite{grattafiori2024llama}. For Vicuna, we use the weights released by OptMerge. For LLaMA 3.1, we use in-the-wild models that share the same architecture but are independently trained with different recipes, demonstrating DCAT’s robustness and practical utility. Full baseline details are in the Appendix Sec.~\ref{sec:append_comparison_detail}.
\vspace{-1em}
\paragraph{Benchmarks.} We extensively compare DCAT’s alignment transfer performance across audio and video benchmarks. For audio, we evaluate on MMAU~\cite{sakshi2024mmau} (1K samples), AIR-Bench Foundation~\cite{yang2024air} (25K samples), MuCho-music~\cite{weck2024muchomusic} (1.2K samples), Clotho-AQA~\cite{lipping2022clotho} (1.4K samples), and VocalSound~\cite{gong_vocalsound} (3.6K samples). For video, we evaluate on Video-MME~\cite{fu2025video} (2.7K samples), TempCompass~\cite{liu2024tempcompass} (5.5K samples), MVBench~\cite{li2024MVBench} (4K samples), and MLVU~\cite{zhou2024mlvu} (2.2K samples). Subcategory-level results are provided in the Appendix Sec.~\ref{sec:append_subcategory}.
\vspace{-1em}
\paragraph{Implementation details.} DCAT progressively merges models layer by layer from the first to the last layer. At each layer, our closed-form update is applied to all linear projections, aligning internal representations across computational modules. For calibration,  we use AVQA~\cite{yang2022avqa}, which supports both audio and video: audio tracks calibrate audio MLLMs, and video tracks calibrate video MLLMs, using 256 samples each. Remaining hyperparameters, including the text subspace dimension $k$, cross-modal alignment weight $\eta$, and anchor parameter $\beta$, are detailed in Appendix~\ref{sec:append_implementation}.

\begin{table}[t]
    \caption{\textbf{Quantitative comparison of DCAT.} We evaluate DCAT against various model merging methods across benchmarks, using a vision model as donor and an audio/video model as recipient. \textbf{(Top)} In Audio+Vision, DCAT transfers vision-model alignment to the audio model, achieving gains up to \textbf{100.92\%} (MuCho-Music) over the original. \textbf{(Bottom)} Video+Vision shows a similar trend, with improvements up to \textbf{53.88\%} (MLVU). ($^*$Since our LLaMA 3.1-based audio model does not tune the language backbone, several merging methods reduce to TA, yielding identical outputs.)}

    \label{tab:quant}
    \centering

    \resizebox{\textwidth}{!}{
    \begin{tabular}{ccl|cccc|cccc|ccc|c}
    \toprule
         & & & \multicolumn{4}{c|}{MMAU} & \multicolumn{4}{c|}{AIR-Bench} & \multirow{2}{*}{MuCho} & \multirow{2}{*}{Clotho} & \multirow{2}{*}{Vocal} & \multirow{2}{*}{Average} \\
         & & & Music & Sound & Speech & Overall & Music & Sound & Speech & Overall & & & & \\
    \midrule
    \multirow{8}{*}{\rotatebox{90}{Vicuna-based}} 
    & \multirow{8}{*}{} 
    & Original & 54.79 & 60.66 & 31.23 & 48.90 & 38.44 & 46.76 & 44.62 & 43.17 & 26.96 & 61.23 & 48.98 & 45.85 \\
    \cmidrule{2-15}
    & & TA~\cite{ilharco2022editing}  &  55.69 & 64.56 & \cellcolor{tabthird}39.34 & 53.20 & 45.64 & 57.48 & 47.03 & 49.09 & 50.04 & \cellcolor{tabthird}81.00 & 44.81 & 55.63 \\
    & & TIES~\cite{yadav2023ties} & 54.49 & 62.46 & 36.04 & 51.00 & 44.86 & 55.92 & 47.51 & 48.68 & 38.42 & 74.97 & \cellcolor{tabsecond}50.93 & 52.80\\
    & & ISO~\cite{marczak2025no} &  \cellcolor{tabthird}58.68 & \cellcolor{tabsecond}66.67 & 34.83 & \cellcolor{tabthird}53.40 & 45.82 & \cellcolor{tabfirst}60.77 & 46.72 & 49.79 & \cellcolor{tabsecond}52.82 & 78.57 & 44.31 & 55.78\\
    & & TSV~\cite{gargiulo2025task} &  55.99 & \cellcolor{tabsecond}66.67 & 35.44 & 52.70 & \cellcolor{tabsecond}47.03 & \cellcolor{tabthird}58.51 & \cellcolor{tabthird}48.69 & \cellcolor{tabsecond}50.51 & \cellcolor{tabthird}50.80 & \cellcolor{tabsecond}81.21 & 47.54 & \cellcolor{tabthird}56.55\\
    & & WUDI~\cite{cheng2025whoever} & 55.99 & \cellcolor{tabthird}65.17 & \cellcolor{tabsecond}39.94 & \cellcolor{tabsecond}53.70 & 45.84 & 58.16 & 46.85 & 49.23 & 49.45 & 80.72 & 46.98 & 56.02 \\
    & & OptMerge~\cite{wei2025optmerge} & \cellcolor{tabsecond}59.58 & 64.26 & 35.44 & 53.10 & \cellcolor{tabthird}47.02 & 58.24 & \cellcolor{tabsecond}48.83 & \cellcolor{tabthird}50.50 & 50.21 & 80.24 & \cellcolor{tabthird}49.07 & \cellcolor{tabsecond}56.62 \\
    & & \textbf{DCAT (Ours)} & \cellcolor{tabfirst}60.22 & \cellcolor{tabfirst}69.06 & \cellcolor{tabfirst}41.74 & \cellcolor{tabfirst}56.20 & \cellcolor{tabfirst}48.46 & \cellcolor{tabsecond}59.51 & \cellcolor{tabfirst}49.36 & \cellcolor{tabfirst}52.44 & \cellcolor{tabfirst}54.17 & \cellcolor{tabfirst}82.96 & \cellcolor{tabfirst}51.25 & \cellcolor{tabfirst}59.40 \\
    \midrule
    \multirow{14}{*}{\rotatebox{90}{LLaMA 3.1-based}} 
    & 
    & Original & 52.40 & 57.66 & 56.46 & 55.50 & 33.53 & \cellcolor{tabthird}36.35 & 53.00 & 42.86 & 52.57 & 55.48 & \cellcolor{tabthird}31.02 & 47.49 \\
    \cmidrule{2-15}
    & \multirow{6}{*}{\rotatebox{90}{$\lambda=1.0$}} 
    & TA$^*$  &  48.50 & 54.65 & 54.65 & 52.60 & 35.54 & 32.97 & 46.23 & 39.68 & 49.62 & 57.07 & 27.57 & 45.31 \\
    & & TIES & 52.10 & 56.16 & \cellcolor{tabfirst}59.76 & 56.00 & \cellcolor{tabsecond}37.43 & 36.11 & 49.39 & 42.43 & \cellcolor{tabthird}54.93 & \cellcolor{tabthird}57.56 & 30.49 & 48.28 \\
    & & ISO & 55.39 & 57.96 & \cellcolor{tabsecond}59.46 & \cellcolor{tabsecond}57.60 & \cellcolor{tabthird}36.15 & 34.69 & 47.01 & 40.63 & \cellcolor{tabsecond}55.94 & \cellcolor{tabsecond}58.67 & 29.35 & 48.44 \\
    & & TSV$^*$ & 48.50 & 54.65 & 54.65 & 52.60 & 35.54 & 32.97 & 46.23 & 39.68 & 49.62 & 57.07 & 27.57 & 45.31 \\
    & & WUDI & 49.10 & 55.56 & 56.16 & 53.60 & 35.87 & 32.53 & 45.97 & 39.56 & 49.96 & 57.28 & 26.90 & 45.46\\
    & & OptMerge$^*$ & 48.50 & 54.65 & 54.65 & 52.60 & 35.54 & 32.97 & 46.23 & 39.68 & 49.62 & 57.07 & 27.57 & 45.31 \\
    \cmidrule{2-15}
    & \multirow{6}{*}{\rotatebox{90}{$\lambda=0.2$}} 
    & TA$^*$  & \cellcolor{tabthird}55.99 & \cellcolor{tabsecond}58.56 & \cellcolor{tabthird}57.06 & \cellcolor{tabthird}57.20 & 35.73 & 35.57 & 51.75 & 42.81 & 54.51 & \cellcolor{tabthird}57.56 & \cellcolor{tabsecond}31.22 & \cellcolor{tabsecond}48.66 \\
    & & TIES & \cellcolor{tabsecond}57.49 & \cellcolor{tabthird}58.26 & \cellcolor{tabthird}57.06 & \cellcolor{tabsecond}57.60 & 35.26 & 36.01 & \cellcolor{tabthird}53.10 & \cellcolor{tabthird}43.37 & 54.59 & 56.66 & 30.58 & \cellcolor{tabthird}48.56\\
    & & ISO & 52.99 & \cellcolor{tabsecond}58.56 & 56.76 & 56.10 & 34.55 & \cellcolor{tabsecond}36.43 & \cellcolor{tabsecond}53.67 & \cellcolor{tabsecond}43.50 & 53.24 & 57.49    & 30.19 & 48.10 \\
    & & TSV$^*$ & \cellcolor{tabthird}55.99 & \cellcolor{tabsecond}58.56 & \cellcolor{tabthird}57.06 & \cellcolor{tabthird}57.20 & 35.73 & 35.57 & 51.75 & 42.81 & 54.51 & \cellcolor{tabthird}57.56 & \cellcolor{tabsecond}31.22 & \cellcolor{tabsecond}48.66\\
    & & WUDI & 54.49 & 57.06 & 55.56 & 55.70 & 35.49 & 35.45 & 52.14 & 42.88 & 52.74 & 57.49 & 29.91 & 47.74 \\
    & & OptMerge$^*$ & \cellcolor{tabthird}55.99 & \cellcolor{tabsecond}58.56 & \cellcolor{tabthird}57.06 & \cellcolor{tabthird}57.20 & 35.73 & 35.57 & 51.75 & 42.81 & 54.51 & \cellcolor{tabthird}57.56 & \cellcolor{tabsecond}31.22 & \cellcolor{tabsecond}48.66\\
    \cmidrule{2-15}
    & 
    & \textbf{DCAT (Ours)} & \cellcolor{tabfirst}58.56 & \cellcolor{tabfirst}59.16 & \cellcolor{tabfirst}59.76 & \cellcolor{tabfirst}59.16 & \cellcolor{tabfirst}37.96 & \cellcolor{tabfirst}37.28 & \cellcolor{tabfirst}54.25 & \cellcolor{tabfirst}44.93 & \cellcolor{tabfirst}58.30 & \cellcolor{tabfirst}61.56 & \cellcolor{tabfirst}31.94 & \cellcolor{tabfirst}51.18\\    
    \bottomrule
    \end{tabular}
    }

    \resizebox{\textwidth}{!}{
    \begin{tabular}{ccl|cccc|cccccc|cc|c}
    \toprule
         & & & \multicolumn{4}{c|}{Video-MME} & \multicolumn{6}{c|}{TempCompass} & \multirow{2}{*}{MVBench} & \multirow{2}{*}{MLVU} & \multirow{2}{*}{Average} \\
         & & & Short & Medium & Long & Overall & Action & Attr. & Dir. & Order & Speed & Overall & & & \\
    \midrule
    \multirow{8}{*}{\rotatebox{90}{Vicuna-based}} 
    & \multirow{8}{*}{} 
    & Original & 39.89 & 34.22 & 30.89 & 35.00 & 78.48 & 54.00 & \cellcolor{tabsecond}48.53 & 47.56 & 45.85 & 55.46 & 38.48 & 32.59 & 40.38 \\
    \cmidrule{2-16}
    & & TA~\cite{ilharco2022editing}  & \cellcolor{tabthird}44.67 & 34.78 & 31.89 & 37.11 & 78.16 & 53.13 & 46.89 & \cellcolor{tabthird}49.09 & 46.99 & 55.38 & 43.68 & 44.84 & 45.25 \\
    & & TIES~\cite{yadav2023ties} & 31.44 & 27.78 & 25.67 & 28.30 & 76.53 & 50.78 & 46.45 & 47.66 & \cellcolor{tabsecond}48.82 & 54.62 & 32.78 & 33.15 & 37.21 \\
    & & ISO~\cite{marczak2025no} & 43.67 & \cellcolor{tabsecond}38.22 & 33.44 & \cellcolor{tabthird}38.44 & 79.38 & 53.81 & 46.45 & 46.44 & 47.77 & 55.38 & 43.00 & \cellcolor{tabsecond}45.31 & 45.53 \\
    & & TSV~\cite{gargiulo2025task} & \cellcolor{tabsecond}46.44 & 36.67 & \cellcolor{tabsecond}34.00 & \cellcolor{tabsecond}39.04 & \cellcolor{tabsecond}82.72 & \cellcolor{tabsecond}55.76 & \cellcolor{tabthird}47.75 & \cellcolor{tabsecond}49.29 & 47.69 & \cellcolor{tabsecond}57.24 & \cellcolor{tabthird}44.58 & \cellcolor{tabthird}45.18 & \cellcolor{tabsecond}46.51\\
    & & WUDI~\cite{cheng2025whoever} & 44.56 & 34.22 & 30.89 & 36.56 & 76.94 & 53.71 & 46.11 & 48.68 & 46.72 & 54.93 & 43.30 & 44.38 & 44.79\\
    & & OptMerge~\cite{wei2025optmerge} & 44.33 & \cellcolor{tabthird}37.00 & \cellcolor{tabthird}33.56 & 38.30 & \cellcolor{tabthird}80.68 & \cellcolor{tabthird}55.57 & 47.40 & 47.26 & \cellcolor{tabthird}48.38 & \cellcolor{tabthird}56.47 & \cellcolor{tabsecond}44.83 & 43.29 & \cellcolor{tabthird}45.72 \\
    & & \textbf{DCAT (Ours)} & \cellcolor{tabfirst}47.96 & \cellcolor{tabfirst}40.46 & \cellcolor{tabfirst}36.17 & \cellcolor{tabfirst}41.53 & \cellcolor{tabfirst}83.80 & \cellcolor{tabfirst}58.42 & \cellcolor{tabfirst}50.98 & \cellcolor{tabfirst}50.83 & \cellcolor{tabfirst}53.52 & \cellcolor{tabfirst}59.71 & \cellcolor{tabfirst}47.65 & \cellcolor{tabfirst}46.34 & \cellcolor{tabfirst}48.81 \\
    \midrule
    \multirow{14}{*}{\rotatebox{90}{LLaMA 3.1-based}} 
    & 
    & Original & 47.33 & 35.44 & 31.78 & 38.19 & 79.87 & 44.53 & 45.33 & 42.28 & 40.44 & 51.28 & 41.00 & 31.79 & 40.57 \\
    \cmidrule{2-16}
    & \multirow{6}{*}{\rotatebox{90}{$\lambda=1.0$}} 
    & TA  &  48.22 & 36.33 & 32.11 & 38.89 & 77.83 & 49.51 & 43.17 & 45.63 & 46.29 & 53.11 & 43.88 & 45.22 & 45.28 \\
    & & TIES & 50.78 & 38.67 & 35.00 & 41.48 & 74.74 & 44.63 & 41.70 & 37.40 & 34.06 & 47.22 & 46.88 & 47.12 & 45.68 \\
    & & ISO & 50.67 & 38.67 & 33.33 & 40.89 & 44.01 & 32.03 & 26.56 & 19.41 & 20.17 & 28.85 & 46.03 & 45.28 & 40.26 \\
    & & TSV & 51.78 & \cellcolor{tabthird}39.11 & 34.67 & 41.85 & 83.29 & \cellcolor{tabthird}52.83 & 46.63 & \cellcolor{tabthird}47.26 & 50.31 & 56.77 & \cellcolor{tabthird}47.00 & \cellcolor{tabthird}48.08 & \cellcolor{tabsecond}48.43 \\
    & & WUDI & 50.78 & 36.67 & 33.00 & 40.15 & 19.48 & 11.91 & 10.73 & 3.86 & 5.24 & 10.53 & 45.35 & 47.51 & 35.89\\
    & & OptMerge & \cellcolor{tabthird}52.44 & \cellcolor{tabsecond}39.89 & 36.22 & \cellcolor{tabthird}42.85 & 5.79 & 4.00 & 4.33 & 1.93 & 1.57 & 3.59 & \cellcolor{tabsecond}47.23 & \cellcolor{tabsecond}49.24 & 35.73 \\
    \cmidrule{2-16}
    & \multirow{6}{*}{\rotatebox{90}{$\lambda=0.2$}} 
    & TA  & \cellcolor{tabsecond}52.56 & 38.11 & \cellcolor{tabthird}38.11 & \cellcolor{tabsecond}42.93 & 63.49 & 38.09 & 36.33 & 35.87 & 39.04 & 43.15 & 43.18 & 40.94 & 42.55 \\
    & & TIES & 46.56 & 37.33 & 35.78 & 39.89 & 65.53 & 37.79 & 37.20 & 36.18 & 44.54 & 44.92 & 45.43 & 38.58 & 42.21 \\
    & & ISO & 45.67 & 36.44 & 36.44 & 39.52 & 83.21 & 52.34 & \cellcolor{tabthird}47.23 & 46.85 & 52.58 & 57.19 & 44.80 & 44.48 & 46.50 \\
    & & TSV & 49.22 & 37.67 & 37.44 & 41.44 & 79.05 & 47.56 & 43.43 & 43.90 & 49.17 & 53.36 & 45.65 & 40.42 & 45.22 \\
    & & WUDI & 51.56 & 37.00 & \cellcolor{tabsecond}38.33 & 42.30 & \cellcolor{tabsecond}85.74 & \cellcolor{tabsecond}53.52 & \cellcolor{tabsecond}47.75 & \cellcolor{tabsecond}48.78 & \cellcolor{tabthird}52.66 & \cellcolor{tabsecond}58.44 & 44.55 & 43.81 & \cellcolor{tabthird}47.28 \\
    & & OptMerge & 47.78 & 37.78 & 36.22 & 40.59 & \cellcolor{tabthird}83.94 & 51.86 & 46.71 & 47.05 & \cellcolor{tabsecond}53.10 & \cellcolor{tabthird}57.30 & 46.48 & 41.83 & 46.55 \\
    \cmidrule{2-16}
    & 
    & \textbf{DCAT (Ours)} & \cellcolor{tabfirst}53.67 & \cellcolor{tabfirst}40.33 & \cellcolor{tabfirst}39.33 & \cellcolor{tabfirst}44.44 & \cellcolor{tabfirst}86.37 & \cellcolor{tabfirst}55.64 & \cellcolor{tabfirst}50.47 & \cellcolor{tabfirst}50.19 & \cellcolor{tabfirst}56.02 & \cellcolor{tabfirst}59.88 & \cellcolor{tabfirst}51.18 & \cellcolor{tabfirst}50.15 & \cellcolor{tabfirst}51.41\\
    \bottomrule
    \end{tabular}
    }
\end{table}

\subsection{Quantitative comparison}\label{sec:quanti}
To evaluate how effectively DCAT boosts modality performance, we compare it against various model merging methods. The Audio+Vision results are shown in the top of Tab.~\ref{tab:quant}, and the Video+Vision results in the bottom. For the LLaMA-based model, the default merging coefficient $\lambda=1.0$ used by most methods consistently degrades performance; for a fair comparison, we therefore additionally report each method's performance at $\lambda=0.2$, which yields its best results. On audio, DCAT improves over the baseline by 7.62\% on comprehensive benchmarks such as MMAU, and by up to 100.92\% on MuCho-Music. On video, DCAT delivers gains of 16.37\% on Video-MME and 23.83\% on MVBench. Together with our analysis, these results demonstrate DCAT's effectiveness.

\begin{figure}
    \centering
    \includegraphics[width=\linewidth]{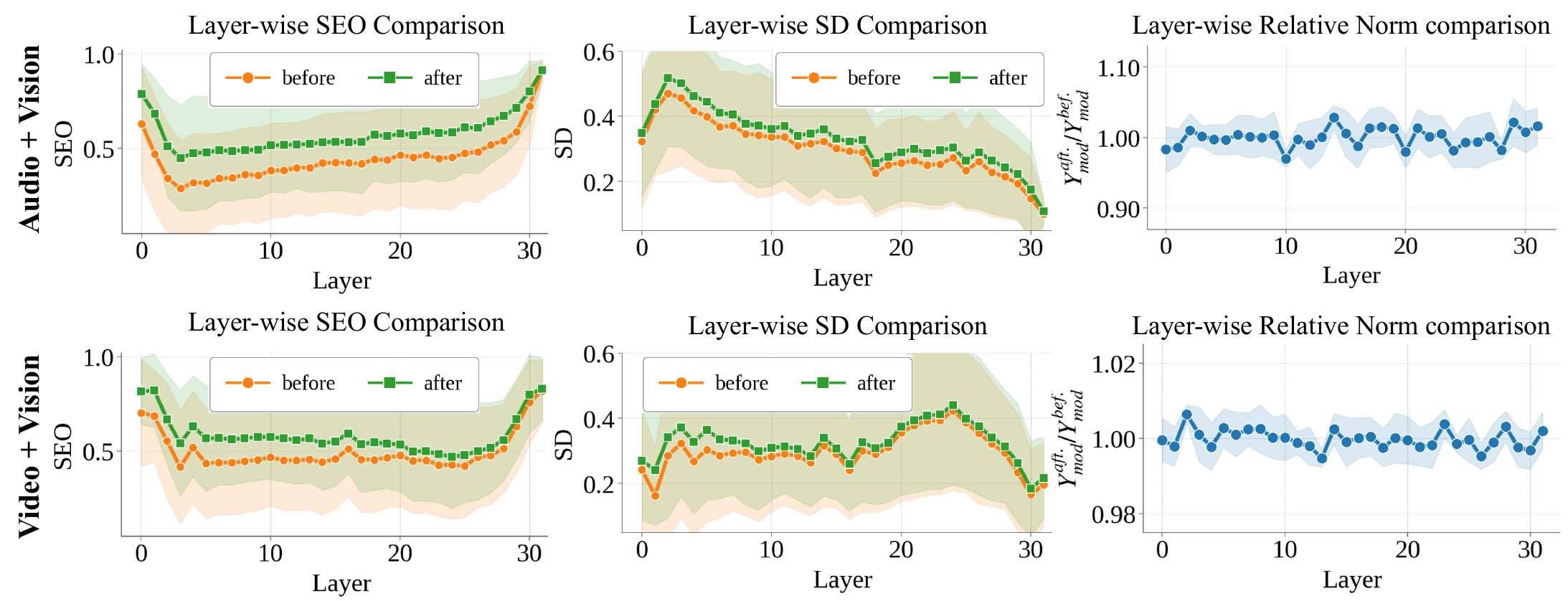}
    \caption{\textbf{SEO, SD and token norm ratio before and after applying DCAT.} We compare SEO, SD, and the norm ratio of modality tokens before and after applying DCAT. As depicted in first and second column, SEO and SD increase to a statistically significant degree ($p < 0.05$), while the relative norm remains close to 1.0, confirming that our DCAT formulation behaves as intended.}
    \label{fig:bef_aft_DCAT}
\end{figure}

\subsection{Does DCAT increase SEO and SD? Yes.}
We theoretically showed that increasing SEO and SD can raise a lower bound on the MI, and on this basis we proposed DCAT. To verify that DCAT indeed increases SEO and SD in practice, we measured both quantities at every layer of each Vicuna-based model under the Audio+Vision and Video+Vision settings. In addition, to check whether the first term of $C_{\delta}$ in Theorem~\ref{thm:lower_bound}, $\frac{k}{2}\log(\mathrm{tr}(\Sigma_y/k))$, is preserved, we measured the token norm ratio $\|Y_\text{mod}^\text{aft.}\|_F/\|Y_\text{mod}^\text{bef.}\|_F$ between before and after applying DCAT, while the second term $\det(\bar{\Sigma}_{z|x})$ is intractable to estimate directly. The results, presented in Fig.~\ref{fig:bef_aft_DCAT}, show that DCAT yields a significant increase in both SEO and SD while keeping the norm ratio essentially unchanged, suggesting that DCAT successfully raises the lower bound on the MI.

\subsection{Ablation study}
To assess the contribution of each component of DCAT, we conduct an ablation study summarized in Tab.~\ref{tab:ablation}. Experiments use Vicuna-based models: the audio+vision setting employs the full MMAU benchmark, while the video+vision setting uses 900 randomly sampled examples from Video-MME. A2--A5 ablate the components of DCAT's objective and the adaptive $\alpha$ schedule and the results show that each component contributes meaningfully to performance and confirm both the validity of design choices. Further ablation study including hyperparameter sweep, calibration dataset size and projection dimension $k$ is provided in Appendix Sec.~\ref{sec:append_ablation}.

\begin{table}[]
    \caption{\textbf{Ablation study on DCAT.} We ablate the individual components of DCAT. (A2--A5) isolate the contribution of each term in the DCAT objective by either dropping a term or fixing the $\alpha$ scheduler, while (A6--A9) examine the sensitivity of DCAT to its key hyperparameters.}
    \label{tab:ablation}
    \centering
    \resizebox{\textwidth}{!}{%
    \begin{tabular}{cl|cccc|cccc}
    \toprule
    \multirow{2}{*}{ID} & \multirow{2}{*}{Study subject} & \multicolumn{4}{c|}{MMAU} & \multicolumn{4}{c}{Video-MME} \\
    & & Music & Sound & Speech & Overall & Short & Medium & Long & Overall \\
    \midrule
    A0 & Original & 54.79 & 60.66 & 31.23 & 48.90 & 38.33 & 33.67 & 30.67 & 34.22 \\
    A1 & Base (DCAT) & 60.22 & 69.06 & 41.74 & 56.20 & 46.33 & 39.67 & 34.33 & 40.11 \\
    \midrule
    A2  & w/o SEO encouraging term in $\mathcal{L}_{\mathrm{self}}(\tau)$ & 57.49 & 66.97 & 34.83 & 53.10 & 45.33 & 34.00 & 30.33 & 36.56 \\
    A3  & w/o SD encouraging term in $\mathcal{L}_{\mathrm{self}}(\tau)$ & 55.09 & 64.26 & 34.53 & 51.30 & 35.67 & 25.00 & 25.00 & 28.56 \\
    A4  & w/o $\mathcal{L}_{\mathrm{cross}}(\tau)$ & 25.45 & 27.63 & 24.62 & 25.90 & 3.33 & 4.67 & 2.67 & 3.56 \\
    A5  & Fixed $\alpha$ scheduling & 55.99 & 61.86 & 36.34 & 51.40 & 36.67 & 31.33 & 27.67 & 31.89 \\
    \bottomrule
    \end{tabular}
    }
\end{table}

\section{Discussion}\label{sec:discussion}
As shown in Fig.~\ref{fig:correlation}, we consistently observe that vision models suffer a performance drop when merging is applied, and this trend holds across multiple models and architectures. Accordingly, most of our experiments employ a vision model as the donor and MLLMs of other modalities as recipients for alignment transfer. We offer the following conjecture as to why the donor has consistently been a vision model: because visual data is easier to obtain and far more abundant than audio or video data, existing vision MLLMs are likely to be more thoroughly aligned than their counterparts in other modalities. Under this hypothesis, DCAT is expected to leverage models trained on common modalities to enhance the performance of MLLMs built for relatively rare and data-scarce modalities.

\section{Conclusion}
In this paper, we report the phenomenon that inter-modality model merging can improve performance, and we interpret this effect from a Mutual Information (MI) perspective by deriving a lower bound on the MI and proposing a monotonic Alignment metric. An extensive empirical and statistical analysis demonstrates that this Alignment exhibits a strong correlation with downstream performance. Building on these theoretical and empirical findings, we propose Directional Cross-Modal Alignment Transfer (DCAT), which transfers alignment from a donor modality model to a recipient modality model in closed form. Across multiple audio and video MLLMs, we validate DCAT on a wide range of benchmarks, achieving average relative improvements of 29.55\% on audio and 24.92\% on video over the original models, thereby substantiating the soundness and the effectiveness of DCAT.

\section*{Acknowledgment}
This work was supported by Institute of Information \& communications Technology Planning \& Evaluation (IITP) grants funded by the Korea government(MSIT) [No.RS-2021-II211343, Artificial Intelligence Graduate School Program (Seoul National University), No.RS-2025-02314125, Effective Human-Machine Teaming With Multimodal Hazy Oracle Models], the National Research Foundation of Korea(NRF) grant funded by the Korea government(MSIT) (No. RS-2025-02263628), Electronics and Telecommunications Research Institute (ETRI) grant funded by the Korean government [26CR1100, Development of molecular endoscopy for real-time precision tumor analysis], and Korea Health Industry Development Institute(RS-2024-00439677). Also, the authors acknowledged the financial support from the BK21 FOUR program of the Education and Research Program for Future ICT Pioneers, Seoul National University.

{
    \small
    \bibliographystyle{plain}
    \bibliography{main}
}

\newpage
\appendix

\renewcommand{\thefigure}{A\arabic{figure}}
\renewcommand{\thetable}{A\arabic{table}}
\renewcommand{\theequation}{A\arabic{equation}}
\setcounter{section}{0}
\setcounter{figure}{0}
\setcounter{table}{0}
\setcounter{lemma}{0}
\setcounter{theorem}{0}
\setcounter{equation}{0}

\section{Code}\label{sec:append_code}
We provide the code implementation of DCAT for reproducibility. Please refer to \href{https://github.com/Hyun1A/DCAT}{this link} for environment setup instructions and execution guidelines.

\section{Proof}\label{sec:append_proof}
\subsection{Proof of Theorem~\ref{thm:lower_bound}}
\label{sec:append_thm_proof}
We first derive two auxiliary lemmas that connect the definitions of SEO and SD
to the covariance quantities used in the information-theoretic lower bound.

\begin{lemma}[SEO bounded by projected covariance ratio]
\label{lem:seo_from_main}
Let $Y_{\mathrm{mod}}\in\mathbb{R}^{N_m\times d}$ be the modality-token
activation matrix, and let $U_k\in\mathbb{R}^{k\times d}$ contain the top-$k$
right singular vectors of the centered text-token activations. Define
\begin{equation}
    \mu_y
    :=
    \frac{1}{N_m}Y_{\mathrm{mod}}^\top \mathbf{1}
    \in \mathbb{R}^{d},
    \qquad
    \widetilde Y_{\mathrm{mod}}
    :=
    Y_{\mathrm{mod}}-\mathbf{1}\mu_y^\top .
\end{equation}
Assume that the mean activation is orthogonal to the text principal subspace\footnote{Our empirical validation on LLaMA-3.1 indicates $\|\mu_y U_k^\top\|_2/\|\mu_y\|_2$ is $0.09726\pm 0.02166$ which is sufficiently small.}:
\begin{equation}
\label{eq:mean_orthogonal_to_text_subspace}
    \mu_y U_k^\top = 0 .
\end{equation}
Define
\begin{equation}
    Z:=Y_{\mathrm{mod}}U_k^\top,
    \qquad
    \widetilde Z:=\widetilde Y_{\mathrm{mod}}U_k^\top ,
\end{equation}
and define the unbiased covariance matrices
\begin{equation}
    \Sigma_y
    :=
    \frac{1}{N_m-1}
    \widetilde Y_{\mathrm{mod}}^\top \widetilde Y_{\mathrm{mod}},
    \qquad
    \Sigma_Z
    :=
    \frac{1}{N_m-1}
    \widetilde Z^\top \widetilde Z .
\end{equation}
The Subspace Energy Overlap is defined using the uncentered activations:
\begin{equation}
\label{eq:seo_main_app}
    \mathrm{SEO}
    =
    \frac{\|Y_{\mathrm{mod}}U_k^\top\|_F^2}
         {\|Y_{\mathrm{mod}}\|_F^2}.
\end{equation}
If $\operatorname{tr}(\Sigma_y)>0$, then
\begin{equation}
\label{eq:seo_trace_ratio_app}
    \mathrm{SEO}
    \le
    \frac{\operatorname{tr}(\Sigma_Z)}
         {\operatorname{tr}(\Sigma_y)}.
\end{equation}
\end{lemma}

\begin{proof}
We decompose the modality activations into their centered component and mean
component:
\begin{equation}
    Y_{\mathrm{mod}}
    =
    \widetilde Y_{\mathrm{mod}}+\mathbf{1}\mu_y^\top .
\end{equation}
Since $\mathbf{1}^\top \widetilde Y_{\mathrm{mod}}=0$, these two components
are orthogonal in Frobenius inner product. Therefore,
\begin{equation}
\label{eq:denominator_decomp}
\begin{aligned}
    \|Y_{\mathrm{mod}}\|_F^2
    &=
    \|\widetilde Y_{\mathrm{mod}}+\mathbf{1}\mu_y^\top\|_F^2 \\
    &=
    \|\widetilde Y_{\mathrm{mod}}\|_F^2
    +
    \|\mathbf{1}\mu_y^\top\|_F^2 \\
    &=
    (N_m-1)\operatorname{tr}(\Sigma_y)
    +
    N_m\|\mu_y\|_2^2 .
\end{aligned}
\end{equation}
Next, consider the projected activations. By the assumption
$\mu_y U_k^\top=0$, we have
\begin{equation}
    \mathbf{1}\mu_y^\top U_k^\top = 0 .
\end{equation}
Hence,
\begin{equation}
\label{eq:numerator_decomp}
\begin{aligned}
    \|Y_{\mathrm{mod}}U_k^\top\|_F^2
    &=
    \|
        \widetilde Y_{\mathrm{mod}}U_k^\top
        +
        \mathbf{1}\mu_y^\top U_k^\top
    \|_F^2 \\
    &=
    \|\widetilde Y_{\mathrm{mod}}U_k^\top\|_F^2 \\
    &=
    \|\widetilde Z\|_F^2 \\
    &=
    (N_m-1)\operatorname{tr}(\Sigma_Z).
\end{aligned}
\end{equation}
Substituting Eq.~\eqref{eq:denominator_decomp} and
Eq.~\eqref{eq:numerator_decomp} into the definition of SEO gives
\begin{equation}
\begin{aligned}
    \mathrm{SEO}
    &=
    \frac{\|Y_{\mathrm{mod}}U_k^\top\|_F^2}
         {\|Y_{\mathrm{mod}}\|_F^2} \\
    &=
    \frac{(N_m-1)\operatorname{tr}(\Sigma_Z)}
         {(N_m-1)\operatorname{tr}(\Sigma_y)
          +
          N_m\|\mu_y\|_2^2}.
\end{aligned}
\end{equation}
Since $N_m\|\mu_y\|_2^2\ge 0$, we have
\begin{equation}
    (N_m-1)\operatorname{tr}(\Sigma_y)
    +
    N_m\|\mu_y\|_2^2
    \ge
    (N_m-1)\operatorname{tr}(\Sigma_y).
\end{equation}
Therefore,
\begin{equation}
\begin{aligned}
    \mathrm{SEO}
    &=
    \frac{(N_m-1)\operatorname{tr}(\Sigma_Z)}
         {(N_m-1)\operatorname{tr}(\Sigma_y)
          +
          N_m\|\mu_y\|_2^2} \\
    &\le
    \frac{(N_m-1)\operatorname{tr}(\Sigma_Z)}
         {(N_m-1)\operatorname{tr}(\Sigma_y)} \\
    &=
    \frac{\operatorname{tr}(\Sigma_Z)}
         {\operatorname{tr}(\Sigma_y)}.
\end{aligned}
\end{equation}
\end{proof}

\begin{lemma}[SD in terms of determinant and trace of $\Sigma_Z$]
\label{lem:sd_from_main}
Let $\Sigma_Z\succ0$ be the empirical covariance matrix of the projected
modality-token activations $Z=Y_{\mathrm{mod}}U_k^\top$, and let
$\lambda_1,\ldots,\lambda_k>0$ denote its eigenvalues. Define
\begin{equation}
    p_j
    =
    \frac{\lambda_j}{\sum_{\ell=1}^{k}\lambda_\ell},
    \qquad
    j=1,\ldots,k.
\end{equation}
The Spectral Diversity is
\begin{equation}
\label{eq:sd_main_app}
    \mathrm{SD}
    =
    \frac{1}{k}
    \exp
    \left(
        -\sum_{j=1}^{k}p_j\log p_j
    \right).
\end{equation}
Define the determinant-normalized spectral volume as
\begin{equation}
\label{eq:sd_det_def_app}
    \mathrm{SD}_{\det}
    :=
    \frac{\det(\Sigma_Z)^{1/k}}
         {\operatorname{tr}(\Sigma_Z)/k}.
\end{equation}
Assume that the projected spectrum satisfies the condition
\begin{equation}
    p_j
    \geq
    \frac{\delta}{k},
    \qquad
    \forall j\in\{1,\ldots,k\},
    \label{eq:anti_collapse}
\end{equation}
for some sufficiently large $ \delta\in(0,1]$. Then
\begin{equation}
\label{eq:sd_det_tau_sd_app}
    \mathrm{SD}_{\det}
    \geq
    \delta\,\mathrm{SD}.
\end{equation}
Equivalently,
\begin{equation}
\label{eq:log_sd_det_tau_sd_app}
    \log\mathrm{SD}_{\det}
    \geq
    \log\mathrm{SD}
    +
    \log\delta.
\end{equation}
\end{lemma}

\begin{proof}
First, by the entropy upper bound of a discrete distribution over $k$ outcomes,
\begin{equation}
    -\sum_{j=1}^{k}p_j\log p_j
    \leq
    \log k.
\end{equation}
Therefore, from Eq.~\eqref{eq:sd_main_app},
\begin{equation}
    0<\mathrm{SD}\leq1.
\end{equation}

Next, using $\sum_j\lambda_j=\operatorname{tr}(\Sigma_Z)$ and
$p_j=\lambda_j/\sum_{\ell=1}^{k}\lambda_\ell$, we have
\begin{equation}
    \lambda_j
    =
    p_j\operatorname{tr}(\Sigma_Z).
\end{equation}
Thus,
\begin{equation}
\begin{aligned}
    \mathrm{SD}_{\det}
    &=
    \frac{\det(\Sigma_Z)^{1/k}}
         {\operatorname{tr}(\Sigma_Z)/k} \\
    &=
    \frac{
        \left(
            \prod_{j=1}^{k}\lambda_j
        \right)^{1/k}
    }{
        \operatorname{tr}(\Sigma_Z)/k
    } \\
    &=
    \frac{
        \left(
            \prod_{j=1}^{k}
            p_j\operatorname{tr}(\Sigma_Z)
        \right)^{1/k}
    }{
        \operatorname{tr}(\Sigma_Z)/k
    } \\
    &=
    k
    \left(
        \prod_{j=1}^{k}p_j
    \right)^{1/k} \\
    &=
    \left(
        \prod_{j=1}^{k}kp_j
    \right)^{1/k}.
\end{aligned}
\end{equation}
By the condition in Eq.~\eqref{eq:anti_collapse}, we have
$kp_j\geq\delta$ for all $j$. Hence,
\begin{equation}
    \mathrm{SD}_{\det}
    =
    \left(
        \prod_{j=1}^{k}kp_j
    \right)^{1/k}
    \geq
    \delta.
\end{equation}
Since $\mathrm{SD}\leq1$, it follows that
\begin{equation}
    \mathrm{SD}_{\det}
    \geq
    \delta
    \geq
    \delta\,\mathrm{SD}.
\end{equation}
Taking logarithms gives
\begin{equation}
    \log\mathrm{SD}_{\det}
    \geq
    \log\mathrm{SD}
    +
    \log\delta.
\end{equation}
\end{proof}
We empirically validate the use of $\delta$ in Sec.~\ref{sec:empirical_validation_tau}. Based on the results of Lemma~\ref{lem:seo_from_main} and Lemma~\ref{lem:sd_from_main}, we prove the main theorem.
\begin{theorem}[Weighted Alignment-MI Lower Bound]
\label{thm:lower_bound}
Let $z=yU_k^\top$, where $U_k\in\mathbb{R}^{k\times d}$ contains the top-$k$
right singular vectors of the centered text activation matrix. Let
$\Sigma_y=\operatorname{Cov}(y)$ and $\Sigma_Z=\operatorname{Cov}(z)$.
For each text-token representation $x$, let
\begin{equation}
    \Sigma_{z|x}
    :=
    \operatorname{Cov}(z|x),
    \qquad
    \bar{\Sigma}_{z|x}
    :=
    \mathbb{E}_{x}
    \left[
        \Sigma_{z|x}
    \right].
\end{equation}
With assumption in Lemma~\ref{lem:sd_from_main}, we also supposed that $\Sigma_Z\succ0$, $\bar{\Sigma}_{z|x}\succ0$, and:
\begin{enumerate}[label=(A\arabic*)]
    \item $z\sim\mathcal{N}(\mu_Z,\Sigma_Z)$;
    \item $z|x\sim\mathcal{N}(\mu_z(x),\Sigma_{z|x})$;
\end{enumerate}
Then, for any $\gamma\in(0,1)$,
\begin{equation}
    I(y;x)
    \geq
    \kappa(\gamma)
    \log
    \left(
        \mathrm{SEO}^{\gamma}
        \mathrm{SD}^{1-\gamma}
    \right)
    +
    C_{\delta},
    \qquad
    \kappa(\gamma)
    =
    \frac{k}{2\min(\gamma,1-\gamma)},
\end{equation}
where
\begin{equation}
\label{eq:c_tau_app}
    C_{\delta}
    =
    \frac{k}{2}
    \log
    \left(
        \frac{\operatorname{tr}(\Sigma_y)}{k}
    \right)
    -
    \frac{1}{2}
    \log\det(\bar{\Sigma}_{z|x})
    +
    \frac{k}{2}\log\delta.
\end{equation}
\end{theorem}

\begin{proof}
Since $z=yU_k^\top$ is a deterministic function of $y$, the Markov chain
\begin{equation}
    x \longleftrightarrow y \longrightarrow z
\end{equation}
holds. By the data processing inequality,
\begin{equation}
    I(y;x)\geq I(z;x).
\end{equation}
It therefore suffices to lower-bound $I(z;x)$.

By the definition of mutual information,
\begin{equation}
    I(z;x)=h(z)-h(z|x).
\end{equation}
Since $z\sim\mathcal{N}(\mu_Z,\Sigma_Z)$ and $\Sigma_Z\succ0$, its marginal entropy satisfies:
\begin{equation}
    h(z)
    =
    \frac{1}{2}
    \log\det(2\pi e\,\Sigma_Z).
\end{equation}
Under the conditional Gaussian assumption
$z|x\sim\mathcal{N}(\mu_z(x),\Sigma_{z|x})$, we have
\begin{equation}
    h(z|x)
    =
    \mathbb{E}_x
    \left[
        \frac{1}{2}
        \log\det(2\pi e\,\Sigma_{z|x})
    \right].
\end{equation}
Since $\log\det(\cdot)$ is concave on the positive definite cone, Jensen's
inequality gives
\begin{equation}
    \mathbb{E}_x
    \left[
        \log\det(\Sigma_{z|x})
    \right]
    \leq
    \log\det
    \left(
        \mathbb{E}_x[\Sigma_{z|x}]
    \right)
    =
    \log\det(\bar{\Sigma}_{z|x}).
\end{equation}
Therefore,
\begin{equation}
    h(z|x)
    \leq
    \frac{1}{2}
    \log\det(2\pi e\,\bar{\Sigma}_{z|x}).
\end{equation}
Combining the preceding displays yields
\begin{equation}
\begin{aligned}
    I(z;x)
    &=
    h(z)-h(z|x) \\
    &\geq
    \frac{1}{2}
    \log\det(2\pi e\,\Sigma_Z)
    -
    \frac{1}{2}
    \log\det(2\pi e\,\bar{\Sigma}_{z|x}) \\
    &=
    \frac{1}{2}
    \log\det(\Sigma_Z)
    -
    \frac{1}{2}
    \log\det(\bar{\Sigma}_{z|x}).
\end{aligned}
\end{equation}

We now express the determinant term using SEO and SD. By the definition of
$\mathrm{SD}_{\det}$,
\begin{equation}
    \det(\Sigma_Z)^{1/k}
    =
    \frac{\operatorname{tr}(\Sigma_Z)}{k}
    \mathrm{SD}_{\det}.
\end{equation}
By Lemma~\ref{lem:seo_from_main},
\begin{equation}
    \frac{\operatorname{tr}(\Sigma_Z)}{k}
    \ge
    \frac{\operatorname{tr}(\Sigma_y)}{k}
    \mathrm{SEO}.
\end{equation}
Hence,
\begin{equation}
\begin{aligned}
    \log\det(\Sigma_Z)
    &=
    k\log
    \left(
        \det(\Sigma_Z)^{1/k}
    \right) \\
    &=
    k\log
    \left(
        \frac{\operatorname{tr}(\Sigma_Z)}{k}
        \mathrm{SD}_{\det}
    \right) \\
    &\ge
    k\log
    \left(
        \frac{\operatorname{tr}(\Sigma_y)}{k}
        \mathrm{SEO} \cdot
        \mathrm{SD}_{\det}
    \right).
\end{aligned}
\end{equation}
By Lemma~\ref{lem:sd_from_main}, the condition in Eq.~\ref{eq:anti_collapse} implies
\begin{equation}
    \log\mathrm{SD}_{\det}
    \geq
    \log\mathrm{SD}
    +
    \log\delta.
\end{equation}
Therefore,
\begin{equation}
\begin{aligned}
    \log\det(\Sigma_Z)
    &\geq
    k\log
    \left(
        \frac{\operatorname{tr}(\Sigma_y)}{k}
    \right)
    +
    k\log\mathrm{SEO}
    +
    k\log\mathrm{SD}
    +
    k\log\delta.
\end{aligned}
\end{equation}
Substituting this into the mutual-information lower bound gives
\begin{equation}
\begin{aligned}
    I(z;x)
    &\geq
    \frac{k}{2}\log\mathrm{SEO}
    +
    \frac{k}{2}\log\mathrm{SD}
    +
    \frac{k}{2}
    \log
    \left(
        \frac{\operatorname{tr}(\Sigma_y)}{k}
    \right)
    -
    \frac{1}{2}
    \log\det(\bar{\Sigma}_{z|x})
    +
    \frac{k}{2}\log\delta \\
    &=
    \frac{k}{2}
    \left(
        \log\mathrm{SEO}
        +
        \log\mathrm{SD}
    \right)
    +
    C_{\delta}.
\end{aligned}
\end{equation}

Finally, we convert the unweighted log-product into the weighted alignment
score. Let
\begin{equation}
    a:=\log\mathrm{SEO},
    \qquad
    b:=\log\mathrm{SD}.
\end{equation}
Since $0<\mathrm{SEO}\leq1$ and $0<\mathrm{SD}\leq1$, we have $a\leq0$ and
$b\leq0$. For any $\gamma\in(0,1)$, let
\begin{equation}
    m_\gamma:=\min(\gamma,1-\gamma).
\end{equation}
Since $\gamma\geq m_\gamma$ and $1-\gamma\geq m_\gamma$, and since $a,b\leq0$,
we have
\begin{equation}
    \gamma a\leq m_\gamma a,
    \qquad
    (1-\gamma)b\leq m_\gamma b.
\end{equation}
Adding both inequalities gives
\begin{equation}
    \gamma a+(1-\gamma)b
    \leq
    m_\gamma(a+b).
\end{equation}
Equivalently,
\begin{equation}
    a+b
    \geq
    \frac{1}{m_\gamma}
    \left(
        \gamma a+(1-\gamma)b
    \right).
\end{equation}
Since
\begin{equation}
    \gamma a+(1-\gamma)b
    =
    \log
    \left(
        \mathrm{SEO}^{\gamma}
        \mathrm{SD}^{1-\gamma}
    \right),
\end{equation}
we obtain
\begin{equation}
    \log\mathrm{SEO}
    +
    \log\mathrm{SD}
    \geq
    \frac{1}{\min(\gamma,1-\gamma)}
    \log
    \left(
        \mathrm{SEO}^{\gamma}
        \mathrm{SD}^{1-\gamma}
    \right).
\end{equation}
Consequently,
\begin{equation}
\begin{aligned}
    I(z;x)
    &\geq
    \frac{k}{2\min(\gamma,1-\gamma)}
    \log
    \left(
        \mathrm{SEO}^{\gamma}
        \mathrm{SD}^{1-\gamma}
    \right)
    +
    C_{\delta}.
\end{aligned}
\end{equation}
Since $I(y;x)\geq I(z;x)$, we conclude that
\begin{equation}
    I(y;x)
    \geq
    \frac{k}{2\min(\gamma,1-\gamma)}
    \log
    \left(
        \mathrm{SEO}^{\gamma}
        \mathrm{SD}^{1-\gamma}
    \right)
    +
    C_{\delta}.
\end{equation}
This completes the proof.
\end{proof}

\subsection{Derivation of Eq.~\ref{eq:target_properties}}\label{sec:append_properties}
In the main text, for the top-$k$ right singular vectors $U_k$ of the text activations, we defined the corresponding orthogonal projector and its orthogonal complement as $P_k\,(=U_k^\top U_k)$ and $P_k^\perp\,(=I-P_k)$, respectively. We further defined the uniform singular value that preserves the total output energy:
\begin{equation}
    \rho \;=\; \frac{\|X_{\mathrm{mod}}(\theta_0+\tau_r)\|_F}{\sqrt{k}}. \nonumber 
\end{equation}
In addition, we used the following target for the modality-token activations after the weight update:
\begin{equation}
T_\alpha = c_\alpha \cdot A\;\mathrm{diag}\left(\sigma_j^\alpha \rho^{1-\alpha}\right)_{j=1}^{k}B^\top U_k,
\qquad
c_\alpha = \frac{\|X_{\mathrm{mod}}(\theta_0+\tau_r) \|_F}{\left(\sum_{j=1}^{k}\sigma_j^{2\alpha}\rho^{2(1-\alpha)}\right)^{1/2}} \nonumber
\end{equation}
We now derive the following two properties of $T_\alpha$:
\begin{equation}
T_\alpha P_k^\perp=0,
\qquad
\|T_\alpha\|_F=\|X_{\mathrm{mod}}(\theta_0+\tau_r)\|_F. \nonumber
\end{equation}
For the first one, $T_\alpha$ can be written in the form $T_\alpha = Q U_k$ for some matrix $Q$. Hence, we obtain:
\begin{align}
T_\alpha P_k^\perp 
&= Q U_k (I - P_k) \nonumber \\
&= Q U_k (I - U_k^\top U_k) \nonumber  \\
&= Q U_k - Q U_k U_k^\top U_k \nonumber  \\
&= Q U_k - Q U_k \nonumber  \\
&= 0, \nonumber 
\end{align}
since $U_k U_k^\top = I_k$. For the second one, we denote $D_{\alpha}$ as:
\begin{equation}
    D_\alpha \;:=\; \mathrm{diag}\!\left(\sigma_j^{\alpha}\,\rho^{1-\alpha}\right)_{j=1}^{k}. \nonumber 
\end{equation}
Then, starting from the definition of $T_\alpha$:
\begin{equation}
    \|T_\alpha\|_F^2
    \;=\;
    c_\alpha^2 \;\cdot\; \bigl\|A\,D_\alpha\,B^\top\,U_k\bigr\|_F^2. \nonumber 
\end{equation}
 
Expanding the Frobenius norm via $\|M\|_F^2 = \mathrm{tr}(M^\top M)$:
\begin{align}
    \bigl\|A\,D_\alpha\,B^\top\,U_k\bigr\|_F^2
    &= \mathrm{tr}\!\Big(
        U_k^\top\,B\,D_\alpha\,\underbrace{A^\top A}_{=\,I_k}\,D_\alpha\,B^\top\,U_k
    \Big) \nonumber  \\[4pt] 
    &= \mathrm{tr}\!\Big(
        U_k^\top\,B\,D_\alpha^2\,B^\top\,U_k
    \Big). \label{eq:step_ATA}
\end{align}
 
Applying the cyclic property of the trace, $\mathrm{tr}(PQRS) = \mathrm{tr}(SPQR)$:
\begin{align}
    \mathrm{tr}\!\Big(
        U_k^\top\,B\,D_\alpha^2\,B^\top\,U_k
    \Big)
    &= \mathrm{tr}\!\Big(
        \underbrace{U_k\,U_k^\top}_{=\,I_k}\,B\,D_\alpha^2\,B^\top
    \Big) \nonumber  \\[4pt]
    &= \mathrm{tr}\!\Big(
        B\,D_\alpha^2\,B^\top
    \Big). \label{eq:step_UkUk}
\end{align}
 
Since $B$ is orthogonal,
\begin{align}
    \mathrm{tr}\!\Big(
        B\,D_\alpha^2\,B^\top
    \Big)
    &= \mathrm{tr}(D_\alpha^2)
    \;=\; \sum_{j=1}^{k} \sigma_j^{2\alpha}\,\rho^{2(1-\alpha)}. \label{eq:step_BBT}
\end{align}
 
Substituting~\eqref{eq:step_ATA}--\eqref{eq:step_BBT} back and inserting the definition of $c_\alpha$:
\begin{align}
    \|T_\alpha\|_F^2
    &= c_\alpha^2 \cdot \sum_{j=1}^{k}\sigma_j^{2\alpha}\,\rho^{2(1-\alpha)} \nonumber  \\[6pt]
    &= \frac{\|X_{\mathrm{mod}}(\theta_0+\tau_r) \|_F^2}{\displaystyle\sum_{j=1}^{k}\sigma_j^{2\alpha}\,\rho^{2(1-\alpha)}}
       \;\cdot\;
       \sum_{j=1}^{k}\sigma_j^{2\alpha}\,\rho^{2(1-\alpha)} \nonumber  \\[6pt]
    &= \|X_{\mathrm{mod}}(\theta_0+\tau_r)\|_F^2. \nonumber 
\end{align}

\subsection{Derivation of Eq.~\ref{eq:alignment_loss}}\label{sec:append_align}
In Eq.~\ref{eq:alignment_loss} of the main text, we adopted the following objective to enhance alignment:
\begin{align}
\mathcal{L}_{\mathrm{self}}(\tau) 
&= \underbrace{\|X_{\mathrm{mod}}(\theta_0 + \tau) P_k^\perp\|_F^2}_{\text{\textbf{Encouraging SEO}}}  \nonumber
+ \underbrace{\|X_{\mathrm{mod}}(\theta_0 + \tau)P_k-T_\alpha\|_F^2}_{\text{\textbf{Encouraging SD}}} \nonumber \\
&= \|X_{\mathrm{mod}}(\theta_0 + \tau)-T_\alpha\|_F^2. \nonumber
\end{align}
Then, the second equality can be derived as follows. From the previous derivation, $T_\alpha$ satisfies $T_\alpha P_k^\perp = 0$, which is equivalent to:
\begin{equation}\label{eq:Talpha_Pk}
    T_\alpha P_k = T_\alpha.
\end{equation}
 
For notational convenience, let us write $Y := X_{\mathrm{mod}}(\theta_0+\tau)$. Since $P_k + P_k^\perp = I$, we can decompose $Y$ and $T_\alpha$ as:
\begin{equation}
    Y = Y P_k + Y P_k^\perp, \qquad T_\alpha = T_\alpha P_k + T_\alpha P_k^\perp = T_\alpha + 0 = T_\alpha, \nonumber
\end{equation}
where the second equality uses~\eqref{eq:Talpha_Pk}.
Therefore:
\begin{equation}\label{eq:decomp}
    Y - T_\alpha
    = (Y P_k - T_\alpha) + Y P_k^\perp. 
\end{equation}
 
Expanding the squared Frobenius norm of~\eqref{eq:decomp}:
\begin{equation}\label{eq:expand}
    \|Y - T_\alpha\|_F^2
    = \|Y P_k - T_\alpha\|_F^2
    + \|Y P_k^\perp\|_F^2
    + 2\,\mathrm{tr}\!\Big(
        (Y P_k - T_\alpha)^\top\,Y P_k^\perp
    \Big). \nonumber
\end{equation}
 
It remains to show that the cross term vanishes. We split it into two parts:
\begin{equation}\label{eq:cross}
    \mathrm{tr}\!\Big((Y P_k - T_\alpha)^\top\,Y P_k^\perp\Big)
    = \underbrace{\mathrm{tr}\!\Big(P_k\,Y^\top Y\,P_k^\perp\Big)}_{\text{(I)}}
    \;-\; \underbrace{\mathrm{tr}\!\Big(T_\alpha^\top\,Y\,P_k^\perp\Big)}_{\text{(II)}}.
\end{equation}
 
For the term (I), the cyclic property of the trace gives us:
\begin{equation}
    \mathrm{tr}\!\Big(P_k\,Y^\top Y\,P_k^\perp\Big)
    = \mathrm{tr}\!\Big(\underbrace{P_k^\perp\,P_k}_{=\,0}\,Y^\top Y\Big)
    = 0,
\end{equation}
where we used $P_k^\perp P_k = 0$. For the term (II), since $T_\alpha = Q\,U_k$ by its definition, we have $T_\alpha^\top = U_k^\top Q^\top$.
By the cyclic property of the trace:
\begin{equation}
    \mathrm{tr}\!\Big(T_\alpha^\top\,Y\,P_k^\perp\Big)
    = \mathrm{tr}\!\Big(U_k^\top\,Q^\top\,Y\,P_k^\perp\Big)
    = \mathrm{tr}\!\Big(\underbrace{P_k^\perp\,U_k^\top}_{=\,0}\,Q^\top\,Y\Big)
    = 0,
\end{equation}
where we used
\begin{equation}
    P_k^\perp\,U_k^\top
    = (I - U_k^\top U_k)\,U_k^\top
    = U_k^\top - U_k^\top\underbrace{U_k\,U_k^\top}_{=\,I_k}
    = U_k^\top - U_k^\top
    = 0. \nonumber
\end{equation}

\subsection{Derivation of closed form for DCAT}\label{sec:append_closed}
  
\begin{proposition}

Let $X_m:=X_{\mathrm{mod}}$, $X_t:=X_{\mathrm{text}}$ for brevity.
Assume that
\begin{equation}
    X_m^\top X_m + \eta X_t^\top X_t \succ 0 .
    \label{eq:positive_definite_assumption}
\end{equation}
Then, for a given layer with pretrained weight $\theta_0\in\mathbb{R}^{d_{\mathrm{in}}\times d_{\mathrm{out}}}$, the DCAT objective is defined as:
\begin{equation}\label{eq:obj}
    \min_{\tau}\;\mathcal{L}_{\mathrm{self}}(\tau)\;+\;\eta\cdot\mathcal{L}_{\mathrm{cross}}(\tau),
\end{equation}
with
\begin{align}
    \mathcal{L}_{\mathrm{self}}(\tau)
    &= \|X_{\mathrm{m}}(\theta_0+\tau)-T_\alpha\|_F^2, \label{eq:Lalign} \nonumber \\
    \mathcal{L}_{\mathrm{cross}}(\tau)
    &= \beta\,\|X_{\mathrm{t}}(\tau-\tau_r)\|_F^2 + (1-\beta)\,\|X_{\mathrm{t}}(\tau-\tau_d)\|_F^2. \nonumber
\end{align}

Then, the objective~\eqref{eq:obj} has a unique minimizer as:
\begin{equation}\label{eq:closed_form}
    \tau^\star
    =
    \bigl(X_m^\top X_m + \eta\,X_t^\top X_t\bigr)^{-1}
    \Bigl[
        X_m^\top\bigl(T_\alpha - X_m\theta_0\bigr)
        \;+\;
        \eta\,X_t^\top X_t\bigl(\beta\tau_r+(1-\beta)\tau_d\bigr)
    \Bigr].
\end{equation}

\end{proposition}

\begin{proof}
The total objective $f(\tau):=\mathcal{L}_{\mathrm{self}}(\tau)+\eta\,\mathcal{L}_{\mathrm{cross}}(\tau)$ is a quadratic function of $\tau$, so we proceed by computing its gradient and setting it to zero. First, $\mathcal{L}_{\mathrm{self}}(\tau)$ can be represented as: 
\begin{align}
    \mathcal{L}_{\mathrm{self}}(\tau)
    &= \|X_{\mathrm{m}}(\theta_0+\tau)-T_\alpha\|_F^2 \nonumber \\
    &= \mathrm{tr}\!\bigl(\tau^\top X_m^\top X_m\,\tau\bigr)
       \;+\; 2\,\mathrm{tr}\!\bigl(\tau^\top X_m^\top(X_m\theta_0-T_\alpha)\bigr)
       \;+\; \mathrm{const}. \nonumber 
\end{align}
Using the identities $\nabla_\tau\,\mathrm{tr}(\tau^\top M\tau) = (M+M^\top)\tau = 2 M\tau$ for symmetric $M$, and $\nabla_\tau\,\mathrm{tr}(\tau^\top C) = C$, we obtain
\begin{equation}\label{eq:grad_align}
    \nabla_\tau\,\mathcal{L}_{\mathrm{self}}(\tau)
    = 2\,X_m^\top X_m\,\tau \;+\; 2\,X_m^\top(X_m\theta_0-T_\alpha)
    = 2\,X_m^\top\!\bigl(X_m(\theta_0+\tau)-T_\alpha\bigr).
\end{equation}
 
Similarly, the gradient of $\mathcal{L}_{\mathrm{cross}}$ can be represented as:
\begin{align}
    \nabla_\tau\,\mathcal{L}_{\mathrm{cross}}(\tau)
    &= 2\beta\,X_t^\top X_t(\tau-\tau_r)
      \;+\; 2(1-\beta)\,X_t^\top X_t(\tau-\tau_d) \nonumber \\
    &= 2\,X_t^\top X_t\,\tau
      \;-\; 2\,X_t^\top X_t\!\bigl(\beta\tau_r+(1-\beta)\tau_d\bigr).
      \label{eq:grad_text}
\end{align}
 
Then, $\nabla_\tau\,\mathcal{L}_{\mathrm{self}}+\eta\,\nabla_\tau\,\mathcal{L}_{\mathrm{cross}}$ from~\eqref{eq:grad_align} and~\eqref{eq:grad_text} is represented as:
\begin{equation}
    X_m^\top X_m\,\tau \;+\; X_m^\top(X_m\theta_0-T_\alpha)
    \;+\; \eta\,X_t^\top X_t\,\tau
    \;-\; \eta\,X_t^\top X_t\!\bigl(\beta\tau_r+(1-\beta)\tau_d\bigr)
    \;=\; 0.
    \label{eq:zero_gradient}
\end{equation}
Rearranging Eq.~\ref{eq:zero_gradient} gives the linear system:
\begin{equation}\label{eq:SPD_system}
    \underbrace{\bigl(X_m^\top X_m + \eta\,X_t^\top X_t\bigr)}_{=:\,A}\,\tau
    \;=\;
    \underbrace{X_m^\top(T_\alpha-X_m\theta_0)
    \;+\;
    \eta\,X_t^\top X_t\!\bigl(\beta\tau_r+(1-\beta)\tau_d\bigr)}_{=:\,B}.
    \nonumber
\end{equation}
Since we stack sufficiently large tokens to construct $X_m$ and $X_t$, we assume $X_m^\top X_m + \eta\,X_t^\top X_t \succ 0$, the closed-form solution for $\tau$ is:
\begin{equation}\label{eq:closed_form}
    \tau^\star
    =
    \bigl(X_m^\top X_m + \eta\,X_t^\top X_t\bigr)^{-1}
    \Bigl[
        X_m^\top\bigl(T_\alpha - X_m\theta_0\bigr)
        \;+\;
        \eta\,X_t^\top X_t\bigl(\beta\tau_r+(1-\beta)\tau_d\bigr)
    \Bigr]. \nonumber
\end{equation} 
\end{proof}


\begin{table}[h]
    \caption{\textbf{Empirical validation on $\mathrm{SD}$, $\mathrm{SD}_{\det}$, and $\delta_{\mathrm{eff}}$.}
    To assess the tightness of the bound, we compute the entropy-based spectral diversity
    $\mathrm{SD}$, the proof-side spectral diversity $\mathrm{SD}_{\det}$, and spectral conversion factor $\delta_{\mathrm{eff}}$ for Vicuna-based models in Fig.~\ref{fig:correlation}. The results show that $\delta_{\mathrm{eff}}$ remains close
    to one, indicating that the entropy-based effective-rank SD introduces only a mild
    constant-level relaxation in the MI lower bound.}
    \label{tab:emprical_validation_tau}
    \centering
    \begin{tabular}{ccc|ccc}
    \toprule
    \multicolumn{3}{c|}{Audio+Vision} & \multicolumn{3}{c}{Video+Vision}\\
    $\mathrm{SD}$ & $\mathrm{SD}_{\det}$ & $\delta_{\mathrm{eff}} \leq $
    & $\mathrm{SD}$ & $\mathrm{SD}_{\det}$ & $\delta_{\mathrm{eff}} \leq $  \\
    \midrule
     0.2839\tiny{$\pm0.0273$} 
     & 0.2752\tiny{$\pm0.0228$} 
     & 0.9704\tiny{$\pm0.0224$}
     & 0.2015\tiny{$\pm0.0178$} 
     & 0.2226\tiny{$\pm0.0197$} 
     & 1.1046\tiny{$\pm0.0140$} \\
     \bottomrule
    \end{tabular}
\end{table}

\subsection{Empirical validation of $\delta$}
\label{sec:empirical_validation_tau}

The proof of Theorem~\ref{thm:lower_bound} involves a spectral non-degeneracy factor
that connects the proof-side spectral term $\mathrm{SD}_{\det}$ with the entropy-based
effective-rank measure $\mathrm{SD}$ used in the main text. 
In particular, the relevant condition can be expressed as the existence of a constant
$\delta>0$ such that
\begin{equation}
\label{eq:sd_det_tau_sd_app}
    \mathrm{SD}_{\det}
    \geq
    \delta\,\mathrm{SD}.
\end{equation}
Equivalently, the largest empirical value supported by a given spectrum is
\begin{equation}
    \delta_{\mathrm{eff}}
    :=
    \frac{\mathrm{SD}_{\det}}{\mathrm{SD}}.
\end{equation}
Importantly, $\delta_{\mathrm{eff}}$ is a spectral conversion factor between two
spectral-spread measures, not a per-eigenvalue probability lower bound of the form
$p_j\geq \delta/k$. Therefore, $\delta_{\mathrm{eff}}$ is not constrained to be at most
one, and values larger than one do not contradict the probability constraint
$\sum_j p_j=1$.

For empirical validation, we compute both $\mathrm{SD}$ and $\mathrm{SD}_{\det}$ using
the features examined in our analysis, and report their ratio
$\delta_{\mathrm{eff}}$ in Table~\ref{tab:emprical_validation_tau}. 
As shown in the table, $\delta_{\mathrm{eff}}$ remains close to one in both settings:
$0.9704$ for Audio+Vision and $1.1046$ for Video+Vision. 
These results indicate that the conversion from $\mathrm{SD}_{\det}$ to the
entropy-based $\mathrm{SD}$ incurs only a mild constant-level relaxation in the MI
lower bound. Thus, although the main text adopts the more interpretable entropy-based effective-rank
measure $\mathrm{SD}$, the empirical gap between $\mathrm{SD}$ and $\mathrm{SD}_{\det}$
is small in our analyzed model families. This supports the use of $\mathrm{SD}$ as a
practical and meaningful spectral diversity metric for the proposed Alignment score.

\section{Implementation Details}\label{sec:append_implementation}
\subsection{Models and computation}\label{sec:append_models_and_computation}
For the Vicuna, the setting allows a controlled comparison because the three MLLMs are trained from the same base under the same recipe except for the modality in OptMerge~\cite{wei2025optmerge}. To complement this controlled setting, we additionally evaluate DCAT on independently trained models that nevertheless share the LLaMA-3.1-8B-Instruct backbones for audio, video, and vision modalities. For the two backbones, the vision MLLM is used as the donor in every experiment, since vision MLLMs are the modality on which standard merging tends to leave performance unchanged or slightly degraded while consistently lifting the other modalities. All experiments run on a single NVIDIA L40S 48\,GB GPU and the DCAT solver over all projections completes in about 25-30 minutes per setting.

\subsection{Calibration and alignment-target data}
For all four settings, calibration data is drawn from the publicly available AVQA dataset~\cite{yang2022avqa} since the evaluation benchmarks (MMAU, Video-MME, etc.) are entirely disjoint from AVQA so that calibration and evaluation never share data.
AVQA provides paired audio and video streams aligned to the same set of question/answer prompts, which lets us instantiate DCAT in a uniform way across audio and video recipients without changing the underlying source of $X_{\mathrm{text}}$. For audio recipients, we feed the \emph{audio track} in AVQA to the recipient MLLM as the modality input; for video recipients, we instead feed the \emph{video clip} of the same sample with the same prompt used for the audio track. 
To solve the objective in Eq. ~\ref{eq:DCAT_objective}, we draw small set of AVQA dataset (256 samples for each modality).

\subsection{Progressive, dependency-aware merging}
Although the closed-form update in Eq.~\ref{eq:DCAT_objective} could in principle be solved for all projections in parallel, we found it preferable to apply the update \emph{progressively} along two axes—across layers and across projection groups within each layer. Concretely, we process layers sequentially, and within each layer we partition the seven projections into four dependency-aware groups that mirror the activation-flow structure of a transformer block: \{\texttt{query,key,value}\} and \{\texttt{out}\} projections in the attention layers, \{\texttt{up,gate}\} and \{\texttt{down}\} projections in the attention projections and the feed-forward networks. 
The four groups are merged in the listed order, and after every group merge we re-collect $X_{\mathrm{mod}}$ and $X_{\mathrm{text}}$ at the relevant point so that the next group is solved against activations that already reflect the previous update. In this way, errors that would arise from treating projections as independent are absorbed iteratively, and the alignment accumulated across layers is corrected as it propagates.
 
\subsection{Hyperparameter values}
In DCAT we should control three hyperparameters: the text-subspace dimension $k$, the recipient/donor balance $\beta$ in the cross-modal alignment term, and the cross-modal alignment weight $\eta$. For example, the text-subspace dimension $k$ is shared within a backbone family ($k=128$ for Vicuna and $k=256$ for LLaMA-3.1-8B-Instruct), reflecting that hidden states of LLaMA-3.1-8B-Instruct are wider and require a larger effective rank to span the principal text directions. The complete configuration is reported in Table~\ref{tab:impl_full}.
 
\begin{table}[t]
    \centering
    \caption{\textbf{Full implementation configuration of DCAT for the four headline settings.} A blank entry under \emph{Evaluation seed} indicates that the benchmark is small enough to be evaluated in full and therefore does not require subsampling.}
    \label{tab:impl_full}
    \small
    \setlength{\tabcolsep}{4pt}
    \renewcommand{\arraystretch}{1.1}
    \resizebox{\textwidth}{!}{
    \begin{tabular}{@{}l|c|c|c|c@{}}
        \toprule
        Component & Vicuna+Audio & Vicuna+Video & LLaMA+Audio & LLaMA+Video \\
        \midrule
        Backbone                       & Vicuna-7B-v1.5                       & Vicuna-7B-v1.5                       & LLaMA-3.1-8B-Instruct                & LLaMA-3.1-8B-Instruct \\
        Calibration source             & AVQA (audio track)                   & AVQA (video track)                   & AVQA (audio track)                   & AVQA (video track) \\
        \# calibration samples         & 256                                  & 256                                  & 256                                  & 256 \\
        $k$ (text subspace dim)        & 128                                  & 128                                  & 256                                  & 256 \\
        $\beta$                & 0.30                                 & 0.30                                 & 0.95                                 & 0.95 \\
        $\eta$                         & $1{\times}10^{5}$                    & $1{\times}10^{5}$                    & $1{\times}10^{6}$                    & $1{\times}10^{6}$ \\
        \bottomrule
    \end{tabular}}
\end{table}
 
\subsection{Closed-form solve and engineering notes}
The DCAT objective in Eq.~\ref{eq:DCAT_objective} is quadratic in $\tau$ and admits a layer-wise SPD system $A\tau=B$ with
\[
A \;=\; X_{\mathrm{mod}}^\top X_{\mathrm{mod}} \;+\; \eta\, X_{\mathrm{text}}^\top X_{\mathrm{text}}.
\]
We solve this system in float64 using a Cholesky-based linear solver, with a small ridge $\epsilon_{\mathrm{solve}}=10^{-6}\,\mathrm{tr}(A)/d$ added on the diagonal for numerical stability. Activation matrices used by the solver are collected in float32, and the resulting weight delta is applied to the recipient checkpoint in its native bfloat16 precision. No gradient-based optimization is involved at any point.

\section{Experimental Details}\label{sec:append_details}
\subsection{Analysis details}\label{sec:append_analysis_details}
\paragraph{Alignment-accuracy Correlation Analysis.} To examine whether the proposed alignment metric is correlated with model performance, we conduct a correlation analysis between alignment and benchmark accuracy using both Pearson and Spearman correlation coefficients. The analysis is performed on Audio+Vision and Video+Vision settings, under Vicuna~\cite{chiang2023vicuna} and LLaMA 3.1~\cite{grattafiori2024llama} backbones. For model merging, we include Task Arithmetic~\cite{ilharco2022editing} (TA), TIES-merging~\cite{yadav2023ties} (TIES), ISO-C~\cite{marczak2025no} (ISO), Task Singular Vector~\cite{gargiulo2025task} (TSV), WUDI~\cite{cheng2025whoever}, and OptMerge~\cite{wei2025optmerge}.

For each merging method, we first construct a fully merged model by merging the Vision model with either the Audio or Video model. Because correlation analysis requires multiple evaluation points rather than a single merged checkpoint, we further generate additional model variants by applying SVD-based low-rank decomposition to the fully merged network parameters. Specifically, starting from each full-merge checkpoint, we decompose the merged weights and reconstruct a series of low-rank approximations with ranks 1, 2, 4, 8, 16, 32, 64, and 128. We then evaluate both the full-merge model and these SVD-derived variants, and compute Pearson and Spearman correlations between their Alignment scores and benchmark accuracies. This allows us to analyze how changes in the merged parameter space are reflected in both alignment and task performance.

The evaluation is conducted separately on the benchmark corresponding to each constituent modality. For example, in the Audio+Vision setting, we measure performance on both the audio benchmark and the vision benchmark. We use TextVQA~\cite{singh2019towards} for vision, MMAU~\cite{sakshi2024mmau} for audio, and Video-MME~\cite{fu2025video} for video. In addition, we examine a broader set of benchmarks, whose results are reported in Sec.~\ref{sec:additional_correlation_analysis} For each benchmark, we evaluate 300 samples that are equally sampled across categories to ensure balanced coverage. For video evaluation, we use tokens from only the first, middle, and last frames of each sample for computational efficiency, instead of processing all frames.

To provide a reference for the effect of merging itself, we additionally evaluate the original modality-specific models, i.e., the models trained on a single modality without merging, on the corresponding benchmarks. These reference results allow us to determine whether merging improves or degrades performance relative to the original specialized models.

For experiments, the proposed alignment metric is computed between modality tokens and text tokens at each backbone layer with $k=32$ and $\gamma=0.8$, and the alignment values are averaged over all samples in the benchmark dataset to obtain a layer-wise alignment profile. Since alignment varies substantially across layers, we analyze the results in a layer-wise manner, as shown in Fig.~\ref{fig:layerwise_alignment}. Consistent with prior observations that cross-modal alignment tends to be strongest in middle layers, we select the 21st layer as a representative layer and report the linear regression plots and correlation values based on this layer in Fig.~\ref{fig:correlation}.

\subsection{Comparison details}\label{sec:append_comparison_detail}
\subsubsection{Baselines}
\paragraph{Task Arithmetic.}
Task Arithmetic represents the effect of task-specific fine-tuning as a \emph{task vector} in weight space, defined by the difference between the fine-tuned model and the pretrained initialization~\cite{ilharco2022editing}. 
Given a pretrained model with parameters $\theta_0$ and modality- or task-specific models $\{\theta_i\}_{i=1}^{N}$ fine-tuned from the same initialization, each task vector is computed as $\tau_i = \theta_i - \theta_0$. 
Model merging is then performed by adding a scaled sum of task vectors back to the pretrained weights:
\begin{equation}
    \theta_{\mathrm{TA}}
    =
    \theta_0
    +
    \lambda
    \sum_{i=1}^{N}
    \left(\theta_i - \theta_0\right),
\end{equation}
where $\lambda$ controls the strength of the merged update. 
The original work shows that such vectors can be added or negated to compose or suppress task behaviors, enabling training-free model editing and multi-task merging at the weight level.

\paragraph{TIES-Merging.}
TIES-Merging addresses parameter interference that arises when merging multiple task-specific models fine-tuned from the same pretrained initialization~\citep{yadav2023ties}. 
Given task vectors $\tau_i = \theta_i - \theta_0$, TIES first trims small-magnitude updates by resetting them to zero, thereby removing redundant parameter changes. 
It then elects a shared sign for each parameter dimension and merges only the task-vector entries whose signs agree with the elected sign. 
A simplified form of the resulting merge can be written as
\begin{equation}
    \theta_{\mathrm{TIES}}
    =
    \theta_0
    +
    \lambda
    \mathbf{m},
    \qquad
    m_j
    =
    \frac{1}{|\mathcal{A}_j|}
    \sum_{i \in \mathcal{A}_j}
    \tilde{\tau}_{i,j},
    \quad
    \mathcal{A}_j
    =
    \left\{
    i :
    \mathrm{sign}(\tilde{\tau}_{i,j})
    =
    \mathrm{sign}\!\left(\sum_{k=1}^{N} \tilde{\tau}_{k,j}\right)
    \right\},
\end{equation}
where $\tilde{\tau}_i$ denotes the trimmed task vector and $\lambda$ controls the merge strength. 
By resolving sign disagreement before averaging, TIES-Merging reduces destructive interference and improves training-free multi-task model merging compared to direct weight averaging or task-vector addition.
\paragraph{Isotropic Merging (ISO-C).}
Isotropic Merging analyzes model merging from the perspective of the singular subspaces of task-specific weight updates~\citep{marczak2025no}. 
Given layer-wise task matrices $\Delta_i = W_i - W_0$, it first constructs a merged task matrix $\Delta = \sum_{i=1}^{N}\Delta_i$ and decomposes it as $\Delta = U\;\mathrm{diag}(\sigma)\cdot V^\top$. 
Instead of preserving the skewed singular value spectrum of the summed update, Iso-C replaces the singular values with a uniform scale, thereby balancing the contribution of different singular directions. 
Iso-CTS further augments this common subspace with task-specific residual directions, which helps retain underrepresented task information when merging many diverse models. 
The resulting merge can be written as
\begin{equation}
    W_{\mathrm{Iso}}
    =
    W_0
    +
    \lambda
    U_{\mathrm{CTS}}
    \left(
    \bar{\sigma} I
    \right)
    V_{\mathrm{CTS}}^\top,
    \qquad
    \bar{\sigma}
    =
    \frac{1}{r}
    \sum_{j=1}^{r}
    \sigma_j,
\end{equation}
where $U_{\mathrm{CTS}}$ and $V_{\mathrm{CTS}}$ denote the orthogonalized bases obtained from common and task-specific subspaces, and $\lambda$ controls the merge strength.

\paragraph{Task Singular Vectors.}
Task Singular Vectors (TSV) analyzes task vectors at the layer level by preserving each task-specific update as a matrix rather than flattening the entire network into a single vector~\citep{gargiulo2025task}. 
Given a pretrained layer $W_0^\ell$ and task-specific layers $\{W_i^\ell\}_{i=1}^{N}$, TSV first decomposes each layer task matrix $\Delta_i^\ell = W_i^\ell - W_0^\ell$ using SVD and retains only the top-$r$ singular components. 
The method observes that overlap between singular vectors from different tasks induces task interference, and therefore TSV-Merge reduces such interference by decorrelating the concatenated left and right singular-vector bases through whitening or an equivalent orthogonal Procrustes step. 
A simplified layer-wise merge can be written as
\begin{equation}
    W_{\mathrm{TSV}}^\ell
    =
    W_0^\ell
    +
    \lambda
    \hat{U}^\ell
    \bar{\Sigma}^\ell
    (\hat{V}^\ell)^\top .
\end{equation}
where $\hat{U}^\ell$ and $\hat{V}^\ell$ denote the decorrelated versions of $\bar{U}^\ell$ and $\bar{V}^\ell$, and $\lambda$ controls the merge strength. 
By combining low-rank approximation with singular-vector decorrelation, TSV-Merge aims to reduce destructive interference among task-specific updates while keeping the merge training-free.

\paragraph{WUDI-Merging.} WUDI-Merging is a data-free model merging method that reduces interference among task-specific experts by exploiting the observation that linear-layer task vectors approximately span the corresponding input subspace~\citep{cheng2025whoever}. 
Given a pretrained model $\theta$ and task vectors $\tau_i=\theta_i-\theta$, WUDI focuses on the linear layers and directly optimizes the merged task vector $\tau_{m,l}$ to stay compatible with each task vector $\tau_{i,l}$. 
For each linear layer $l$, the optimization objective is written as
\begin{equation}
\begin{aligned}
    \mathcal{L}_l(\tau_{m,l})
    &=
    \sum_i
    \frac{1}{\|\tau_{i,l}\|_F^2}
    \left\|
    (\tau_{m,l}-\tau_{i,l})\tau_{i,l}^{\top}
    \right\|_F^2,\quad \theta_m=\theta + \tau_m .
\end{aligned}
\end{equation}
The merged task vectors are optimized independently for each linear layer and then assembled into the final merged model. 
By using task vectors themselves as guidance signals, WUDI-Merging mitigates destructive interference without requiring additional training data, test-time samples, or extra task-specific modules.

\paragraph{OptMerge.}
OptMerge studies model merging for MLLMs and proposes a data-free optimization method for combining task- or modality-specific expert models~\citep{wei2025optmerge}. 
Given task vectors $\tau_i=\theta_i-\theta_0$, OptMerge observes that directly optimizing the merged task vector can be unstable due to redundant and noisy components in task-specific updates. 
For each linear layer $l$, it first applies a low-rank approximation to the centered task matrix, $\tau_{i,l}-\bar{\tau}_l \approx U_{i,k}\Sigma_{i,k}V_{i,k}^{\top}$, and uses the denoised update $\tilde{\tau}_{i,l}=U_{i,k}\Sigma_{i,k}V_{i,k}^{\top}+\bar{\tau}_l$ for optimization. 
The layer-wise merged vector is obtained by minimizing task-vector interactions:
\begin{equation}
\begin{aligned}
    \tau_{m,l}^{*}
    &=
    \arg\min_{\tau_{m,l}}
    \sum_{i=1}^{N}
    \frac{
    \left\|
    (\tau_{m,l}-\tilde{\tau}_{i,l})
    S_{i,l}^{\top}
    \right\|_{F}^{2}
    }{
    \|\tau_{i,l}\|_{F}^{2}
    },
    \quad
    S_{i,l}=\Sigma_{i,k}V_{i,k}^{\top},\quad \theta_m=\theta_0+\tau_m .
\end{aligned}
\end{equation}
For LoRA fine-tuned models, OptMerge further uses direct truncated SVD, mean initialization, and SGD-based optimization to avoid excessive growth of the merged vector norm.

\subsubsection{Base models}
\paragraph{Vicuna-based MLLMs.}
We use the Vicuna-based modality-specific MLLMs released by OptMerge~\citep{wei2025optmerge}. 
All models share Vicuna-7B-v1.5 as the language backbone, while employing modality-specific encoders and connectors: the vision-language model uses CLIP-ViT-L-336px~\cite{radford2021learning} with an MLP projector, the audio-language model uses BEATs-Iter3+~\cite{liu2025ola} with a Q-Former~\cite{li2023blip} connector, and the video-language model uses LanguageBind~\cite{zhu2023languagebind} with an MLP projector. 
These models are trained on bi-modal modality-text data for vision, audio, and video, respectively, following a two-stage procedure consisting of connector alignment and subsequent fine-tuning. 
Following the original modality-merging setup, we preserve each modality-specific encoder and connector, and apply merging only to the shared language-model parameters. 
This setup allows us to evaluate whether cross-modality merging can transfer alignment across different modality-specific MLLMs while keeping the underlying language backbone fixed.

\paragraph{LLaMA 3.1-based MLLMs.}
We further evaluate DCAT on LLaMA 3.1-based MLLMs collected from publicly available in-the-wild checkpoints. 
Specifically, we use LLaVA-LLaMA3.1-8B\footnote{\url{https://huggingface.co/modelscope/llava-llama3.1-8b}} as the vision-language model, Ultravox-v0.4.1-LLaMA-3.1-8B\footnote{\url{https://huggingface.co/fixie-ai/ultravox-v0_4_1-llama-3_1-8b}} as the audio-language model, and LLaVA-Video-LLaMA-3.1-8B\footnote{\url{https://huggingface.co/weizhiwang/LLaVA-Video-Llama-3.1-8B}} as the video-language model. 
Unlike the Vicuna-based models, these checkpoints are independently developed and trained with different data mixtures, training recipes, and modality-specific components, while sharing the same LLaMA 3.1-8B language-model architecture. 
For each model, we preserve its original modality encoder and connector, and apply merging only to the language-model parameters that are architecturally compatible across models. 
This setting provides a more practical and less controlled testbed for evaluating whether DCAT can robustly transfer cross-modal alignment across heterogeneous real-world MLLMs.

\paragraph{Considered, but not experimented models.}
Our merging setup assumes that the modality-specific MLLMs share the same language-model architecture and parameter shapes, since both DCAT and the baseline merging methods operate directly at the weight level. 
We also considered other in-the-wild MLLMs that appear to share the same backbone family, such as Qwen-Audio~\cite{chu2023qwen} and LLaVA-Next~\cite{li2024llava} with Qwen-1.5~\cite{bai2023qwen} backbones. 
However, these models often differ in tokenizer size, input embedding size, and LM head size, which prevents direct weight-space merging across all compared methods without additional architectural modification or retraining. 
Similarly, several recent public MLLMs, including Idefics3-8B-Llama3\footnote{\url{https://huggingface.co/HuggingFaceM4/Idefics3-8B-Llama3}} and Llama-3.1-Nemotron-Nano-VL-8B-V1\footnote{\url{https://huggingface.co/nvidia/Llama-3.1-Nemotron-Nano-VL-8B-V1}}, were considered but excluded because their backbone implementations or parameter shapes are not fully compatible with the other modality-specific models. 
We did not introduce additional training or ad-hoc resizing procedures to force compatibility, as this would confound the comparison among training-free merging methods. 
To the best of our knowledge, our experiments cover the publicly available in-the-wild MLLMs that satisfy the required architectural compatibility conditions; we provide a more detailed discussion of this limitation in Sec.~\ref{sec:limitation}.

\subsubsection{Benchmarks}
To compare the overall performance across various tasks, we evaluated the baseline and merged models on multiple audio and video question-answering benchmarks. For audio, we used MMAU~\cite{sakshi2024mmau}, AIR-Bench-Foundation~\cite{yang2024air}, MuChoMusic~\cite{weck2024muchomusic}, Clotho-AQA~\cite{lipping2022clotho}, and VocalSound~\cite{gong_vocalsound}. For video, we used Video-MME~\cite{fu2025video}, TempCompass~\cite{liu2024tempcompass}, MVBench~\cite{li2024MVBench}, MLVU~\cite{zhou2024mlvu}, EgoSchema~\cite{mangalam2023egoschema}, and LongVideoBench~\cite{wu2024longvideobench}.

MMAU and AIR-Bench-Foundation consist of 1,000 and 24,683 multiple-choice questions (MCQs), respectively, paired with audio clips covering natural sounds, speech, and music. MuChoMusic contains 1,187 MCQs designed to evaluate the music understanding capability of models. For Clotho-AQA, we used a subset consisting of 1,442 binary QA questions related to audio. VocalSound contains 3,591 questions for classifying six types of human vocal sounds. We used a custom prompt to prevent the model from generating answers outside the given choices.

Video-MME contains 2,700 MCQs based on 900 videos with diverse durations and categories. For TempCompass, we used the multiple-choice, caption-matching, and yes/no subsets, which contain 1,580, 1,503, and 2,453 questions, respectively. These subsets evaluate the temporal understanding capability of models across five categories: action, attribute change, direction, event order, and speed. MVBench consists of 20 diverse question sets for evaluating temporal perception, with 200 questions in each set. MLVU consists of 2,174 MCQs that evaluate long-video understanding ability across various genres. For EgoSchema, we used the subset containing 500 long-form video MCQ pairs. LongVideoBench consists of 1,337 MCQs designed to evaluate long-context video understanding. Overall results for these benchmarks, except EgoSchema and LongVideoBench, are reported in Table~\ref{tab:quant}, and detailed results for specific categories in MVBench and TempCompass are provided in Appendix~\ref{sec:append_subcategory}.

Additionally, we used GQA~\cite{hudson2019gqa}, MMBench~\cite{liu2024mmbench}, MMStar~\cite{chen2024we}, POPE~\cite{li2023evaluating}, and TextVQA~\cite{singh2019towards} to evaluate the vision capabilities of the models. The GQA testdev-balanced split contains 12,578 real-world visual question-answering examples. MMBench (Dev\_EN) and MMStar contain 1,164 and 1,500 questions, respectively, for evaluating various capabilities of vision-language models. POPE consists of 9,000 binary QA pairs designed to evaluate object hallucination. TextVQA (Val) consists of 5,000 questions for evaluating a model's ability to read text in images. Quantitative results for the vision benchmarks are provided in Appendix~\ref{sec:append_vision_benchmark}.

Each benchmark was evaluated using the lmms-eval~\cite{zhang2024lmmsevalrealitycheckevaluation,lmms_eval2024} toolkit. However, some prompts or evaluation methods were customized for our evaluation. Example prompts and evaluation methods are summarized in Table~\ref{tab:benchmark-summary}.

\section{Additional Analysis}\label{sec:append_analysis}
\subsection{Additional correlation analysis}\label{sec:additional_correlation_analysis}
To further assess the generality of the relationship between the proposed Alignment score and modality performance, we extend the correlation analysis to a broader set of benchmarks. For vision, we additionally consider GQA~\cite{hudson2019gqa}, OCR-VQA~\cite{mishra2019ocr}, MathVision~\cite{wang2024measuring}, and VizWiz~\cite{gurari2018vizwiz}; for audio, we use VocalSound~\cite{gong_vocalsound}, MuCho-music~\cite{weck2024muchomusic}, Clotho-AQA~\cite{lipping2022clotho}; and for video, we use TempCompass~\cite{liu2024tempcompass}, MVBench~\cite{li2024MVBench}, and MLVU~\cite{zhou2024mlvu}. The experimental protocol follows Sec.~\ref{sec:append_analysis_details}. In particular, we perform the analysis for both Audio+Vision and Video+Vision settings on both Vicuna and LLaMA 3.1 backbones. For each merged setting, we evaluate the benchmarks corresponding to each constituent modality. As in the main analysis, we examine the relationship between Alignment score and modality performance both layer-wise and at the representative middle layer (21st layer).

The layer-wise results are shown in Fig.~\ref{fig:supple_layerwise_corr_vicuna}-~\ref{fig:supple_layerwise_corr_llama}, and the corresponding 21st-layer regression and correlation results are presented in Fig.~\ref{fig:supple_21th_corr_vicuna}-\ref{fig:supple_21th_corr_llama}. Across all additional benchmarks and backbone settings, we observe a consistent trend. The strongest positive correlation between the Alignment score and performance appears in the middle layers, in agreement with the observations from the main analysis. Moreover, both Pearson and Spearman correlation coefficients between alignment and accuracy are positive across all cases in the 21st layer, indicating that higher alignment is consistently associated with better benchmark accuracy. These results further support that the proposed Alignment score is not specific to a particular benchmark or backbone, but instead provides a robust indicator of model quality across diverse evaluation settings.

\begin{figure}[h]
  \centering
  \vspace{-1.0em}
  \includegraphics[width=\linewidth]{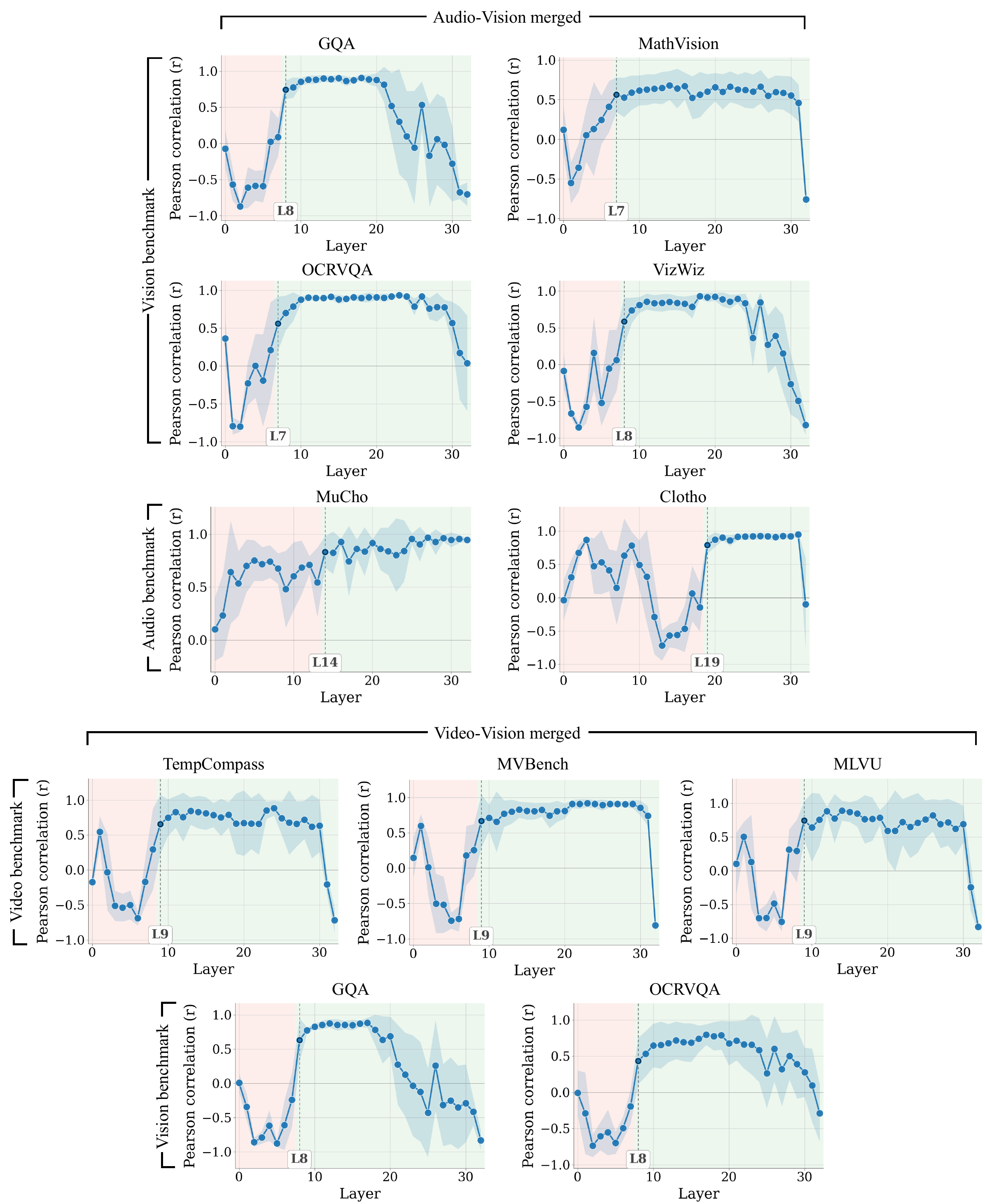}
  \caption{\textbf{Layer-wise Pearson correlation between Alignment score and modality performance on Vicuna backbone.} We measure the Alignment score in a layer-wise manner for vision, audio, and video benchmark on Audio+Vision and Video+Vision Vicuna backbone. Across all cases, the correlation between Alignment score and accuracy is strongest in the middle layers.}
  \label{fig:supple_layerwise_corr_vicuna}
\end{figure}

\begin{figure}[h]
  \centering
  \vspace{-1.0em}
  \includegraphics[width=\linewidth]{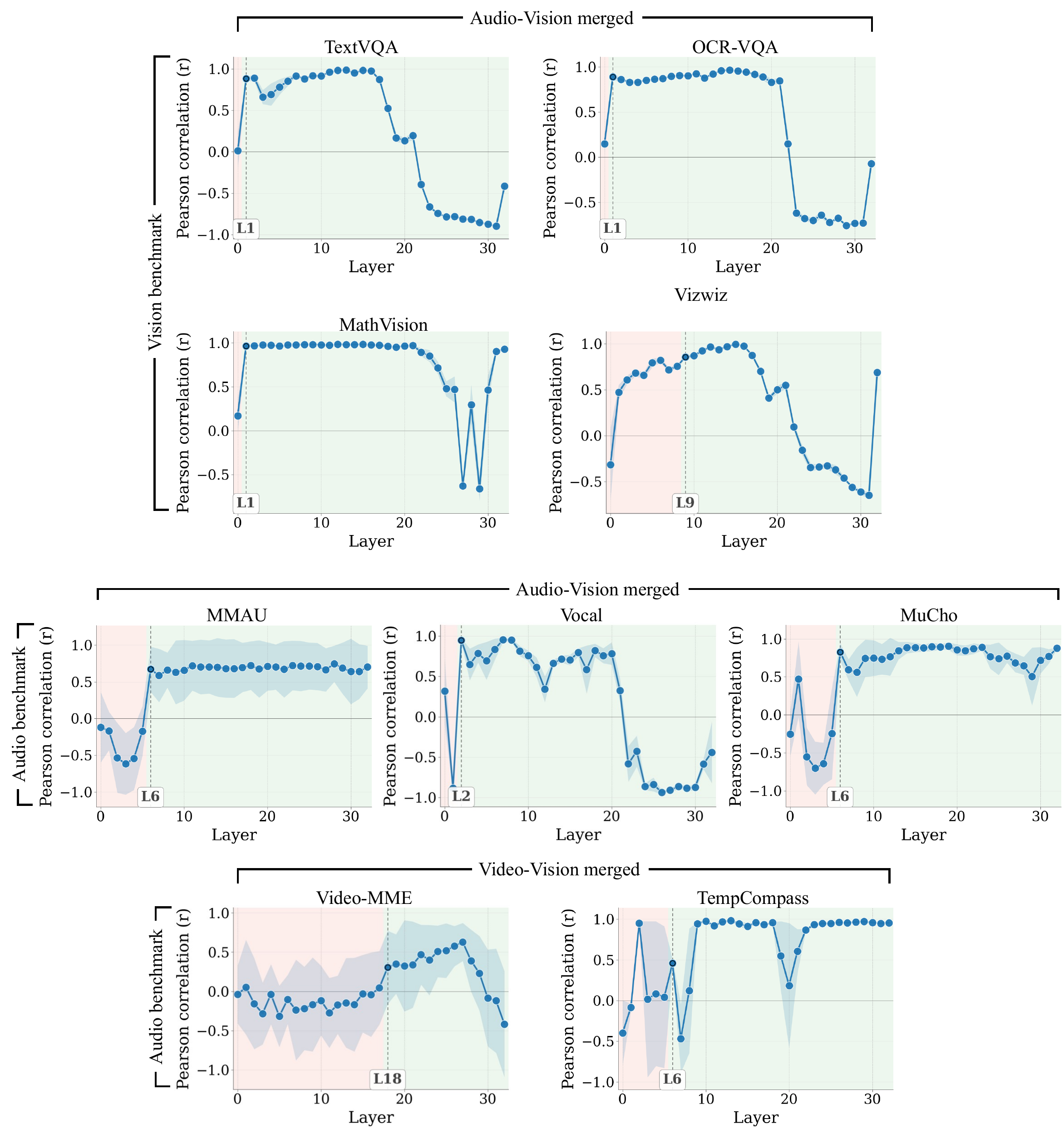}
  \caption{\textbf{Layer-wise Pearson correlation between Alignment score and modality performance on LLaMA-3.1.} We measure the Alignment score in a layer-wise manner for vision, audio, and video on Audio+Vision and Video+Vision LLaMa backbone. Across all cases, the correlation between the Alignment score and accuracy is strongest in the middle layers.}
  \label{fig:supple_layerwise_corr_llama}
\end{figure}

\begin{figure}[h]
  \centering
  \vspace{-1.0em}
  \includegraphics[width=\linewidth]{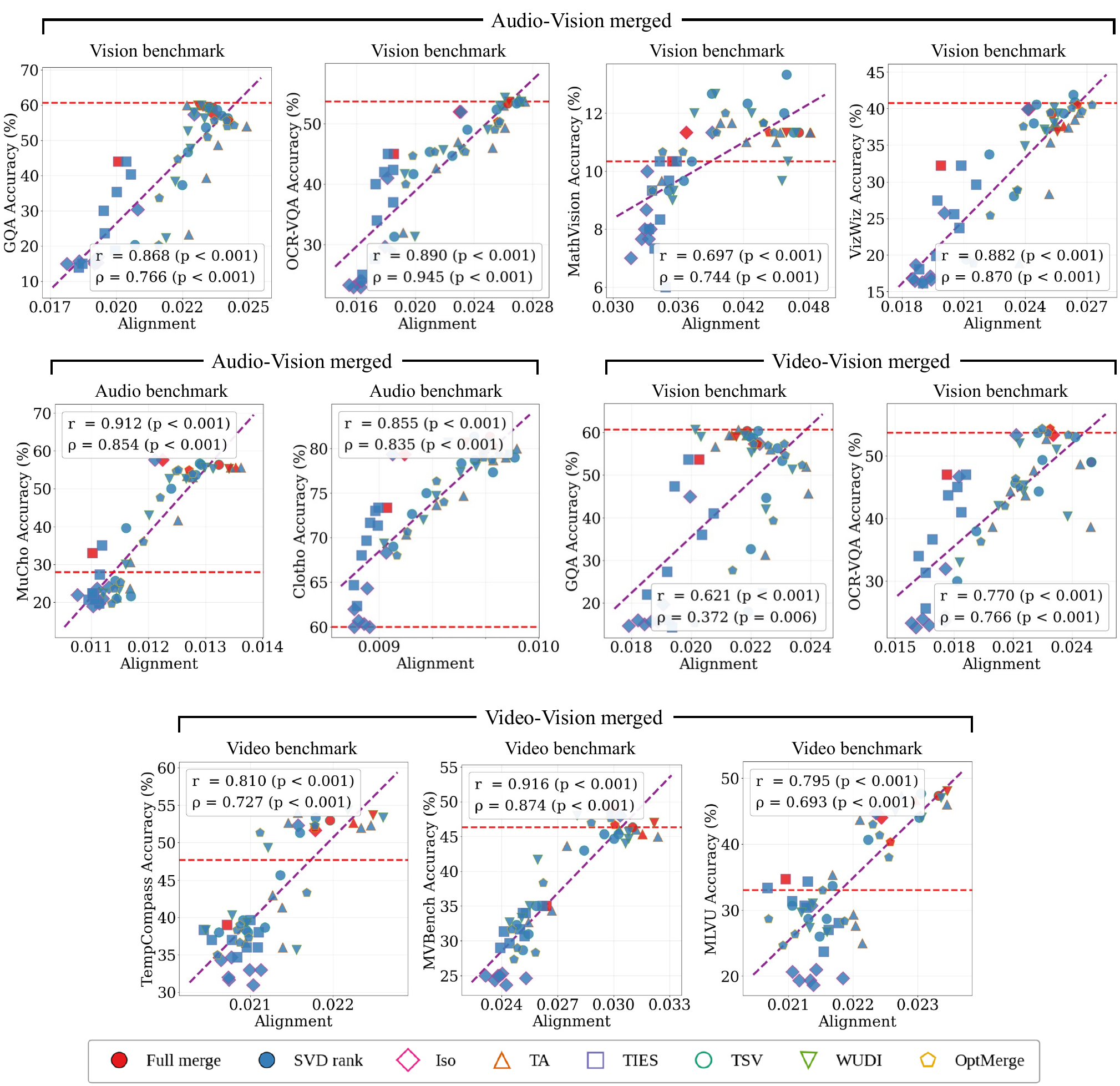}
    \caption{\textbf{Correlation between Alignment score and modality performance on the 21st layer of Vicuna backbone.} We analyze the correlation between the Alignment score and modality performance across checkpoints obtained from merging methods~\cite{ilharco2022editing, yadav2023ties, marczak2025no, gargiulo2025task, cheng2025whoever, wei2025optmerge}, and their SVD-derived variants. The red dotted line refers to the performance of original model. Across all results, the Alignment score shows a positive correlation with modality performance under both Pearson and Spearman measures.}
  \label{fig:supple_21th_corr_vicuna}
\end{figure}

\begin{figure}[h]
  \centering
  \includegraphics[width=\linewidth]{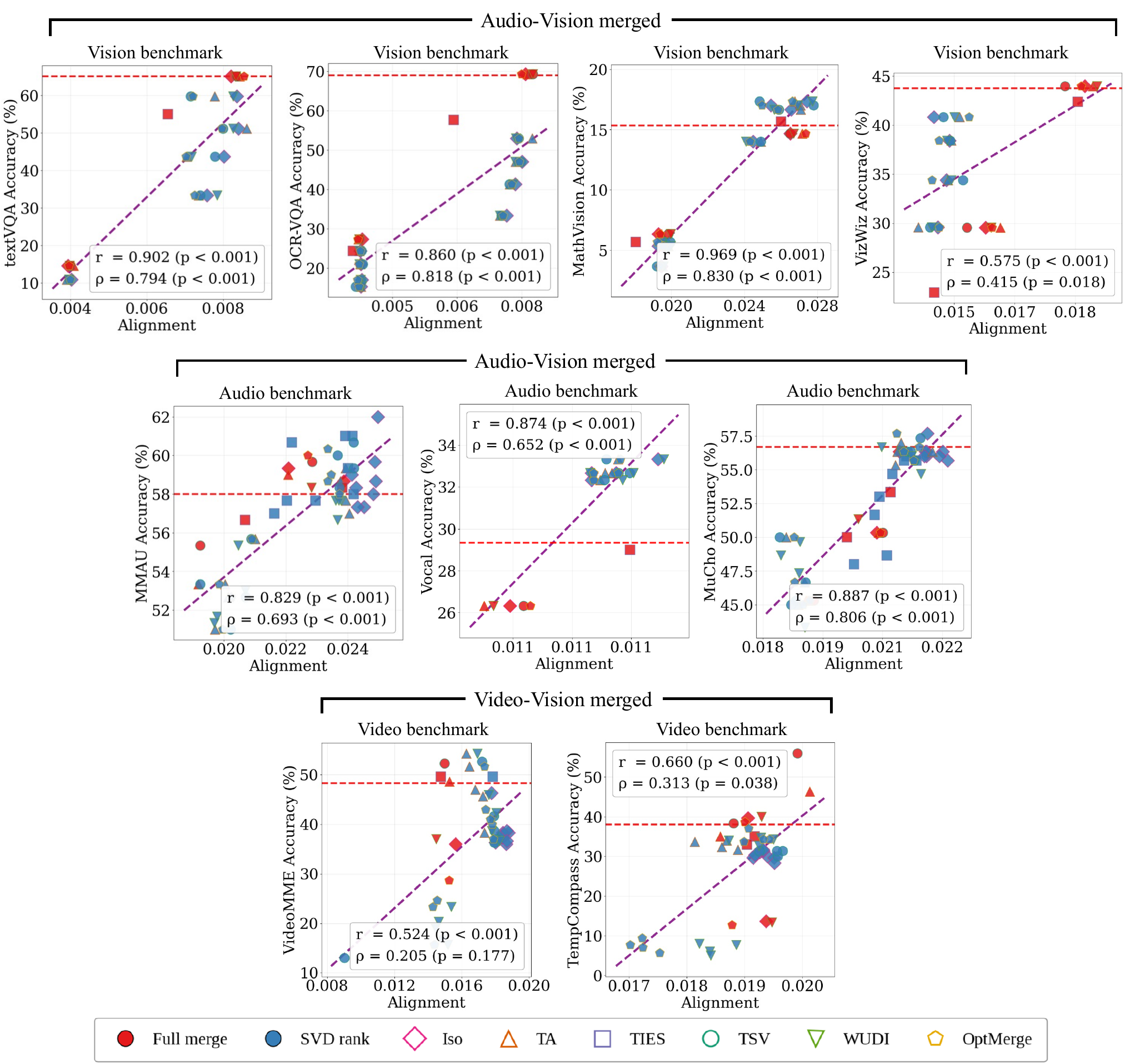}
    \caption{\textbf{Correlation between Alignment score and modality performance on the 21st layer of LLaMA-3.1 backbone.} We analyze the correlation between the Alignment score and modality performance across checkpoints obtained from merging methods~\cite{ilharco2022editing, yadav2023ties, marczak2025no, gargiulo2025task, cheng2025whoever, wei2025optmerge}, and their SVD-derived variants. The red dotted line refers to the performance of original model. Across all results, the Alignment score shows a positive correlation with modality performance under both Pearson and Spearman measures.}
    \label{fig:supple_21th_corr_llama}
  \vspace{-1.5em}
\end{figure}

\FloatBarrier

\subsection{Additional layer-wise SEO and SD analysis}
To further examine how DCAT affects the internal structure of the merged model, we conduct an additional layer-wise analysis of SEO and SD at the level of individual transformer components. Specifically, for the Vicuna backbone, we compare SEO and SD before and after applying DCAT for each major submodule in both the MLP and attention blocks, including the MLP projections (\texttt{gate}, \texttt{down}, and \texttt{up}) and the attention projections (\texttt{q}, \texttt{k}, \texttt{v}, and \texttt{o}). This analysis is performed on the merged Vicuna models in the Audio+Vision and Video+Vision settings.

The results are shown in Fig.~\ref{fig:supple_additional_seo_sd_av} and Fig.~\ref{fig:supple_additional_seo_sd_vdv}. Across all examined components and across both merged models, applying DCAT consistently increases both SEO and SD. This trend is observed not only at the model level but also throughout the individual layers and their constituent projections, indicating that the effect of DCAT is broadly distributed across the transformer rather than being confined to a small subset of parameters. These results provide additional empirical support for our theoretical motivation, showing that DCAT systematically encourages the structural properties associated with a higher mutual-information lower bound throughout the network.

\begin{figure}[h]
  \centering
  \vspace{-1.0em}
  \includegraphics[width=\linewidth]{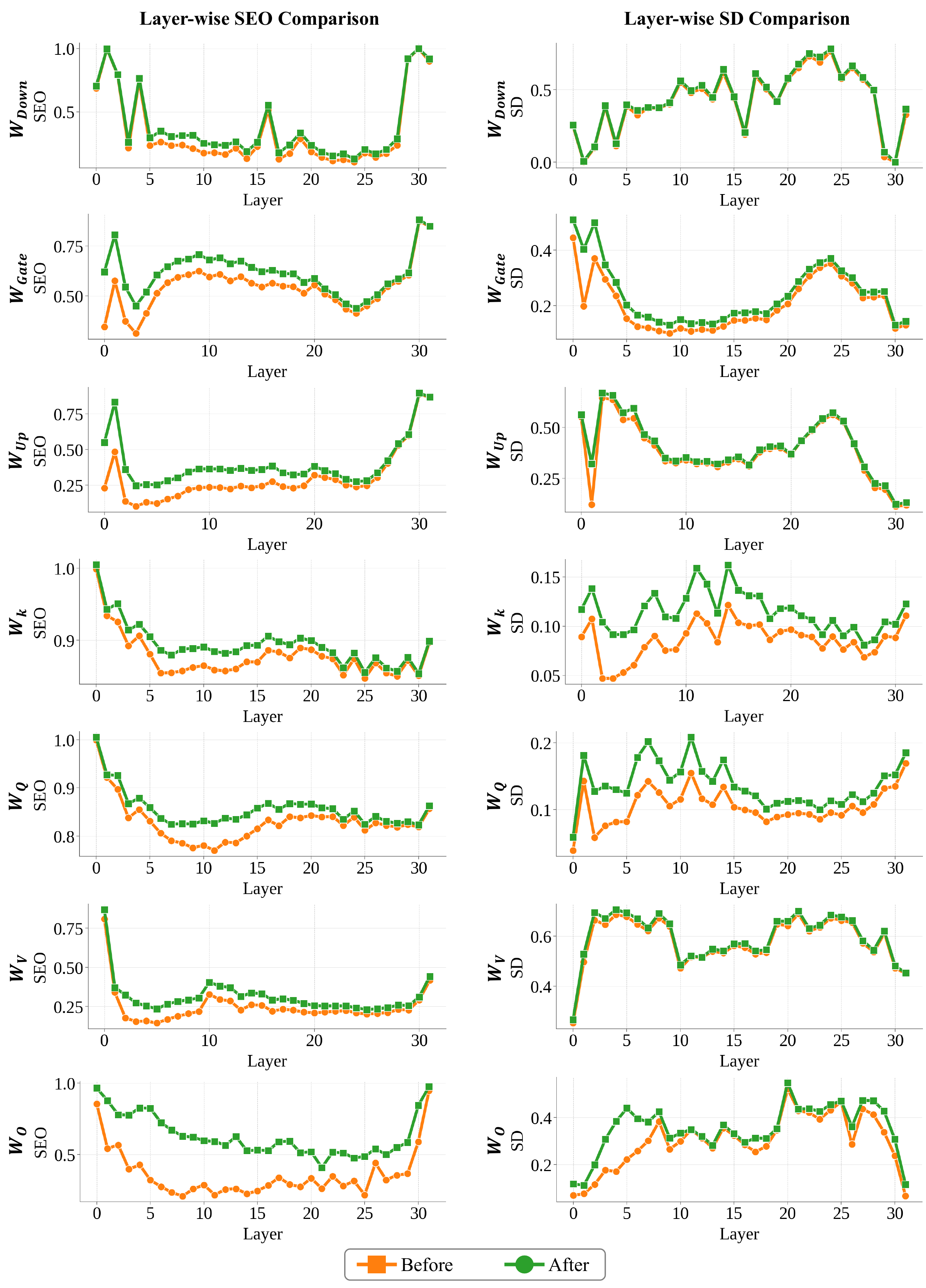}
    \caption{\textbf{Layer-wise SEO and SD of individual transformer submodules before and after applying DCAT on the Audio+Vision Vicuna.} We compare SEO and SD for the MLP projections ($W_{\mathrm{gate}}$, $W_{\mathrm{down}}$, and $W_{\mathrm{up}}$) and the attention projections ($W_{k}$, $W_{q}$, $W_{v}$, and $W_{o}$) across all Vicuna layers. In all examined components, applying DCAT consistently increases both SEO and SD.}
    \label{fig:supple_additional_seo_sd_av}
  \vspace{-1.5em}
\end{figure}

\begin{figure}[h]
  \centering
  \vspace{-1.0em}
  \includegraphics[width=\linewidth]{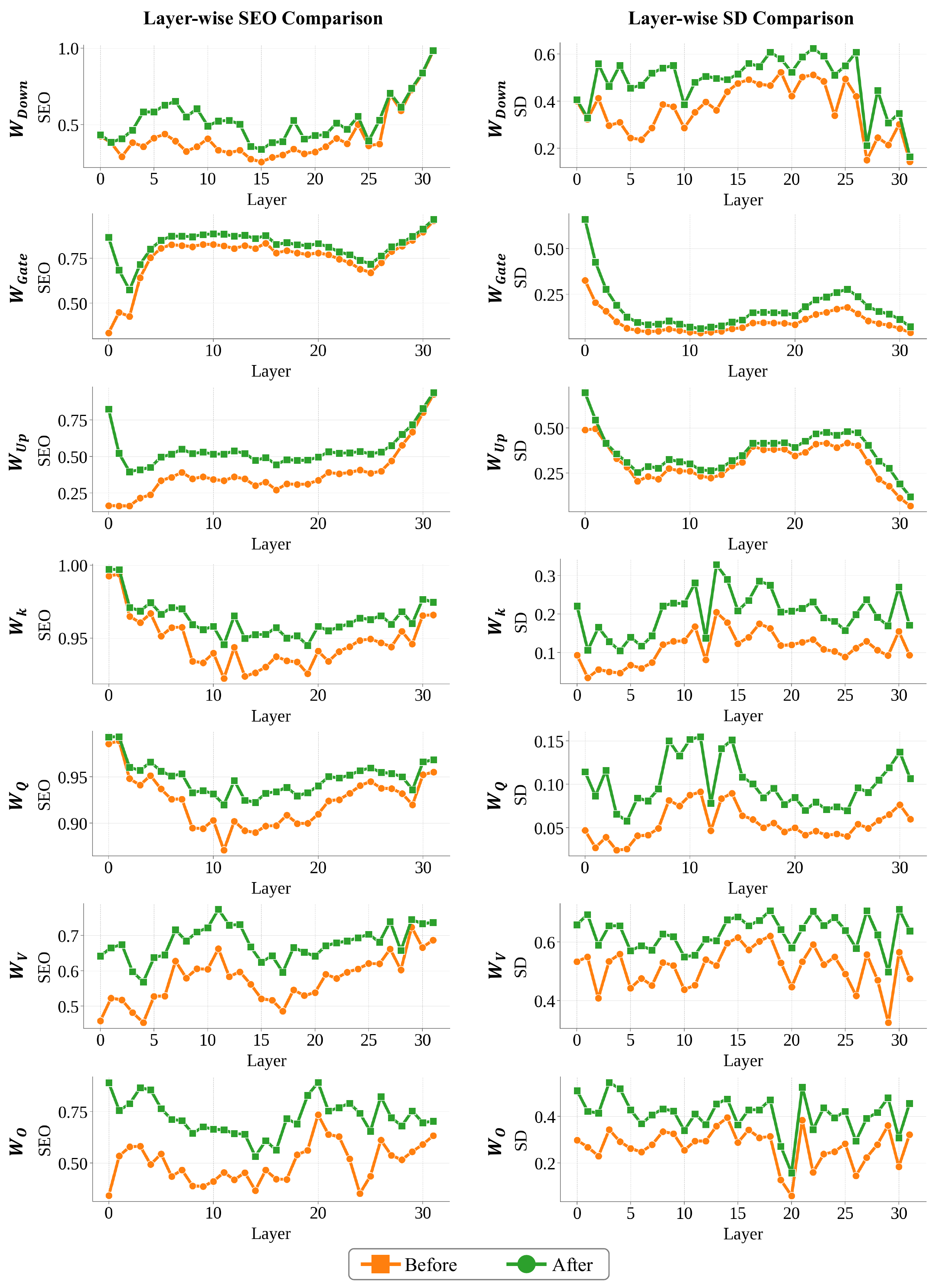}
    \caption{\textbf{Layer-wise SEO and SD of individual transformer submodules before and after applying DCAT on the Video+Vision Vicuna.} We compare SEO and SD for the MLP projections ($W_{\mathrm{gate}}$, $W_{\mathrm{down}}$, and $W_{\mathrm{up}}$) and the attention projections ($W_{k}$, $W_{q}$, $W_{v}$, and $W_{o}$) across all Vicuna layers. In all examined components, applying DCAT consistently increases both SEO and SD.}    \label{fig:supple_additional_seo_sd_vdv}
  \vspace{-1.5em}
\end{figure}

\FloatBarrier

\begin{figure}[h]
  \centering
  \vspace{-1.0em}
  \includegraphics[width=\linewidth]{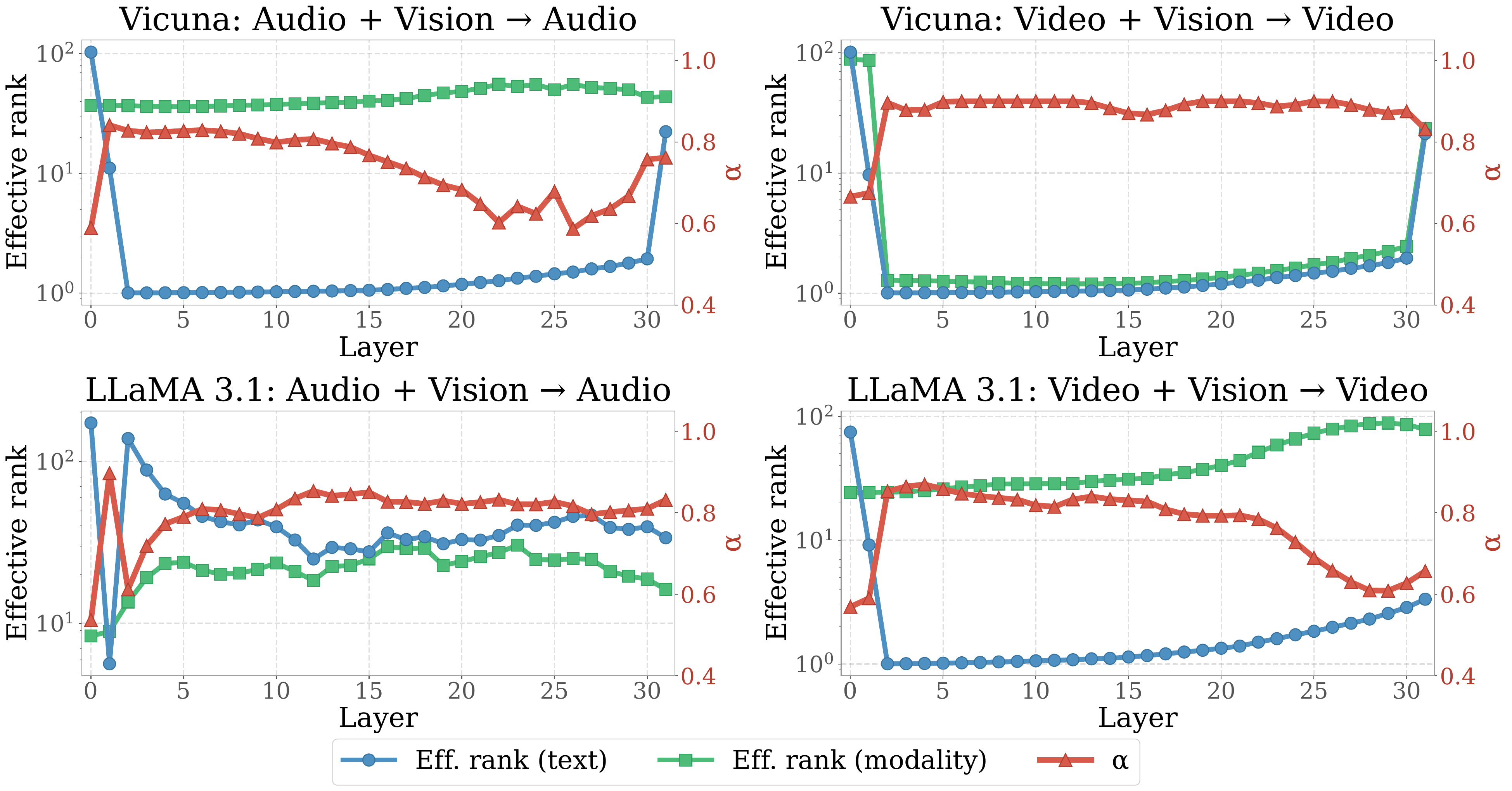}
  \caption{\textbf{Layer-wise effective rank and scheduling of $\alpha$.} DCAT schedules $\alpha$ inversely with the text effective rank, as small text effective rank with large $\alpha$ could introduce noisy alignment between text and modality-token activations.}
  \label{fig:plot_eff_rank_alpha}
  \vspace{-0.5em}
\end{figure}

\subsection{Rationale of scheduling $\alpha$}

DCAT adjusts $\alpha$ based on the relative changes in SEO and SD, as described in Eq.~\ref{eq:alpha_schedule}.
This design is motivated by the layer-wise structure of the text subspace.
As shown in Fig.~\ref{fig:plot_eff_rank_alpha}, the effective rank of text-token activations generally decreases in the middle layers, although the degree of this trend varies across architectures and modalities.
When the text subspace is effectively low-rank, aggressively increasing SD may spread modality information over weak or noisy text-space directions, leading to noisy alignment rather than reliable modality-text alignment.
To mitigate this risk, DCAT uses the relative behavior of SEO and SD to adaptively schedule $\alpha$, making the target more conservative when the reliable text subspace is limited.

Empirically, this schedule induces an inverse relationship between the text effective rank and $\alpha$.
In layers where the text effective rank is small, $\alpha$ becomes larger, encouraging DCAT to emphasize SEO and preserve the recipient's projected spectrum rather than over-promoting SD.
Conversely, when the text effective rank is larger, $\alpha$ decreases and allows stronger encouragement of SD.
Although the LLaMA 3.1-based audio model shows a relatively different effective-rank pattern compared with the other settings, the scheduled $\alpha$ still follows the intended inverse trend with respect to the text effective rank.

\section{Additional Results}

\subsection{Statistical significance of presented results}\label{sec:append_statistical}
To assess whether the quantitative improvements reported in the main text are statistically significant, we additionally compute the performance gain relative to the original modality-specific model and the corresponding $p$-value for each benchmark in the Audio+Vision and Video+Vision Vicuna settings. Specifically, for each merging method (\cite{ilharco2022editing, yadav2023ties, marczak2025no, gargiulo2025task, cheng2025whoever, wei2025optmerge}), including ours, we report the performance difference in percentage points ($\Delta$, \%p) with respect to the original model, together with its statistical significance.

We evaluate statistical significance using McNemar's test, which is appropriate in our setting because all compared models are evaluated on the same paired samples from each benchmark. Since each sample yields a matched pair of outcomes between two models (correct or incorrect), McNemar's test provides a natural way to determine whether the observed difference in accuracy is statistically meaningful beyond aggregate score differences.

The results in Table~\ref{tab:quant_audio} for Audio+Vision and Table~\ref{tab:quant_video} for Video+Vision show that our method yields statistically significant improvements across all reported settings. In the tables, $p$-values below 0.05 are highlighted in yellow, and those below 0.01 are highlighted in red. While baseline merging methods sometimes improve performance, their gains are not consistently significant and vary depending on the benchmark and backbone. In contrast, our method achieves positive improvements with statistically significant $p$-values in all cases, demonstrating that the performance gains reported in the main text are not only numerically favorable but also consistently supported at the sample level.

\begin{table}[]
    \caption{\textbf{$p$-value of DCAT on audio+vision}}
    \label{tab:quant_audio}
    \centering
    \resizebox{\textwidth}{!}{
    \begin{tabular}{cl|cc|cc|cc|cc|cc}
    \toprule
         & & \multicolumn{2}{c|}{MMAU} & \multicolumn{2}{c|}{AIR-Bench} & \multicolumn{2}{c|}{MuCho} & \multicolumn{2}{c|}{Clotho} & \multicolumn{2}{c}{Vocal} \\
         & & $\Delta$ (\%p) & $p$-value & $\Delta$ (\%p) & $p$-value & $\Delta$ (\%p) & $p$-value & $\Delta$ (\%p) & $p$-value & $\Delta$ (\%p) & $p$-value \\
    \midrule
    \multirow{7}{*}{\rotatebox{90}{Vicuna-based}} & TA & +4.30 & \cellcolor{tabthird}1.71e-02 & +5.91 & \cellcolor{tabfirst}3.62e-61 & +23.08 & \cellcolor{tabfirst}8.57e-29 & +19.76 & \cellcolor{tabfirst}4.14e-36 & -4.18 & \cellcolor{tabfirst}1.75e-07\\
    & TIES & +2.10 & 1.78e-01 & +5.50 & \cellcolor{tabfirst}3.74e-78 & +11.46 & \cellcolor{tabfirst}7.79e-17 & +13.73 & \cellcolor{tabfirst}1.83e-24 & +1.95 & \cellcolor{tabthird}1.55e-02\\
    & ISO & +4.50 & \cellcolor{tabfirst}5.57e-03 & +6.62 & \cellcolor{tabfirst}3.53e-88 & +25.86 & \cellcolor{tabfirst}1.64e-41 & +17.34 & \cellcolor{tabfirst}2.99e-33 & -4.68 & \cellcolor{tabfirst}2.86e-13\\
    & TSV & +3.80 & \cellcolor{tabthird}2.80e-02 & +7.34 & \cellcolor{tabfirst}3.15e-112 & +23.84 & \cellcolor{tabfirst}1.65e-33 & +19.97 & \cellcolor{tabfirst}4.21e-37 & -1.45 & \cellcolor{tabthird}4.50e-02\\
    & WUDI & +4.80 & \cellcolor{tabfirst}6.74e-03 & +6.06 & \cellcolor{tabfirst}5.07e-64 & +22.49 & \cellcolor{tabfirst}1.35e-27 & +19.49 & \cellcolor{tabfirst}1.09e-34 & -2.01 & \cellcolor{tabfirst}9.10e-03\\
    & OptMerge & +4.20 & \cellcolor{tabfirst}7.74e-03 & +7.33 & \cellcolor{tabfirst}6.15e-127 & +23.25 & \cellcolor{tabfirst}5.65e-37 & +19.00 & \cellcolor{tabfirst}1.97e-36 & +0.08 & 9.39e-01\\
    & \textbf{DCAT (Ours)} & +7.30 & \cellcolor{tabfirst}4.12e-04 & +9.27 & \cellcolor{tabfirst}2.42e-116 & +27.21 & \cellcolor{tabfirst}7.82e-39 & +21.73 & \cellcolor{tabfirst}2.27e-43 & +2.27 & \cellcolor{tabfirst}5.83e-03\\
    \midrule
    \multirow{13}{*}{\rotatebox{90}{LLaMA 3.1-based}}  & TA$^*$ ($\lambda=1.0$)  & -2.90 & 8.41e-02 & -3.18 & \cellcolor{tabfirst}2.78e-17 & -2.95 & 1.01e-01 & +1.60 & 2.62e-01 & -3.45 & \cellcolor{tabfirst}3.47e-04\\
    & TIES ($\lambda=1.0$) & +0.50 & 7.85e-01 & -0.43 & 2.17e-01 & +2.36 & 1.67e-01 & +2.08 & 5.99e-02 & -0.53 & 4.11e-01\\
    & ISO ($\lambda=1.0$) & +2.10 & 1.52e-01 & -2.23 & \cellcolor{tabfirst}5.00e-11 & +3.37 & \cellcolor{tabthird}4.42e-02 & +3.19 & \cellcolor{tabfirst}1.48e-03 & -1.67 & \cellcolor{tabfirst}1.26e-06\\
    & TSV$^*$ ($\lambda=1.0$) & -2.90 & 8.41e-02 & -3.18 & \cellcolor{tabfirst}2.78e-17 & -2.95 & 1.01e-01 & +1.60 & 2.62e-01 & -3.45 & \cellcolor{tabfirst}3.47e-04\\
    & WUDI ($\lambda=1.0$) & -1.90 & 2.69e-01 & -3.30 & \cellcolor{tabfirst}2.47e-18 & -2.61 & 1.43e-01 & +1.80 & 2.12e-01 & -4.12 & \cellcolor{tabfirst}1.80e-05\\
    & OptMerge$^*$ ($\lambda=1.0$) & -2.90 & 8.41e-02 & -3.18 & \cellcolor{tabfirst}2.78e-17 & -2.95 & 1.01e-01 & +1.60 & 2.62e-01 & -3.45 & \cellcolor{tabfirst}3.47e-04\\
    \cmidrule{2-12}
    & TA$^*$ ($\lambda=0.2$)  & +1.70 & 2.21e-01 & -0.04 & 8.92e-01 & +1.94 & 2.17e-01 & +2.08 & \cellcolor{tabfirst}7.85e-03 & +0.19 & 6.84e-01\\
    & TIES ($\lambda=0.2$) & +2.10 & 1.01e-01 & +0.51 & 6.96e-02 & +2.02 & 1.17e-01 & +1.18 & 7.48e-02 & -0.45 & 1.09e-01\\
    & ISO ($\lambda=0.2$) & +0.60 & 6.75e-01 & +0.64 & \cellcolor{tabthird}1.54e-02 & +0.67 & 4.75e-01 & +2.01 & \cellcolor{tabfirst}9.14e-04 & -0.84 & \cellcolor{tabfirst}2.88e-03\\
    & TSV$^*$ ($\lambda=0.2$) & +1.70 & 2.21e-01 & -0.04 & 8.92e-01 & +1.94 & 2.17e-01 & +2.08 & \cellcolor{tabfirst}7.85e-03 & +0.19 & 6.84e-01\\
    & WUDI ($\lambda=0.2$) & +0.20 & 9.39e-01 & +0.02 & 9.57e-01 & +0.17 & 9.50e-01 & +2.01 & \cellcolor{tabthird}1.49e-02 & -1.11 & \cellcolor{tabthird}3.11e-02\\
    & OptMerge$^*$ ($\lambda=0.2$) & +1.70 & 2.21e-01 & -0.04 & 8.92e-01 & +1.94 & 2.17e-01 & +2.08 & \cellcolor{tabfirst}7.85e-03 & +0.19 & 6.84e-01\\
    & \textbf{DCAT (Ours)} & +3.66 & \cellcolor{tabfirst}5.20e-03 & +2.07 & \cellcolor{tabfirst}3.41e-14 & +5.73 & \cellcolor{tabfirst}1.73e-05 & +6.08 & \cellcolor{tabfirst}2.34e-13 & +0.92 & \cellcolor{tabfirst}5.33e-03\\
    \bottomrule
    \end{tabular}
    }
\end{table}

\begin{table}[]
    \caption{\textbf{$p$-value of DCAT on video+vision}}
    \label{tab:quant_video}
    \centering
    \resizebox{\textwidth}{!}{
    \begin{tabular}{cl|cc|cc|cc|cc}
    \toprule
         & & \multicolumn{2}{c|}{Video-MME} & \multicolumn{2}{c|}{TempCompass} & \multicolumn{2}{c|}{MVBench} & \multicolumn{2}{c}{MLVU} \\
         & & $\Delta$ (\%p) & $p$-value & $\Delta$ (\%p) & $p$-value & $\Delta$ (\%p) & $p$-value & $\Delta$ (\%p) & $p$-value \\
    \midrule
    \multirow{7}{*}{\rotatebox{90}{Vicuna-based}} & TA & +2.11 & 7.31e-02 & -0.07 & 9.47e-01 & +5.20 & \cellcolor{tabfirst}1.30e-07 & +12.37 & \cellcolor{tabfirst}3.29e-20 \\
    & TIES & -6.70 & \cellcolor{tabfirst}7.38e-16 & -0.83 & 1.81e-01 & -5.70 & \cellcolor{tabfirst}2.99e-17 & +0.41 & 6.36e-01 \\
    & ISO & +3.44 & \cellcolor{tabfirst}8.96e-04 & -0.07 & 9.43e-01 & +4.53 & \cellcolor{tabfirst}1.74e-06 & +13.06 & \cellcolor{tabfirst}2.53e-21 \\
    & TSV & +4.04 & \cellcolor{tabfirst}1.33e-04 & +1.79 & \cellcolor{tabthird}1.13e-02 & +6.10 & \cellcolor{tabfirst}2.25e-10 & +12.97 & \cellcolor{tabfirst}1.32e-28 \\
    & WUDI & +1.56 & 1.81e-01 & -0.52 & 5.39e-01 & +4.83 & \cellcolor{tabfirst}8.67e-07 & +12.05 & \cellcolor{tabfirst}1.21e-19 \\
    & OptMerge & +3.30 & \cellcolor{tabfirst}6.17e-05 & +1.01 & 6.10e-02 & +6.35 & \cellcolor{tabfirst}5.89e-14 & +10.35 & \cellcolor{tabfirst}6.89e-24 \\
    & \textbf{DCAT (Ours)} & +6.53 & \cellcolor{tabfirst}3.06e-08 & +4.25 & \cellcolor{tabfirst}1.11e-06 & +9.17 & \cellcolor{tabfirst}2.93e-20 & +13.75 & \cellcolor{tabfirst}6.69e-39 \\
    \midrule
\multirow{13}{*}{\rotatebox{90}{LLaMA 3.1-based}} & TA$^*$ ($\lambda=1.0$) & +0.70 & 3.94e-01 & +1.82 & \cellcolor{tabthird}1.18e-02 & +2.88 & \cellcolor{tabfirst}1.14e-04 & +11.55 & \cellcolor{tabfirst}5.43e-31 \\
    & TIES ($\lambda=1.0$) & +3.30 & \cellcolor{tabfirst}1.15e-04 & -4.06 & \cellcolor{tabfirst}4.81e-08 & +5.88 & \cellcolor{tabfirst}1.44e-13 & +13.75 & \cellcolor{tabfirst}1.74e-39 \\
    & ISO ($\lambda=1.0$) & +2.70 & \cellcolor{tabfirst}2.30e-03 & -22.44 & \cellcolor{tabfirst}1.88e-166 & +5.03 & \cellcolor{tabfirst}3.54e-09 & +11.96 & \cellcolor{tabfirst}1.18e-29 \\
    & TSV$^*$ ($\lambda=1.0$) & +3.67 & \cellcolor{tabfirst}8.75e-06 & +5.49 & \cellcolor{tabfirst}6.86e-15 & +6.00 & \cellcolor{tabfirst}5.26e-15 & +14.54 & \cellcolor{tabfirst}1.27e-45 \\
    & WUDI ($\lambda=1.0$) & +1.96 & \cellcolor{tabthird}1.49e-02 & -40.75 & \cellcolor{tabfirst}0.00e+00 & +4.35 & \cellcolor{tabfirst}3.66e-08 & +13.89 & \cellcolor{tabfirst}6.65e-40 \\
    & OptMerge$^*$ ($\lambda=1.0$) & +4.67 & \cellcolor{tabfirst}1.20e-07 & -47.69 & \cellcolor{tabfirst}0.00e+00 & +6.23 & \cellcolor{tabfirst}4.68e-14 & +15.59 & \cellcolor{tabfirst}6.01e-47 \\
    \cmidrule{2-10}
    & TA$^*$ ($\lambda=0.2$) & +4.74 & \cellcolor{tabfirst}1.16e-06 & -8.13 & \cellcolor{tabfirst}2.98e-24 & +2.18 & \cellcolor{tabthird}1.44e-02 & +9.43 & \cellcolor{tabfirst}3.83e-17 \\
    & TIES ($\lambda=0.2$) & +1.70 & 9.89e-02 & -6.36 & \cellcolor{tabfirst}4.20e-15 & +4.43 & \cellcolor{tabfirst}3.71e-07 & +5.80 & \cellcolor{tabfirst}6.91e-06 \\
    & ISO ($\lambda=0.2$) & +1.33 & 2.07e-01 & +5.91 & \cellcolor{tabfirst}5.03e-14 & +3.80 & \cellcolor{tabfirst}8.24e-06 & +11.64 & \cellcolor{tabfirst}2.07e-20 \\
    & TSV$^*$ ($\lambda=0.2$) & +3.26 & \cellcolor{tabfirst}1.17e-03 & +2.08 & \cellcolor{tabfirst}7.87e-03 & +4.65 & \cellcolor{tabfirst}1.02e-07 & +7.54 & \cellcolor{tabfirst}6.81e-10 \\
    & WUDI ($\lambda=0.2$) & +4.11 & \cellcolor{tabfirst}3.67e-05 & +7.15 & \cellcolor{tabfirst}6.29e-21 & +3.55 & \cellcolor{tabfirst}7.58e-05 & +11.45 & \cellcolor{tabfirst}9.42e-22 \\
    & OptMerge$^*$ ($\lambda=0.2$) & +2.41 & \cellcolor{tabthird}1.91e-02 & +6.02 & \cellcolor{tabfirst}8.59e-15 & +5.48 & \cellcolor{tabfirst}5.69e-10 & +9.25 & \cellcolor{tabfirst}8.34e-14 \\
    & \textbf{DCAT (Ours)} & +6.25 & \cellcolor{tabfirst}3.85e-09 & +8.60 & \cellcolor{tabfirst}7.24e-23 & +10.18 & \cellcolor{tabfirst}6.03e-30 & +18.36 & \cellcolor{tabfirst}5.51e-57 \\
    \bottomrule
    \end{tabular}
    }
\end{table}

\subsection{Additional quantitative results}\label{sec:append_subcategory}
For some of the results presented in Tab.\ref{tab:quant}, we reported only the overall or average scores rather than detailed subcategory results due to space constraints. In this section, we provide these detailed subcategory results, as well as additional benchmark results on EgoSchema\cite{mangalam2023egoschema} and LongVideoBench~\cite{wu2024longvideobench}. The results are reported in Tab.\ref{tab:append_MVBench_sub} and Tab.\ref{tab:append_mlvu_tempcompass_sub}. These results show that DCAT consistently improves performance across diverse subcategories and additional benchmarks, while maintaining a clear performance advantage over other merging methods.

\newcommand{\tabincell}[2]{\begin{tabular}{@{}#1@{}}#2\end{tabular}}
\begin{table}[]
    \caption{\textbf{Detailed results for MVBench benchmark.} We report the specific performance measured from subcategories on MVBench.}
    \label{tab:append_MVBench_sub}
    \centering
    \resizebox{\textwidth}{!}{
    \begin{tabular}{cl|cccccccccc}
    \toprule
         & & \multicolumn{10}{c}{MVBench} \\
         & & \tabincell{c}{Action \\ Antonym} & \tabincell{c}{Action \\ Count} & \tabincell{c}{Action \\ Localization} & \tabincell{c}{Action \\ Prediction} & \tabincell{c}{Action \\ Sequence} & \tabincell{c}{Character \\ Order} & \tabincell{c}{Counterfactual \\ Inference} & \tabincell{c}{Egocentric \\ Navigation} & \tabincell{c}{Episodic \\ Reasoning} & \tabincell{c}{Fine Grained \\ Action}\\
    \midrule
    \multirow{8}{*}{\rotatebox{90}{Vicuna-based}} & Original & 51.50 & \cellcolor{tabfirst}38.00 & 27.50 & 29.00 & 35.00 & \cellcolor{tabsecond}47.00 & 34.00 & 27.00 & 32.00 & 40.00 \\
    \cmidrule{2-12}
 & TA & \cellcolor{tabthird}62.00 & 32.50 & \cellcolor{tabsecond}36.00 & 46.00 & 48.50 & 45.00 & 34.00 & \cellcolor{tabfirst}34.00 & 40.50 & 38.00 \\
 & TIES & 55.50 & 16.50 & 26.00 & 24.50 & 32.00 & 28.50 & 31.50 & 25.00 & 22.50 & 30.50 \\
 & ISO & 54.50 & \cellcolor{tabfirst}38.00 & 28.00 & \cellcolor{tabfirst}49.00 & 46.00 & \cellcolor{tabthird}46.50 & \cellcolor{tabthird}41.00 & \cellcolor{tabfirst}34.00 & 37.50 & \cellcolor{tabsecond}41.00 \\
 & TSV & 60.50 & \cellcolor{tabthird}36.50 & 31.00 & \cellcolor{tabsecond}48.00 & \cellcolor{tabthird}50.00 & 46.00 & 38.50 & \cellcolor{tabfirst}34.00 & \cellcolor{tabthird}41.00 & \cellcolor{tabthird}40.50 \\
 & WUDI & \cellcolor{tabsecond}63.00 & 32.00 & \cellcolor{tabthird}34.00 & \cellcolor{tabthird}46.50 & \cellcolor{tabthird}50.00 & \cellcolor{tabthird}46.50 & 35.00 & \cellcolor{tabsecond}32.00 & 38.50 & 39.50 \\
 & OptMerge & 61.50 & \cellcolor{tabthird}36.50 & 28.50 & \cellcolor{tabfirst}49.00 & \cellcolor{tabsecond}52.50 & 44.00 & \cellcolor{tabsecond}43.00 & \cellcolor{tabthird}31.50 & \cellcolor{tabsecond}43.50 & 38.50 \\
 & \textbf{DCAT (Ours)} & \cellcolor{tabfirst}65.00 & \cellcolor{tabsecond}37.00 & \cellcolor{tabfirst}37.00 & \cellcolor{tabfirst}49.00 & \cellcolor{tabfirst}53.00 & \cellcolor{tabfirst}49.00 & \cellcolor{tabfirst}46.00 & \cellcolor{tabfirst}34.00 & \cellcolor{tabfirst}45.50 & \cellcolor{tabfirst}42.00 \\
    \midrule
    \multirow{14}{*}{\rotatebox{90}{LLaMA 3.1-based}} & Original & 61.50 & 18.00 & 30.00 & \cellcolor{tabsecond}47.00 & 45.00 & 50.00 & \cellcolor{tabsecond}39.50 & 21.00 & 33.50 & 36.50 \\
    \cmidrule{2-12}
 & TA$^*$ ($\lambda=1.0$) & 71.00 & \cellcolor{tabsecond}45.50 & 26.00 & 44.00 & 45.50 & 49.00 & 31.00 & 25.00 & 44.50 & 36.50 \\
 & TIES ($\lambda=1.0$) & \cellcolor{tabthird}75.00 & 39.50 & 27.50 & \cellcolor{tabthird}46.00 & 50.00 & \cellcolor{tabthird}55.50 & 32.50 & 33.50 & 47.00 & \cellcolor{tabthird}40.50 \\
 & ISO ($\lambda=1.0$) & 64.50 & \cellcolor{tabthird}44.50 & 27.50 & \cellcolor{tabsecond}47.00 & 50.00 & 52.00 & \cellcolor{tabthird}36.00 & 37.00 & 46.00 & 40.00 \\
 & TSV$^*$ ($\lambda=1.0$) & \cellcolor{tabfirst}76.50 & 41.00 & 28.50 & 45.50 & 50.50 & 55.00 & \cellcolor{tabthird}36.00 & 26.00 & 45.00 & 39.50 \\
 & WUDI ($\lambda=1.0$) & 69.00 & 42.00 & \cellcolor{tabthird}30.50 & 45.00 & 53.00 & \cellcolor{tabthird}55.50 & 30.00 & 29.50 & 46.00 & 38.00 \\
 & OptMerge$^*$ ($\lambda=1.0$) & 71.00 & 36.50 & \cellcolor{tabsecond}32.50 & 45.00 & 53.00 & \cellcolor{tabfirst}57.00 & 33.00 & 34.50 & 46.00 & \cellcolor{tabsecond}41.50 \\
    \cmidrule{2-12}
 & TA$^*$ ($\lambda=0.2$) & 63.50 & 28.00 & 24.00 & 45.00 & 54.00 & 44.50 & 35.50 & 1.50 & 45.00 & 38.50 \\
 & TIES ($\lambda=0.2$) & 64.50 & 27.50 & 21.00 & 44.00 & 57.00 & 49.50 & 32.50 & 36.50 & 49.50 & 36.50 \\
 & ISO ($\lambda=0.2$) & 58.00 & 27.50 & 25.50 & 40.50 & 56.50 & 48.00 & \cellcolor{tabfirst}40.00 & 36.50 & \cellcolor{tabsecond}51.00 & 35.50 \\
 & TSV$^*$ ($\lambda=0.2$) & 62.00 & 28.00 & 24.50 & 40.50 & \cellcolor{tabsecond}59.00 & 50.00 & \cellcolor{tabsecond}39.50 & \cellcolor{tabsecond}38.00 & \cellcolor{tabthird}50.50 & 36.50 \\
 & WUDI ($\lambda=0.2$) & 61.50 & 23.00 & 27.50 & 39.00 & 55.00 & 49.50 & 35.50 & \cellcolor{tabthird}37.50 & 49.50 & 38.50 \\
 & OptMerge$^*$ ($\lambda=0.2$) & 64.00 & 28.50 & 26.50 & 42.50 & \cellcolor{tabthird}58.50 & 49.50 & \cellcolor{tabfirst}40.00 & \cellcolor{tabthird}37.50 & \cellcolor{tabthird}50.50 & 34.50 \\
 & \textbf{DCAT (Ours)} & \cellcolor{tabsecond}75.50 & \cellcolor{tabfirst}48.50 & \cellcolor{tabfirst}33.50 & \cellcolor{tabfirst}48.50 & \cellcolor{tabfirst}60.00 & \cellcolor{tabsecond}56.00 & \cellcolor{tabsecond}39.50 & \cellcolor{tabfirst}39.50 & \cellcolor{tabfirst}51.50 & \cellcolor{tabfirst}42.00 \\
    \bottomrule
    \end{tabular}
    }
    \centering
    \resizebox{\textwidth}{!}{
    \begin{tabular}{cl|ccccccccccc}
    \toprule
         & & \multicolumn{11}{c}{MVBench} \\
         & & \tabincell{c}{Fine Grained \\ Pose} & \tabincell{c}{Moving \\ Attribute} & \tabincell{c}{Moving \\ Count} & \tabincell{c}{Moving \\ Direction} & \tabincell{c}{Object \\ Existence} & \tabincell{c}{Object \\ Interaction} & \tabincell{c}{Object \\ Shuffle} & \tabincell{c}{Scene \\ Transition} & \tabincell{c}{State \\ Change} & \tabincell{c}{Unexpected \\ Action} & Average\\
    \midrule
    \multirow{8}{*}{\rotatebox{90}{Vicuna-based}} & Original & 28.00 & 43.50 & 22.00 & 27.50 & 47.00 & 52.50 & \cellcolor{tabsecond}38.50 & 65.50 & 39.50 & 44.50 & 38.48 \\
    \cmidrule{2-13}
 & TA & \cellcolor{tabthird}33.00 & 47.00 & \cellcolor{tabthird}30.00 & \cellcolor{tabthird}30.00 & \cellcolor{tabsecond}54.50 & 54.00 & 34.50 & \cellcolor{tabthird}82.00 & \cellcolor{tabthird}45.00 & 47.00 & 43.68 \\
 & TIES & 31.00 & 42.00 & 20.50 & 19.50 & 39.00 & 47.00 & 30.00 & 54.00 & 39.00 & 41.00 & 32.78 \\
 & ISO & 29.50 & 47.50 & \cellcolor{tabsecond}34.50 & 27.00 & \cellcolor{tabthird}53.00 & 51.50 & 35.00 & 78.50 & 39.00 & 49.00 & 43.00 \\
 & TSV & 32.00 & \cellcolor{tabsecond}49.50 & 29.00 & \cellcolor{tabsecond}31.50 & 52.50 & \cellcolor{tabthird}55.00 & \cellcolor{tabthird}37.00 & \cellcolor{tabsecond}82.50 & \cellcolor{tabsecond}46.50 & \cellcolor{tabthird}50.00 & \cellcolor{tabthird}44.58 \\
 & WUDI & 30.50 & 47.00 & 27.50 & 28.50 & \cellcolor{tabfirst}55.00 & 52.00 & 34.00 & \cellcolor{tabthird}82.00 & \cellcolor{tabthird}45.00 & 47.50 & 43.30 \\
 & OptMerge & \cellcolor{tabsecond}33.50 & \cellcolor{tabthird}49.00 & 28.50 & \cellcolor{tabsecond}31.50 & 47.00 & \cellcolor{tabsecond}56.50 & \cellcolor{tabthird}37.00 & \cellcolor{tabfirst}83.00 & \cellcolor{tabsecond}46.50 & \cellcolor{tabfirst}55.50 & \cellcolor{tabsecond}44.83 \\
 & \textbf{DCAT (Ours)} & \cellcolor{tabfirst}35.00 & \cellcolor{tabfirst}51.00 & \cellcolor{tabfirst}36.50 & \cellcolor{tabfirst}33.00 & \cellcolor{tabsecond}54.50 & \cellcolor{tabfirst}58.50 & \cellcolor{tabfirst}41.00 & \cellcolor{tabfirst}83.00 & \cellcolor{tabfirst}48.50 & \cellcolor{tabsecond}54.50 & \cellcolor{tabfirst}47.65 \\
    \midrule
    \multirow{14}{*}{\rotatebox{90}{LLaMA 3.1-based}} & Original & 34.50 & 45.50 & 31.00 & 25.00 & \cellcolor{tabthird}51.00 & 58.50 & 31.50 & 79.00 & 38.50 & 43.50 & 41.00 \\
    \cmidrule{2-13}
 & TA$^*$ ($\lambda=1.0$) & 35.50 & \cellcolor{tabsecond}48.00 & 38.00 & 22.50 & \cellcolor{tabsecond}53.00 & 54.00 & 29.50 & 82.00 & 44.50 & 52.50 & 43.88 \\
 & TIES ($\lambda=1.0$) & 34.00 & \cellcolor{tabthird}47.00 & 32.00 & 30.50 & \cellcolor{tabfirst}53.50 & \cellcolor{tabsecond}67.00 & \cellcolor{tabsecond}35.50 & 87.00 & \cellcolor{tabthird}45.50 & 58.50 & 46.88 \\
 & ISO ($\lambda=1.0$) & 30.50 & 46.00 & 36.50 & 31.00 & 48.00 & 55.50 & 32.00 & \cellcolor{tabthird}88.00 & 43.50 & 65.00 & 46.03 \\
 & TSV$^*$ ($\lambda=1.0$) & \cellcolor{tabthird}36.50 & \cellcolor{tabsecond}48.00 & 34.50 & \cellcolor{tabthird}33.00 & \cellcolor{tabsecond}53.00 & 64.00 & \cellcolor{tabthird}34.00 & 87.00 & \cellcolor{tabsecond}46.00 & 60.50 & \cellcolor{tabthird}47.00 \\
 & WUDI ($\lambda=1.0$) & 35.00 & 44.50 & 36.00 & 27.00 & \cellcolor{tabfirst}53.50 & 59.50 & 30.00 & 83.50 & 41.50 & 58.00 & 45.35 \\
 & OptMerge$^*$ ($\lambda=1.0$) & 36.00 & 45.00 & 35.00 & 32.50 & \cellcolor{tabsecond}53.00 & 64.50 & \cellcolor{tabsecond}35.50 & 87.00 & 45.00 & 61.00 & \cellcolor{tabsecond}47.23 \\
    \cmidrule{2-13}
 & TA$^*$ ($\lambda=0.2$) & 36.00 & 46.00 & 38.00 & \cellcolor{tabsecond}34.00 & 48.50 & \cellcolor{tabthird}65.00 & 30.00 & 76.00 & \cellcolor{tabthird}45.50 & 65.00 & 43.18 \\
 & TIES ($\lambda=0.2$) & \cellcolor{tabsecond}37.00 & 46.50 & \cellcolor{tabsecond}40.00 & \cellcolor{tabfirst}35.00 & 48.00 & 63.50 & 31.50 & 78.00 & 43.50 & 67.00 & 45.43 \\
 & ISO ($\lambda=0.2$) & 35.00 & 45.50 & 35.50 & 22.50 & 49.00 & 59.50 & 33.50 & \cellcolor{tabsecond}89.00 & 40.00 & 67.50 & 44.80 \\
 & TSV$^*$ ($\lambda=0.2$) & \cellcolor{tabsecond}37.00 & 45.50 & 37.00 & 26.00 & 48.00 & 63.50 & 32.00 & 82.50 & 45.00 & \cellcolor{tabthird}68.00 & 45.65 \\
 & WUDI ($\lambda=0.2$) & 35.50 & 44.00 & 38.50 & 28.00 & 44.50 & 63.00 & 28.50 & 87.00 & 40.00 & 65.50 & 44.55 \\
 & OptMerge$^*$ ($\lambda=0.2$) & \cellcolor{tabthird}36.50 & 44.00 & \cellcolor{tabthird}39.50 & \cellcolor{tabsecond}34.00 & 45.50 & 64.50 & 33.00 & \cellcolor{tabthird}88.00 & 42.50 & \cellcolor{tabsecond}70.00 & 46.48 \\
 & \textbf{DCAT (Ours)} & \cellcolor{tabfirst}39.00 & \cellcolor{tabfirst}49.00 & \cellcolor{tabfirst}41.00 & \cellcolor{tabfirst}35.00 & \cellcolor{tabsecond}53.00 & \cellcolor{tabfirst}68.00 & \cellcolor{tabfirst}36.50 & \cellcolor{tabfirst}90.00 & \cellcolor{tabfirst}47.00 & \cellcolor{tabfirst}70.50 & \cellcolor{tabfirst}51.18 \\
    \bottomrule
    \end{tabular}
    }
\end{table}

\begin{table}[]
    \caption{\textbf{Detailed results for MLVU, TempCompass benchmark and other benchmarks.} We report the specific performance measured from subcategories on MLVU, TempCompass(Multi Choice, Caption Matching, Yes/No), and other two benchmarks(EgoSchema, LongVideoBench).}
    \label{tab:append_mlvu_tempcompass_sub}
    \centering
    \resizebox{\textwidth}{!}{
    \begin{tabular}{cl|cccccccc|cccccc}
    \toprule
         & & \multicolumn{8}{c|}{MLVU} & \multicolumn{6}{c}{TempCompass - Multi Choice} \\
         & & Count & \tabincell{c}{Topic \\ Reasoning} & Ego & Order & PlotQA & Needle & \tabincell{c}{Anomaly \\ Reco.} & Average & Action & Attr. & Dir. & Order & Speed & Overall \\
    \midrule
    \multirow{8}{*}{\rotatebox{90}{Vicuna-based}} & Original & 27.18 & 57.79 & 28.41 & 23.17 & 30.80 & 33.80 & 27.00 & 32.59 & 70.41 & 35.07 & 33.13 & 37.75 & 32.81 & 42.28 \\
    \cmidrule{2-16}
    & TA & \cellcolor{tabsecond}34.47 & 65.40 & 34.66 & 23.17 & \cellcolor{tabthird}50.09 & \cellcolor{tabsecond}45.07 & \cellcolor{tabthird}61.00 & 44.84 & \cellcolor{tabthird}82.84 & 45.14 & \cellcolor{tabsecond}37.01 & 38.74 & 35.33 & 48.29 \\
    & TIES & 25.73 & 58.56 & 27.27 & 23.17 & 32.47 & 31.83 & 33.00 & 33.15 & 63.61 & 33.33 & 32.54 & 35.76 & 30.60 & 39.56 \\
    & ISO & 30.10 & 67.30 & \cellcolor{tabthird}42.33 & 22.39 & 49.35 & \cellcolor{tabthird}44.23 & \cellcolor{tabsecond}61.50 & \cellcolor{tabsecond}45.31 & 81.66 & \cellcolor{tabthird}46.18 & 35.52 & 39.40 & 35.65 & 48.10 \\
    & TSV & \cellcolor{tabthird}33.98 & \cellcolor{tabthird}70.72 & \cellcolor{tabsecond}42.90 & \cellcolor{tabthird}23.94 & 48.98 & 42.25 & 53.50 & \cellcolor{tabthird}45.18 & \cellcolor{tabfirst}86.69 & \cellcolor{tabsecond}47.57 & \cellcolor{tabthird}36.42 & \cellcolor{tabsecond}42.38 & 35.02 & \cellcolor{tabsecond}50.06 \\
    & WUDI & 33.01 & 65.78 & 35.51 & 23.17 & \cellcolor{tabsecond}50.28 & 43.38 & 59.50 & 44.38 & 82.25 & 45.49 & 36.12 & \cellcolor{tabthird}40.07 & \cellcolor{tabthird}36.28 & \cellcolor{tabthird}48.48 \\
    & OptMerge & 30.58 & \cellcolor{tabsecond}71.10 & 36.08 & \cellcolor{tabsecond}25.10 & 43.60 & 40.56 & 56.00 & 43.29 & 79.29 & 44.10 & 35.22 & 37.75 & \cellcolor{tabsecond}38.17 & 47.34 \\
    & \textbf{DCAT (Ours)} & \cellcolor{tabfirst}35.74 & \cellcolor{tabfirst}72.37 & \cellcolor{tabfirst}44.24 & \cellcolor{tabfirst}25.73 & \cellcolor{tabfirst}51.73 & \cellcolor{tabfirst}46.04 & \cellcolor{tabfirst}62.33 & \cellcolor{tabfirst}46.34 & \cellcolor{tabsecond}86.27 & \cellcolor{tabfirst}48.16 & \cellcolor{tabfirst}38.06 & \cellcolor{tabfirst}43.54 & \cellcolor{tabfirst}38.45 & \cellcolor{tabfirst}51.62 \\
    \midrule
    \multirow{14}{*}{\rotatebox{90}{LLaMA 3.1-based}} & Original & 3.40 & 50.19 & 34.66 & 16.22 & 43.04 & 35.49 & 39.50 & 31.79 & 87.57 & 46.88 & \cellcolor{tabthird}36.12 & \cellcolor{tabthird}37.09 & 35.96 & 49.24 \\
    \cmidrule{2-16}
    & TA$^*$ ($\lambda=1.0$) & \cellcolor{tabthird}37.86 & 70.72 & 43.75 & 23.94 & 47.50 & 44.79 & 48.00 & 45.22 & 86.98 & 45.14 & 33.73 & \cellcolor{tabthird}37.09 & 36.91 & 48.48 \\
    & TIES ($\lambda=1.0$) & \cellcolor{tabsecond}38.35 & 73.00 & 46.02 & 23.17 & \cellcolor{tabthird}51.21 & 47.61 & 50.50 & 47.12 & 91.12 & \cellcolor{tabthird}49.31 & 34.03 & 35.10 & 42.27 & 50.89 \\
    & ISO ($\lambda=1.0$) & 24.27 & 67.30 & 42.05 & \cellcolor{tabthird}25.10 & 49.17 & \cellcolor{tabthird}49.58 & \cellcolor{tabfirst}59.50 & 45.28 & 67.46 & 38.19 & 26.27 & 18.21 & 17.03 & 33.86 \\
    & TSV$^*$ ($\lambda=1.0$) & \cellcolor{tabsecond}38.35 & \cellcolor{tabthird}74.52 & 47.73 & 23.55 & 51.02 & 47.89 & 53.50 & \cellcolor{tabthird}48.08 & 90.53 & 48.61 & 33.13 & 34.77 & \cellcolor{tabthird}43.85 & 50.70 \\
    & WUDI ($\lambda=1.0$) & \cellcolor{tabsecond}38.35 & 70.72 & 46.31 & 23.55 & 50.09 & 49.01 & \cellcolor{tabthird}54.50 & 47.51 & 17.16 & 8.68 & 6.57 & 4.64 & 6.62 & 8.86 \\
    & OptMerge$^*$ ($\lambda=1.0$) & \cellcolor{tabfirst}39.32 & 74.14 & 48.01 & \cellcolor{tabsecond}26.25 & \cellcolor{tabsecond}51.95 & 49.01 & \cellcolor{tabsecond}56.00 & \cellcolor{tabsecond}49.24 & 6.51 & 5.21 & 4.18 & 0.99 & 2.21 & 3.86 \\
    \cmidrule{2-16}
    & TA$^*$ ($\lambda=0.2$) & 13.11 & 70.72 & 47.16 & 14.67 & 50.46 & 48.45 & 42.00 & 40.94 & \cellcolor{tabsecond}93.20 & 48.96 & 34.93 & 36.76 & \cellcolor{tabsecond}47.00 & \cellcolor{tabsecond}52.72 \\
    & TIES ($\lambda=0.2$) & 21.36 & 54.37 & \cellcolor{tabsecond}51.14 & 11.97 & 41.19 & 42.54 & 47.50 & 38.58 & 89.65 & 43.75 & 28.96 & 31.79 & 37.85 & 46.96 \\
    & ISO ($\lambda=0.2$) & 21.84 & 73.38 & \cellcolor{tabthird}49.15 & 23.55 & 47.68 & 46.76 & 49.00 & 44.48 & 90.53 & 48.61 & 34.03 & \cellcolor{tabthird}37.09 & 42.27 & 51.01 \\
    & TSV$^*$ ($\lambda=0.2$) & 17.48 & 68.82 & \cellcolor{tabthird}49.15 & 13.51 & 43.04 & 41.97 & 49.00 & 40.42 & \cellcolor{tabthird}92.60 & 48.96 & 33.73 & 34.11 & 41.64 & 50.76 \\
    & WUDI ($\lambda=0.2$) & 13.59 & \cellcolor{tabsecond}75.29 & 48.30 & 21.62 & 48.24 & \cellcolor{tabsecond}50.14 & 49.50 & 43.81 & 92.31 & \cellcolor{tabsecond}49.65 & \cellcolor{tabsecond}36.42 & \cellcolor{tabsecond}38.08 & 43.53 & \cellcolor{tabthird}52.53 \\
    & OptMerge$^*$ ($\lambda=0.2$) & 17.96 & 69.20 & 48.58 & 18.92 & 45.83 & 45.35 & 47.00 & 41.83 & 91.72 & 47.57 & 31.94 & 35.43 & 42.27 & 50.32 \\
    & \textbf{DCAT (Ours)} & 37.38 & \cellcolor{tabfirst}76.81 & \cellcolor{tabfirst}51.70 & \cellcolor{tabfirst}27.03 & \cellcolor{tabfirst}52.69 & \cellcolor{tabfirst}50.99 & \cellcolor{tabthird}54.50 & \cellcolor{tabfirst}50.15 & \cellcolor{tabfirst}93.58 & \cellcolor{tabfirst}50.50 & \cellcolor{tabfirst}37.75 & \cellcolor{tabfirst}39.25 & \cellcolor{tabfirst}47.90 & \cellcolor{tabfirst}53.15 \\
    \bottomrule
    \end{tabular}
    }

    \centering
    \resizebox{\textwidth}{!}{
    \begin{tabular}{cl|cccccc|cccccc|cc}
    \toprule
         & & \multicolumn{6}{c|}{TempCompass - Caption Matching} & \multicolumn{6}{c|}{TempCompass - Yes/No} & \multicolumn{2}{c}{Others}\\
         & & Action & Attr. & Dir. & Order & Speed & Overall & Action & Attr. & Dir. & Order & Speed & Overall & EgoSchema & \tabincell{c}{LongVideo \\ Bench} \\
    \midrule
    \multirow{8}{*}{\rotatebox{90}{Vicuna-based}} & Original & \cellcolor{tabthird}85.86 & \cellcolor{tabthird}59.38 & \cellcolor{tabthird}55.96 & 49.67 & \cellcolor{tabsecond}59.11 & \cellcolor{tabsecond}61.88 & 79.39 & \cellcolor{tabsecond}62.72 & \cellcolor{tabsecond}54.05 & 53.66 & 46.37 & 60.01 & 48.60 & 39.04 \\
    \cmidrule{2-16}
    & TA & 67.00 & 54.17 & 54.43 & \cellcolor{tabfirst}51.67 & 47.08 & 54.89 & \cellcolor{tabsecond}81.08 & 57.59 & 48.58 & 55.24 & \cellcolor{tabthird}53.82 & 60.25 & 48.80 & 37.62 \\
    & TIES & \cellcolor{tabfirst}86.87 & 58.33 & 49.54 & 48.67 & 52.92 & 59.08 & 78.72 & 57.14 & \cellcolor{tabthird}53.85 & \cellcolor{tabfirst}56.28 & \cellcolor{tabsecond}57.36 & \cellcolor{tabsecond}61.60 & 38.20 & 35.90 \\
    & ISO & 80.47 & 57.29 & 52.60 & 47.00 & 55.33 & 58.42 & 77.53 & 56.47 & 49.80 & 51.57 & 50.84 & 58.21 & 47.20 & 38.07 \\
    & TSV & 82.49 & \cellcolor{tabfirst}62.15 & \cellcolor{tabsecond}56.57 & \cellcolor{tabthird}51.00 & 56.01 & \cellcolor{tabthird}61.54 & \cellcolor{tabthird}80.57 & 56.92 & 49.60 & 53.40 & 50.65 & 59.23 & \cellcolor{tabsecond}51.20 & \cellcolor{tabthird}40.02 \\
    & WUDI & 64.65 & 55.21 & 52.91 & 49.67 & 45.70 & 53.63 & 80.07 & 58.04 & 48.38 & 54.71 & 53.45 & 59.89 & \cellcolor{tabthird}49.60 & 38.30 \\
    & OptMerge & 81.48 & 57.99 & 53.21 & 46.33 & \cellcolor{tabthird}57.39 & 59.15 & \cellcolor{tabsecond}81.08 & \cellcolor{tabthird}61.38 & 51.82 & \cellcolor{tabthird}55.50 & 49.53 & \cellcolor{tabthird}60.70 & 47.60 & \cellcolor{tabsecond}40.46 \\
    & \textbf{DCAT (Ours)} & \cellcolor{tabsecond}86.08 & \cellcolor{tabsecond}60.67 & \cellcolor{tabfirst}57.23 & \cellcolor{tabsecond}51.28 & \cellcolor{tabfirst}59.77 & \cellcolor{tabfirst}62.67 & \cellcolor{tabfirst}81.24 & \cellcolor{tabfirst}63.56 & \cellcolor{tabfirst}55.61 & \cellcolor{tabsecond}56.23 & \cellcolor{tabfirst}59.02 & \cellcolor{tabfirst}63.11 & \cellcolor{tabfirst}52.05 & \cellcolor{tabfirst}40.54 \\
    \midrule
    \multirow{14}{*}{\rotatebox{90}{LLaMA 3.1-based}} & Original & 78.11 & 32.99 & 48.32 & 41.00 & 32.30 & 46.71 & 76.35 & 50.45 & 49.60 & 47.38 & 47.49 & 55.40 & 37.20 & 40.17 \\
    \cmidrule{2-16}
    & TA$^*$ ($\lambda=1.0$) & 67.00 & 44.79 & 42.51 & 47.67 & 42.27 & 48.77 & \cellcolor{tabthird}78.04 & \cellcolor{tabfirst}55.36 & \cellcolor{tabthird}50.00 & \cellcolor{tabfirst}50.79 & \cellcolor{tabsecond}54.00 & \cellcolor{tabfirst}58.74 & 41.80 & 40.31 \\
    & TIES ($\lambda=1.0$) & 58.59 & 41.67 & 41.28 & 37.67 & 38.83 & 43.58 & 73.48 & 43.53 & 47.17 & 39.01 & 26.63 & 47.09 & \cellcolor{tabthird}48.40 & 44.35 \\
    & ISO ($\lambda=1.0$) & 54.55 & 29.17 & 29.05 & 32.00 & 37.46 & 36.33 & 25.34 & 29.91 & 25.10 & 10.47 & 12.66 & 21.04 & 41.20 & 45.25 \\
    & TSV$^*$ ($\lambda=1.0$) & 82.49 & 54.86 & \cellcolor{tabsecond}57.19 & 56.33 & 54.30 & 61.01 & \cellcolor{tabfirst}79.56 & \cellcolor{tabsecond}54.24 & 48.79 & \cellcolor{tabsecond}50.00 & 51.96 & \cellcolor{tabthird}58.09 & 46.60 & 44.35 \\
    & WUDI ($\lambda=1.0$) & 2.36 & 1.74 & 0.61 & 0.33 & 0.69 & 1.13 & 29.39 & 20.54 & 20.24 & 6.02 & 6.89 & 17.37 & 34.80 & 42.78 \\
    & OptMerge$^*$ ($\lambda=1.0$) & 0.67 & 0.35 & 1.53 & 0.00 & 0.69 & 0.67 & 7.94 & 5.58 & 6.28 & 4.19 & 1.68 & 5.22 & 37.20 & 45.10 \\
    \cmidrule{2-16}
    & TA$^*$ ($\lambda=0.2$) & 49.49 & 37.85 & 36.09 & 45.00 & 46.05 & 42.78 & 53.55 & 31.25 & 37.45 & 28.01 & 30.54 & 37.22 & \cellcolor{tabsecond}49.20 & 48.09 \\
    & TIES ($\lambda=0.2$) & 53.20 & 31.25 & 32.72 & 32.33 & 42.27 & 38.26 & 57.94 & 38.17 & 45.75 & 42.67 & 49.72 & 47.70 & 44.60 & 46.67 \\
    & ISO ($\lambda=0.2$) & 87.21 & 56.94 & 52.60 & 56.00 & 62.54 & 62.87 & 77.03 & 51.79 & \cellcolor{tabsecond}52.63 & 47.38 & \cellcolor{tabthird}53.26 & 57.68 & 42.80 & 46.30 \\
    & TSV$^*$ ($\lambda=0.2$) & 75.76 & 50.00 & 44.34 & 55.00 & 59.11 & 56.62 & 72.97 & 45.09 & 49.39 & 42.93 & 48.23 & 53.04 & 46.80 & \cellcolor{tabthird}48.17 \\
    & WUDI ($\lambda=0.2$) & \cellcolor{tabsecond}93.94 & \cellcolor{tabsecond}64.58 & \cellcolor{tabthird}56.88 & \cellcolor{tabsecond}62.00 & \cellcolor{tabthird}63.23 & \cellcolor{tabthird}67.93 & 77.87 & 48.88 & 49.39 & 46.86 & 52.33 & 56.42 & 47.40 & \cellcolor{tabsecond}48.54 \\
    & OptMerge$^*$ ($\lambda=0.2$) & \cellcolor{tabthird}93.60 & \cellcolor{tabthird}63.54 & \cellcolor{tabthird}56.88 & \cellcolor{tabthird}61.00 & \cellcolor{tabsecond}67.35 & \cellcolor{tabsecond}68.26 & 74.66 & 47.10 & \cellcolor{tabthird}50.00 & 45.29 & 51.77 & 55.08 & 47.00 & 47.64 \\
    & \textbf{DCAT (Ours)} & \cellcolor{tabfirst}94.25 & \cellcolor{tabfirst}65.36 & \cellcolor{tabfirst}58.10 & \cellcolor{tabfirst}63.35 & \cellcolor{tabfirst}68.15 & \cellcolor{tabfirst}69.14 & \cellcolor{tabsecond}78.30 & \cellcolor{tabthird}52.70 & \cellcolor{tabfirst}54.05 & \cellcolor{tabthird}48.50 & \cellcolor{tabfirst}54.25 & \cellcolor{tabsecond}58.54 & \cellcolor{tabfirst}49.80 & \cellcolor{tabfirst}49.06 \\
    \bottomrule
    \end{tabular}
    }
\end{table}

\subsection{Results on vision benchmarks}\label{sec:append_vision_benchmark}
We use vision MLLMs as donor models to improve the performance of audio and video MLLMs. A natural question is how the vision performance of the resulting transferred models changes after DCAT. We investigate this question using Vicuna-based models, and report the results in Tab.\ref{tab:vision_benchmarks_vicuna_grouped}. We evaluate on a diverse set of vision benchmarks, including GQA\cite{hudson2019gqa}, MMBench~\cite{liu2024mmbench}, MMStar~\cite{chen2024we}, POPE~\cite{li2023evaluating}, and Text-VQA~\cite{singh2019towards}. When the audio MLLM is used as the recipient, the vision performance is largely preserved or slightly improved, indicating only minor perturbations. In contrast, when the video MLLM is used as the recipient, we observe some degree of degradation in vision performance. This behavior is expected because DCAT transfers alignment rather than modality-specific capabilities. We emphasize that our goal is not to obtain a single model that performs well on both modalities, but rather to transfer the capability of a strongly aligned MLLM from one modality to another in order to improve performance on the target modality.

\begin{table*}[t]
    \caption{\textbf{Results on vision benchmarks (Vicuna-based).} The performance of Vicuna based audio-vision(A+V) and video-vision(Vd+V) models are measured from five different benchmarks: GQA, MMBench-Dev(EN), MMStar, POPE, TextVQA-Val.}
    \label{tab:vision_benchmarks_vicuna_grouped}
    \centering
    \resizebox{\textwidth}{!}{
    \begin{tabular}{l|c|c|c|c|c|c|c|c|c|c|c|c}
    \toprule
    \multirow{2}{*}{} 
    & \multicolumn{2}{c|}{GQA}
    & \multicolumn{2}{c|}{MMBench}
    & \multicolumn{2}{c|}{MMStar}
    & \multicolumn{2}{c|}{POPE}
    & \multicolumn{2}{c|}{Text-VQA}
    & \multicolumn{2}{c}{Average} \\
    \midrule
    Original
    & \multicolumn{2}{c|}{61.91}
    & \multicolumn{2}{c|}{63.83}
    & \multicolumn{2}{c|}{35.44}
    & \multicolumn{2}{c|}{84.31}
    & \multicolumn{2}{c|}{47.76}
    & \multicolumn{2}{c}{58.65} \\
    \midrule
    Merging Method 
    & A+V & Vd+V 
    & A+V & Vd+V  
    & A+V & Vd+V  
    & A+V & Vd+V  
    & A+V & Vd+V  
    & A+V & Vd+V \\
    \midrule
    TA 
    & 61.76 & 58.08 & 62.97 & 60.40 & 35.62 & 34.69 & 85.34 & 84.83 & 49.10 & 45.39 & 58.96 & 56.68 \\
    TIES 
    & 43.89 & 53.02 & 57.99 & 59.36 & 33.98 & 33.61 & 79.56 & 85.06 & 42.99 & 46.26 & 51.68 & 55.46 \\
    ISO 
    & 58.75 & 57.62 & 61.43 & 60.74 & 34.07 & 32.48 & 85.31 & 82.51 & 48.15 & 46.89 & 57.54 & 56.05 \\
    TSV 
    & 60.92 & 60.75 & 63.83 & 63.66 & 34.26 & 33.58 & 86.02 & 85.58 & 49.86 & 47.87 & 58.98 & 58.29 \\
    WUDI 
    & 61.64 & 58.28 & 63.06 & 60.82 & 35.39 & 33.97 & 85.32 & 84.36 & 48.88 & 45.71 & 58.86 & 56.63 \\
    OptMerge 
    & 58.08 & 59.34 & 62.71 & 62.20 & 34.06 & 33.82 & 86.08 & 86.20 & 49.35 & 48.71 & 58.06 & 58.05 \\
    \midrule
    \textbf{DCAT (Ours)}
    & 61.43 & 40.62 & 63.75 & 51.46 & 33.53 & 35.06 & 85.13 & 61.08 & 49.30 & 24.34 & 58.63 & 42.51  \\
    \bottomrule
    \end{tabular}
    }
\end{table*}

\subsection{Results on audio-video DCAT}
Beyond using a vision MLLM as the donor in DCAT, we also conducted experiments with Vicuna-based models in two additional settings: Video-to-Audio, where a video MLLM is used as the donor and an audio MLLM as the recipient, and Audio-to-Video, where an audio MLLM is used as the donor and a video MLLM as the recipient. The results are reported in Tab.\ref{tab:audio_video_dcat}. In the Video-to-Audio setting, DCAT improves the performance of the audio MLLM, whereas the improvement in the Audio-to-Video setting is limited. We hypothesize that this discrepancy arises from differences in the intrinsic alignment quality of the donor model, which is consistent with the conjecture discussed in Sec.\ref{sec:discussion} of the main manuscript.

\begin{table*}[t]
    \centering
    \caption{\textbf{Audio-video cross-modal transfer results.}
    We report DCAT performance when transferring alignment between audio and video models.
    The left table evaluates video-to-audio transfer on MMAU~\cite{sakshi2024mmau}, while the right table evaluates audio-to-video transfer on Video-MME~\cite{fu2025video}.}
    \label{tab:audio_video_dcat}
    \setlength{\tabcolsep}{2pt}
    \renewcommand{\arraystretch}{1.08}

    \begin{minipage}[t]{0.49\textwidth}
        \centering
        \small
        \textbf{MMAU (Video-to-Audio)}
        \vspace{2pt}
        
        \begin{tabular}{lcccc}
            \toprule
            \textbf{Model} 
            & Music & Sound & Speech & Overall \\
            \midrule
            Original
            & 54.79 & 60.66 & 31.23 & 48.90 \\
            DCAT
            & 57.49 & 66.07 & 37.24 & 53.60 \\
            \bottomrule
        \end{tabular}
    \end{minipage}%
    \hfill
    \begin{minipage}[t]{0.49\textwidth}
        \centering
        \small
        \textbf{Video-MME (Audio-to-Video)}
        \vspace{2pt}
        
        \begin{tabular}{lcccc}
            \toprule
            \textbf{Model} 
            & Short & Medium & Long & Overall \\
            \midrule
            Original
            & 39.89 & 34.22 & 30.89 & 35.00 \\
            DCAT
            & 37.78 & 33.78 & 29.33 & 33.63 \\
            \bottomrule
        \end{tabular}
    \end{minipage}
\end{table*}

\subsection{Calibration-matched baselines}
We compare DCAT with LoRA fine-tuning (LoRA rank=256, batch size=8, 128 steps), RegMean~\cite{jin2022dataless}, and RegMean++~\cite{huu2025regmean++} under a matched calibration budget of 256 AVQA samples. Tab.~\ref{tab:calib_matched} report results on Vicuna-based models. DCAT outperforms the strongest calibration-aware baseline on both MMAU and Video-MME (by 1.9 and 1.3 points overall, respectively) across all categories in each benchmark, suggesting that DCAT's gains arise from its alignment-transfer objective, which leverages the donor model to jointly improve subspace overlap, spectral diversity, and text-compatible directions.

\begin{table}[h]
    \caption{\textbf{Calibration-matched comparison.} We compare DCAT with LoRA fine-tuning, RegMean~\cite{jin2022dataless}, and RegMean++~\cite{huu2025regmean++} under a matched calibration budget of 256 AVQA samples. MMAU is evaluated with Audio+Vision Vicuna-based models, and Video-MME with Video+Vision Vicuna-based models. LoRA is trained with rank 256, batch size 8, and 128 steps. The best results are in \textbf{bold}, and the second-best are \underline{underlined}.}
    \label{tab:calib_matched}
    \centering
    \renewcommand{\arraystretch}{1.0}
    \setlength{\tabcolsep}{9pt}
    \resizebox{\textwidth}{!}{%
    \begin{tabular}{l|cccc|cccc}
    \toprule
    \multirow{2}{*}{\textbf{Method}} 
    & \multicolumn{4}{c|}{MMAU} & \multicolumn{4}{c}{Video-MME} \\
    & Music & Sound & Speech & Overall & Short & Medium & Long & Overall \\
    \midrule
    Original & 54.8 & 60.7 & 31.2 & 48.9 & 39.9 & 34.2 & 30.9 & 35.0 \\
    LoRA fine-tuned & 53.6 & 62.2 & 33.0 & 49.6 & 44.2 & 35.7 & 31.3 & 37.1 \\
    RegMean~\cite{jin2022dataless} & 58.5 & \underline{66.5} & \underline{37.9} & \underline{54.3} & 45.8 & 38.4 & \underline{34.8} & 39.7 \\
    RegMean++~\cite{huu2025regmean++} & \underline{58.8} & \underline{66.5} & 35.5 & 53.6 & \underline{46.9} & \underline{39.3} & 34.6 & \underline{40.3} \\
    \midrule
    DCAT (ours) & \textbf{60.2} & \textbf{69.1} & \textbf{41.7} & \textbf{56.2} & \textbf{48.0} & \textbf{40.5} & \textbf{36.2} & \textbf{41.6} \\
    \bottomrule
    \end{tabular}
    }
\end{table}

\subsection{Calibration sample robustness}
To study the robustness of DCAT with respect to the calibration sample, we randomly sampled 256 AVQA training examples using 10 different random seeds and applied DCAT to both Audio+Vision and Video+Vision settings. The results in Tab.~\ref{tab:calib_seed} show low overall standard deviations of 0.33 on MMAU and 0.14 on VideoMME across calibration subsets, demonstrating that DCAT is robust to calibration sampling and random seed within AVQA.

\begin{table}[h]
    \caption{\textbf{Robustness to calibration sampling.} We randomly sample 256 AVQA training examples with 10 different random seeds and apply DCAT to the Audio+Vision (MMAU) and Video+Vision (Video-MME) settings. We report per-seed results, together with the mean and standard deviation over the 10 calibration subsets.}
    \label{tab:calib_seed}
    \centering
    \renewcommand{\arraystretch}{1.0}
    \setlength{\tabcolsep}{9pt}
    \resizebox{\textwidth}{!}{%
    \begin{tabular}{l|cccc|cccc}
    \toprule
    \multirow{2}{*}{\textbf{Seed}} 
    & \multicolumn{4}{c|}{MMAU} & \multicolumn{4}{c}{Video-MME} \\
    & Music & Sound & Speech & Overall & Short & Medium & Long & Overall \\
    \midrule
    0 & 59.38 & 68.57 & 40.04 & 56.00 & 48.64 & 40.53 & 35.53 & 41.57 \\
    1 & 59.08 & 68.27 & 39.74 & 55.70 & 48.76 & 39.98 & 35.53 & 41.42 \\
    2 & 60.88 & 67.67 & 39.74 & 56.10 & 48.09 & 40.31 & 35.76 & 41.39 \\
    3 & 59.68 & 66.47 & 40.34 & 55.50 & 47.64 & 41.09 & 35.31 & 41.35 \\
    4 & 57.89 & 67.67 & 40.34 & 55.30 & 48.76 & 40.98 & 35.20 & 41.64 \\
    5 & 57.89 & 68.87 & 41.54 & 56.10 & 47.87 & 41.09 & 35.09 & 41.35 \\
    6 & 59.68 & 68.27 & 40.34 & 56.10 & 47.53 & 41.09 & 35.09 & 41.24 \\
    7 & 59.68 & 67.67 & 40.94 & 56.10 & 47.98 & 40.64 & 35.64 & 41.42 \\
    8 & 58.19 & 67.97 & 40.34 & 55.50 & 48.20 & 40.20 & 35.20 & 41.20 \\
    9 & 60.28 & 67.67 & 40.64 & 56.20 & 48.31 & 40.31 & 35.09 & 41.24 \\
    \midrule
    \textbf{Mean} & \textbf{59.26} & \textbf{67.91} & \textbf{40.40} & \textbf{55.86} & \textbf{48.18} & \textbf{40.62} & \textbf{35.34} & \textbf{41.38} \\
    \textbf{Std.} & \textbf{1.01} & \textbf{0.66} & \textbf{0.54} & \textbf{0.33} & \textbf{0.44} & \textbf{0.42} & \textbf{0.25} & \textbf{0.14} \\
    \bottomrule
    \end{tabular}
    }
\end{table}

\subsection{Language ability on alignment-transferred models}
We evaluate general language ability after DCAT on 13 text benchmarks~\cite{hendrycks2020measuring, clark2018think, zellers2019hellaswag, sakaguchi2021winogrande, bisk2020piqa, clark2019boolq, welbl2017crowdsourcing, mihaylov2018can, talmor2019commonsenseqa, sap2019social, lai2017race, roemmele2011choice, zhou2023instruction} for Vicuna-based Audio+Vision and Video+Vision models in Tab.~\ref{tab:lang_ability}. McNemar’s test shows no significant change for Audio+Vision ($p=0.506$) and a significant improvement for Video+Vision ($p\ll0.01$), indicating no degradation.

\begin{table}[h]
    \caption{\textbf{Language ability before and after DCAT.} We evaluate general language ability on 13 text benchmarks for the Vicuna-based Audio+Vision and Video+Vision models before and after applying DCAT. $\Delta$ denotes After $-$ Before. McNemar's test shows no significant change for Audio+Vision ($p=0.506$) and a significant improvement for Video+Vision ($p\ll0.01$).}
    \label{tab:lang_ability}
    \centering
    \renewcommand{\arraystretch}{1.0}
    \setlength{\tabcolsep}{9pt}
    \resizebox{0.8\textwidth}{!}{%
    \begin{tabular}{l|ccc|ccc}
    \toprule
    \multirow{2}{*}{\textbf{Task}} 
    & \multicolumn{3}{c|}{Audio+Vision} & \multicolumn{3}{c}{Video+Vision} \\
    & Before & After & $\Delta$ & Before & After & $\Delta$ \\
    \midrule
    MMLU & 0.492 & 0.491 & $-$0.001 & 0.501 & 0.510 & $+$0.009 \\
    ARC & 0.550 & 0.514 & $-$0.036 & 0.506 & 0.521 & $+$0.015 \\
    HellaSwag & 0.763 & 0.745 & $-$0.018 & 0.777 & 0.775 & $-$0.002 \\
    WinoGrande & 0.731 & 0.710 & $-$0.021 & 0.722 & 0.722 & 0.000 \\
    PIQA & 0.787 & 0.784 & $-$0.003 & 0.778 & 0.786 & $+$0.008 \\
    BoolQ & 0.821 & 0.841 & $+$0.020 & 0.837 & 0.844 & $+$0.007 \\
    SciQ & 0.890 & 0.900 & $+$0.010 & 0.886 & 0.902 & $+$0.016 \\
    OpenBookQA & 0.446 & 0.454 & $+$0.008 & 0.450 & 0.462 & $+$0.012 \\
    CommonsenseQA & 0.567 & 0.699 & $+$0.132 & 0.651 & 0.708 & $+$0.057 \\
    Social IQa & 0.506 & 0.500 & $-$0.006 & 0.479 & 0.491 & $+$0.012 \\
    RACE & 0.426 & 0.435 & $+$0.009 & 0.426 & 0.434 & $+$0.008 \\
    CoPA & 0.870 & 0.860 & $-$0.010 & 0.870 & 0.860 & $-$0.010 \\
    IFEval & 0.314 & 0.279 & $-$0.035 & 0.307 & 0.311 & $+$0.004 \\
    \midrule
    \textbf{Mean} & \textbf{0.628} & \textbf{0.632} & \textbf{+0.004} & \textbf{0.630} & \textbf{0.641} & \textbf{+0.011} \\
    \bottomrule
    \end{tabular}
    }
\end{table}

\subsection{Additional ablation study}\label{sec:append_ablation}
In this section, in addition to the component-level ablation studies conducted in Tab.~\ref{tab:ablation}, we present ablation studies on key hyperparameters. Specifically, we study the balancing coefficient $\beta$ in $\mathcal{L}_\mathrm{cross}(\tau)$, the balancing coefficient $\eta$ in Eq.~\ref{eq:DCAT_objective}, the subspace dimension $k$ used to construct $P_k$, and the size of the calibration dataset. All experiments are conducted with Vicuna-based models. For each hyperparameter, we select the sweep range of the parameters to the values where marked performance degradation is observed on target-modality benchmarks such as MMAU and Video-MME.
The results of these additional ablation studies are reported in Tab.~\ref{tab:append_ablation}. We observe that appropriate choices of $\beta$ and $\eta$ are important for achieving strong performance. In contrast, DCAT maintains robust performance with respect to the subspace dimension $k$ and the number of calibration samples. These experiments provide empirical rationale for our hyperparameter choices.

\begin{table}[h]
    \caption{\textbf{Hyperparameter study on DCAT.} We conduct a hyperparameter study for DCAT. We evaluate a range of options for $\beta$, $\eta$, $k$, and the calibration sample size, and provide the rationale behind our hyperparameter choices.}
    \label{tab:append_ablation}
    \centering
    \renewcommand{\arraystretch}{1.0}
    \setlength{\tabcolsep}{9pt}
    \resizebox{\textwidth}{!}{%
    \begin{tabular}{l|cccc|cccc}
    \toprule
    \multirow{2}{*}{\textbf{Study subject}} 
    & \multicolumn{4}{c|}{MMAU} & \multicolumn{4}{c}{Video-MME} \\
    & Music & Sound & Speech & Overall & Short & Medium & Long & Overall \\
    \midrule
    Original & 54.79 & 60.66 & 31.23 & 48.90 & 38.33 & 33.67 & 30.67 & 34.22 \\
    Base (DCAT) & 60.22 & 69.06 & 41.74 & 56.20 & 46.33 & 39.67 & 34.33 & 40.11 \\
    \midrule
    $\beta=0.1$ in Eq.~\ref{eq:text_regularization}
    & 50.90 & 60.66 & 34.53 & 48.70 & 37.33 & 29.00 & 26.33 & 30.89 \\
    $\beta=0.2$ in Eq.~\ref{eq:text_regularization}
    & 61.08 & 66.97 & 39.64 & 55.10 & 49.33 & 36.67 & 34.67 & 40.22 \\
    $\beta=0.4$ in Eq.~\ref{eq:text_regularization}
    & 60.18 & 67.27 & 39.34 & 54.80 & 49.33 & 37.00 & 37.67 & 41.33 \\
    $\beta=0.5$ in Eq.~\ref{eq:text_regularization}
    & 58.98 & 66.67 & 39.64 & 54.30 & 49.67 & 37.00 & 38.33 & 41.67 \\
    $\beta=0.6$ in Eq.~\ref{eq:text_regularization}
    & 58.68 & 65.47 & 35.44 & 52.40 & 48.33 & 37.67 & 35.00 & 40.33 \\
    $\beta=1.0$ in Eq.~\ref{eq:text_regularization}
    & 51.50 & 58.56 & 33.63 & 47.90 & 33.00 & 26.33 & 27.33 & 28.89 \\
    \midrule
    $\eta=10^3$ in Eq.~\ref{eq:DCAT_objective}
    & 37.72 & 44.44 & 25.23 & 35.80 & 1.33 & 1.67 & 1.67 & 1.56 \\
    $\eta=10^4$ in Eq.~\ref{eq:DCAT_objective}
    & 56.29 & 66.67 & 37.54 & 53.50 & 46.67 & 34.33 & 34.67 & 38.56 \\
    $\eta=5\times10^4$ in Eq.~\ref{eq:DCAT_objective}
    & 58.38 & 66.37 & 39.04 & 54.60 & 49.67 & 38.00 & 36.00 & 41.22 \\
    $\eta=5\times10^5$ in Eq.~\ref{eq:DCAT_objective}
    & 57.49 & 66.67 & 38.74 & 54.30 & 50.67 & 37.33 & 34.33 & 40.78 \\
    $\eta=10^6$ in Eq.~\ref{eq:DCAT_objective}
    & 58.89 & 66.67 & 38.14 & 53.90 & 50.00 & 37.00 & 35.33 & 40.78 \\
    $\eta=10^7$ in Eq.~\ref{eq:DCAT_objective}
    & 56.89 & 66.67 & 33.33 & 52.30 & 45.67 & 37.33 & 34.67 & 39.22 \\
    \midrule
    $k=32$ 
    & 56.89 & 66.37 & 33.33 & 52.20 & 41.00 & 34.67 & 30.00 & 35.22 \\
    $k=64$
    & 58.68 & 66.67 & 39.34 & 54.90 & 49.33 & 37.67 & 35.67 & 40.89 \\
    $k=256$
    & 58.08 & 66.37 & 37.24 & 53.90 & 49.67 & 37.00 & 37.33 & 41.33 \\
    $k=512$
    & 58.68 & 66.97 & 33.63 & 53.10 & 44.33 & 37.67 & 34.67 & 38.89 \\
    \midrule
    $\#\text{calib. samples}=64$
    & 55.99 & 65.17 & 35.44 & 52.20 & 40.33 & 38.00 & 32.00 & 36.78 \\
    $\#\text{calib. samples}=1024$
    & 56.89 & 67.57 & 34.83 & 53.10 & 44.33 & 40.00 & 32.67 & 39.00 \\
    \bottomrule
    \end{tabular}
    }
\end{table}

\begingroup
\footnotesize
\setlength{\tabcolsep}{4pt}
\renewcommand{\arraystretch}{1.0}
\begin{xltabular}{\textwidth}{@{}
>{\centering\arraybackslash}p{1.0cm}
p{1.6cm}
p{1.0cm}
Y
r
>{\raggedright\arraybackslash}p{1.6cm}
@{}}
\caption{\textbf{Summary of benchmarks used in our evaluation.} Each benchmark's question type, example prompt, number of questions, and evaluation method are listed. For benchmarks's prompt where we modified the default lmms-eval prompt, both the default and custom prompts are provided.}\label{tab:benchmark-summary}\\
\toprule
\centering\textbf{Modality}
& \centering\textbf{Benchmark}
& \centering\textbf{Question Type}
& \centering\textbf{Example Prompt}
& \centering\textbf{\#\,Q}
& \centering\arraybackslash\textbf{Evaluation Method} \\
\midrule
\endfirsthead
\multicolumn{6}{c}{\tablename\ \thetable\ -- continued from previous page}\\
\toprule
\multicolumn{1}{c}{\textbf{Modality}} &
\multicolumn{1}{c}{\textbf{Benchmark}} &
\multicolumn{1}{c}{\makecell{\textbf{Question}\\\textbf{Type}}} &
\multicolumn{1}{c}{\textbf{Example Prompt}} &
\multicolumn{1}{c}{\textbf{\# Q}} &
\multicolumn{1}{c}{\makecell{\textbf{Evaluation}\\\textbf{Method}}} \\
\midrule
\endhead
\midrule
\multicolumn{6}{r}{\textit{continued on next page}} \\
\endfoot
\bottomrule
\endlastfoot
\multirow{5}{*}{\textbf{Audio}} & MMAU & MCQ & {"Based on the given audio, identify the source of the speaking voice. \newline A. Man \newline B. Woman \newline C. Child \newline D. Robot \newline Answer with the option's letter from the given choices directly."} & 1000 & rule and based parsing and matching \\
\cmidrule{2-6}
 & AIR-Bench-Foundation & MCQ & {"In terms of mood, what does this music suggest? \newline A. party \newline B. emotional \newline C. soundscape \newline D. trailer \newline  \newline Answer with the option's letter from the given choices directly."} & 24683 & rule and based parsing and matching \\
 \cmidrule{2-6}
 & MuChoMusic & MCQ & {"Which instrument initiates the piece? \newline (A) Synthesizer (B) Rueful tune (C) Acoustic guitar (D) Vocals  \newline Answer with the option's letter from the given choices directly: "} & 1187 & rule-based parsing and matching \\
 \cmidrule{2-6}
 & Clotho-AQA (filtered binary subset) & binary QA & {"Are people talking? \newline Answer the question using a single word only. "} & 1442 & exact match \\
 \cmidrule{2-6}
 & VocalSound & 6-class classification & {*We use the custom prompt for both Vicuna and LLaMA evaluation. \newline  \newline Default: \newline "Classify the human vocal sound to VocalSound in English: " \newline  \newline Custom: \newline "What type of vocal sound is in the audio? Answer with exactly one of: Laughter, Sniff, Throat clearing, Cough, Sigh, Sneeze."} & 3591 & rule-based parsing and matching \\
\midrule
\multirow{8}{*}{\textbf{Video}} & VideoMME & MCQ & {*We use the custom prompt for Vicuna evaluation. \newline  \newline Default: \newline "Select the best answer to the following multiple-choice question based on the video and the subtitles. Respond with only the letter (A, B, C, or D) of the correct option. \newline How many national flags appear in the video? \newline A. 3. \newline B. 4. \newline C. 2. \newline D. 5. \newline  \newline Answer with the option's letter from the given choices directly." \newline  \newline Custom: \newline "Select the best answer to the following multiple-choice question based on the video. Respond with only the letter (A, B, C, or D) of the correct option. How many national flags appear in the video? A. 3. B. 4. C. 2. D. 5. The best answer is:"} & 2700 & rule-based parsing and matching \\
\cmidrule{2-6}
 & TempCompass (multi-choice) & MCQ & {*We use the custom prompt for Vicuna evaluation. \newline  \newline Default: \newline "What is the man doing in the video? \newline A. dunking a basketball \newline B. dribbling a basketball \newline C. passing a basketball \newline Please directly give the best option:" \newline  \newline Custom: \newline "What is the man doing in the video? \newline A. dunking a basketball \newline B. dribbling a basketball \newline C. passing a basketball The best answer is:"} & 1580 & rule-based parsing and matching / calls gpt-3.5-turbo-1106~\cite{DBLP:journals/corr/abs-2005-14165} for judging on rule failure \\
 \cmidrule{2-6}
 & TempCompass (caption-matching) & MCQ & {*We use the custom prompt for Vicuna evaluation. \newline  \newline Default: \newline "Which description is a more suitable match for the video? \newline Option 1: The man is dribbling a basketball. \newline Option 2: A man is dunking a basketball. \newline Please directly give the best option:" \newline  \newline Custom: \newline "Which description is a more suitable match for the video? \newline Option 1: The man is dribbling a basketball. \newline Option 2: A man is dunking a basketball. The best answer is:"} & 1503 & rule-based parsing and matching / calls gpt-3.5-turbo-1106 for judging on rule failure \\
 \cmidrule{2-6}
 & TempCompass (yes/no) & binary QA & {*We use the custom prompt for Vicuna evaluation. \newline  \newline Default: \newline "Is the man dunking? \newline Please answer yes or no:" \newline  \newline Custom: \newline "Is the man dunking? Please answer yes or no:"} & 2453 & rule-based parsing and matching / calls gpt-3.5-turbo-1106 for judging on rule failure \\
 \cmidrule{2-6}
 & MVBench & MCQ & {*We use the custom prompt for Vicuna evaluation. \newline  \newline Default: \newline "Question:What is the action performed by the person in the video? \newline Option: \newline (A) Not sure \newline (B) Scattering something down \newline (C) Piling something up \newline Only give the best option. \newline " \newline  \newline Custom: \newline "Question:What is the action performed by the person in the video? \newline Option: \newline (A) Not sure \newline (B) Scattering something down \newline (C) Piling something up \newline The best answer is:"} & 4000 & rule and based parsing and matching \\
 \cmidrule{2-6}
 & MLVU & MCQ & {"What color is the water on the tropical beach in the video? \newline (A) Dark blue \newline (B) Clear \newline (C) Green \newline (D) Azure blue \newline  \newline Only give the best option. \newline Best option: ("} & 2174 & rule-based parsing and matching \newline\newline * uses per- task-type accuracy, not overall accuracy \\
 \cmidrule{2-6}
 & EgoSchema (subset) & MCQ & {"Taking into account all the actions performed by c, what can you deduce about the primary objective and focus within the video content? \newline A. C is cooking. \newline B. C is doing laundry. \newline C. C is cleaning the kitchen. \newline D. C is cleaning dishes. \newline E. C is cleaning the bathroom. \newline Answer with the option's letter from the given choices directly."} & 500 & rule-based parsing and matching \\
 \cmidrule{2-6}
 & LongVideo- Bench (Val) & MCQ & {"In front of an ancient building, a woman walks by holding something. When the subtitle says 'temples scattered throughout the', what color is the woman's clothing at that moment? \newline A. yellow \newline B. green \newline C. blue \newline D. red \newline Answer with the option's letter from the given choices directly. \newline "} & 1337 & rule-based parsing and matching \\
\midrule
\multirow{5}{*}{\textbf{Vision}} & GQA & Free-form short-answer & {"Who is wearing the dress? \newline Answer the question using a single word or phrase."} & 12578 & exact match \\
\cmidrule{2-6}
 & MMBench (Dev\_EN) & MCQ & {"Identify the emotion expressed in this image. There are several options: \newline A. happiness \newline B. sadness \newline C. anger \newline D. loneliness \newline  \newline Answer with the option's letter from the given choices directly."} & 1164 & rule-based parsing and matching / calls gpt-4o-2024-11-20~\cite{hurst2024gpt} for parsing if rule-based parsing fails \newline\newline *adapts circular evaluation \\
 \cmidrule{2-6}
 & MMStar & MCQ & {"What seems to be the theme of the image? \newline Options: A: Hanging Posters, B: Music, C: Home decoration, D: Playing guitar \newline Answer with the option's letter from the given choices directly"} & 1500 & rule-based parsing and matching \newline\newline * uses per- task-type accuracy, not overall accuracy \\
 \cmidrule{2-6}
 & POPE & binary QA & {"Is there a snowboard in the image? \newline Answer the question using a single word or phrase."} & 9000 & exact match \\
 \cmidrule{2-6}
 & TextVQA (Val) & Free-form short-answer & {"What is the brand of this camera? \newline Answer the question using a single word or phrase."} & 5000 & VQA soft accuracy \\
\end{xltabular}
\endgroup

\subsection{Efficiency study}\label{sec:append_efficiency}
All efficiency numbers are measured on a single NVIDIA L40S 48GB GPU with the same numerical precision as the original checkpoints.
For each method and recipient setting, we time only the core merging algorithm, excluding per-checkpoint adapter loading and saving so that the reported cost reflects merging rather than I/O.
All methods use the merging hyperparameters from their respective references ($300$ Adam iterations for WUDI and $300$ SGD iterations for OptMerge) from the source of OptMerge~\cite{wei2025optmerge}.

Tab.~\ref{tab:efficiency} shows that DCAT remains practical despite using calibration activations.
For the LLaMA 3.1-based settings, DCAT has a merging time comparable to stronger baselines such as TSV and OptMerge: $28.45$/$34.57$ minutes for LLaMA-Audio/Video, compared with $25.61$/$33.02$ for TSV and $29.96$/$35.55$ for OptMerge.
Its peak memory is also comparable, requiring $16.72$GB mainly to store layer-wise outputs for a small number of calibration samples, while TSV and OptMerge use $12.12$GB and $18.66$GB.
For Vicuna, DCAT is more expensive than lightweight baselines because it merges the full weights after merging LoRA adapters, making its cost closer to the full-weight LLaMA setting.
Even so, DCAT finishes within about $24$--$25$ minutes and uses $14.56$GB on a single L40S GPU, indicating that its one-time merging cost remains highly feasible.

\begin{table*}[t]
    \caption{\textbf{Efficiency study of model merging methods.}
    We compare the computational cost of DCAT with existing merging baselines across four recipient settings: Vicuna-Audio, Vicuna-Video, LLaMA-Audio, and LLaMA-Video.
    We report the wall-clock merging time and peak GPU memory measured during the merging procedure.
    All results are measured on the same hardware and with the same precision.}
    \label{tab:efficiency}
    \centering
    \small
    \renewcommand{\arraystretch}{1.18}
    \setlength{\tabcolsep}{4.5pt}
    \resizebox{\textwidth}{!}{%
    \begin{tabular}{l|cccc|cccc}
    \toprule
    \multirow{2}{*}{\textbf{Method}}
    & \multicolumn{4}{c|}{\textbf{Merging Time (min.)} $\downarrow$}
    & \multicolumn{4}{c}{\textbf{Peak GPU Memory (GB)} $\downarrow$} \\
    & Vicuna-Audio & Vicuna-Video & LLaMA-Audio & LLaMA-Video
    & Vicuna-Audio & Vicuna-Video & LLaMA-Audio & LLaMA-Video \\
\midrule                                                                                                                                
      TA       
      & 0.14 & 0.45 & 0.74  & 0.93  & 0.00 & 0.00 & 2.94  & 2.94  \\
      TIES     
      & 1.18 & 0.51 & 1.81  & 1.64  & 0.00 & 0.00 & 22.51 & 22.51 \\
      ISO      
      & 0.58 & 0.53 & 8.96  & 12.03 & 0.04 & 0.04 & 8.09  & 8.09  \\
      TSV      
      & 0.68 & 0.69 & 25.61 & 33.02 & 0.05 & 0.05 & 12.12 & 12.12 \\
      WUDI     
      & 5.76 & 5.75 & 23.20 & 23.20 & 2.29 & 2.29 & 4.89  & 4.89  \\
      OptMerge 
      & 5.86 & 5.79 & 29.96 & 35.55 & 2.33 & 2.33 & 18.66 & 18.66 \\
      \midrule                                                  
      \textbf{DCAT}                                      
      & 23.64 & 25.39 & 28.45 & 34.57 & 14.56 & 14.56 & 16.72 & 16.72 \\      
    \bottomrule
    \end{tabular}
    }
\end{table*}

\begin{figure}[!h]
  \centering
  \vspace{-1.0em}
  \includegraphics[width=0.5\linewidth]{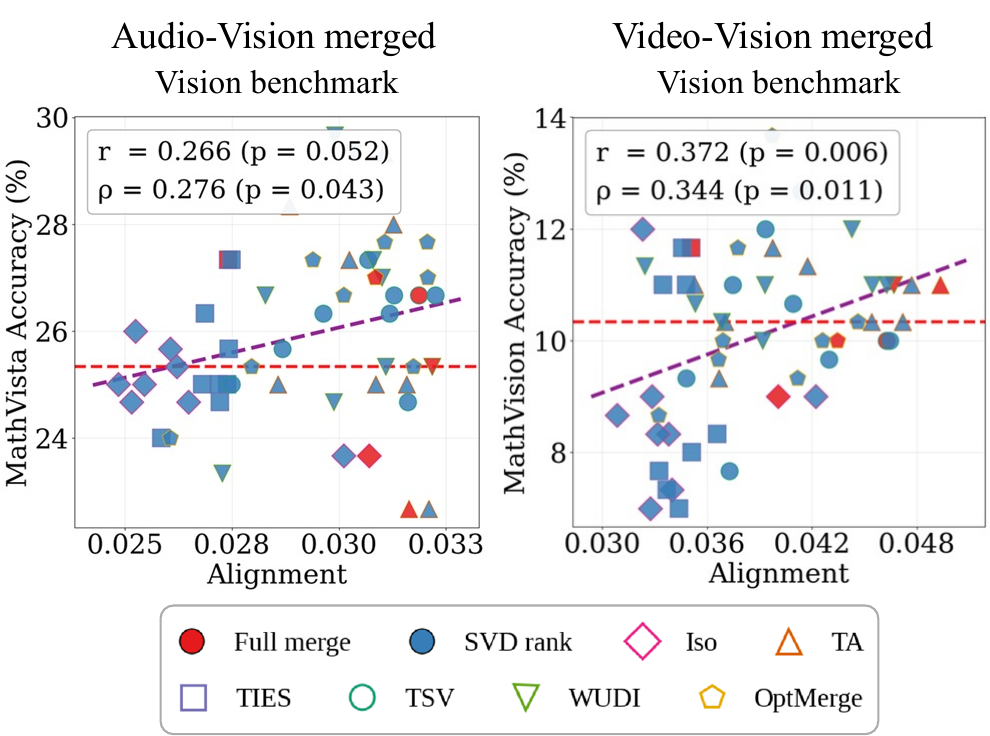}
  \caption{\textbf{Failure case on alignment--performance correlation.} Although our experiments across hundreds of model weights and diverse benchmarks demonstrate a strong correlation between the alignment metric and performance, we observe a small number of benchmarks for which this correlation is not pronounced. These cases suggest that further study may be needed to explain the underlying factors driving performance improvements.}
  \label{fig:append_limitation}
  \vspace{-1.5em}
\end{figure}
\section{Limitation}\label{sec:limitation}
First, although our analyses across hundreds of model checkpoints and diverse benchmarks reveal a strong and statistically meaningful correlation between the proposed Alignment score and downstream performance, this relationship is not universal across all evaluation settings. 
As shown in Fig.~\ref{fig:append_limitation}, a small number of benchmarks exhibit weak or negligible correlation, suggesting that factors beyond modality--text alignment may dominate performance in certain tasks. 
For example, some benchmarks may depend more heavily on modality-specific perception, reasoning format, dataset bias, or evaluation protocol than on the alignment properties captured by our score. 
Therefore, our Alignment score should be interpreted as a useful and broadly predictive indicator for cross-modal transfer, rather than as a complete explanation of all performance variations. 
Understanding when alignment is insufficient, and identifying complementary factors that better explain such exceptional cases, is an important direction for future work. Second, DCAT assumes that the models to be merged have compatible parameter shapes, as it operates directly in weight space. 
This assumption is shared by many training-free merging methods, but it limits applicability to in-the-wild MLLMs whose language backbones differ in tokenizer size, embedding matrix size, LM head size, or other architectural components. 
Such layers are particularly important in LLMs, since they directly determine token representation and generation behavior; na\"ive resizing, padding, or partial copying can severely degrade model performance and would also introduce confounding factors beyond the merging method itself. 
For this reason, we restrict our experiments to models with shape-compatible language-model parameters and do not apply additional retraining or ad-hoc architectural surgery to force compatibility. 
Extending directional cross-modal alignment transfer to models with different vocabularies, embedding spaces, or heterogeneous backbone shapes would substantially broaden the coverage of DCAT and represents a promising future research direction.

\section{Broader Impacts}\label{sec:append_broader_impacts}
This work proposes an efficient model-merging framework that transfers alignment across MLLMs trained for different modalities, improving target-modality performance without additional large-scale training. This may have positive societal impact by reducing the data and compute costs required to develop capable multimodal systems, especially for data-scarce, high-volume, or domain-specific modalities such as audio, video, medical, or scientific data. As a result, the approach could improve accessibility for researchers and organizations with limited computational resources. However, stronger audio and video understanding capabilities may also be misused for surveillance, sensitive attribute inference, large-scale profiling, or downstream systems that make unfair or unsafe decisions. In addition, cross-modality alignment transfer may propagate or amplify biases, privacy risks, or failure modes inherited from the donor model, and incorrect multimodal understanding could cause harm when used in high-stakes applications. To mitigate these risks, practical deployments should include careful evaluation across demographic groups and domains, privacy-preserving calibration data, safety testing before and after merging, restrictions on sensitive use cases, and human oversight in high-risk settings. Although this work is primarily foundational, its direct path to improving multimodal capability warrants consideration of these potential impacts.

\clearpage
\newpage
\input{checklist.tex}

\end{document}

%% file: checklist.tex
\section*{NeurIPS Paper Checklist}
\begin{enumerate}
\item {\bf Claims}
    \item[] Question: Do the main claims made in the abstract and introduction accurately reflect the paper's contributions and scope?
    \item[] Answer: \answerYes{} 
    \item[] Justification: We identify a phenomenon in which merging MLLMs trained on different modalities using existing model-merging methods can substantially improve performance on specific target modalities. We theoretically characterize this phenomenon and leverage it to propose a method for improving the performance of modality-specific models in settings where both models and data are scarce. These contributions are clearly presented in the abstract and introduction.

\item {\bf Limitations}
    \item[] Question: Does the paper discuss the limitations of the work performed by the authors?
    \item[] Answer: \answerYes{} 
    \item[] Justification: We've included limitations in Sec.~\ref{sec:limitation}.

\item {\bf Theory assumptions and proofs}
    \item[] Question: For each theoretical result, does the paper provide the full set of assumptions and a complete (and correct) proof?
    \item[] Answer: \answerYes{} 
    \item[] Justification: We present Theorem~\ref{thm:lower_bound} and provide detailed proofs of the theorem and its accompanying propositions in Appendix Sec.~\ref{sec:append_proof}. We also explicitly state the Gaussian assumption required for the theorem.

    \item {\bf Experimental result reproducibility}
    \item[] Question: Does the paper fully disclose all the information needed to reproduce the main experimental results of the paper to the extent that it affects the main claims and/or conclusions of the paper (regardless of whether the code and data are provided or not)?
    \item[] Answer: \answerYes{} 
    \item[] Justification: We describe the experimental settings and implementation details in the main manuscript Sec.~\ref{sec:experimental_settings}. To ensure fair comparison, we report the hyperparameter values used for the baselines in Sec.~\ref{sec:quanti}. Additional details are provided in Appendix Sec.~\ref{sec:append_details}.

\item {\bf Open access to data and code}
    \item[] Question: Does the paper provide open access to the data and code, with sufficient instructions to faithfully reproduce the main experimental results, as described in supplemental material?
    \item[] Answer: \answerYes{} 
    \item[] Justification: All datasets used for our experiments and evaluation are open-source and publicly accessible, and we will release our code upon acceptance. In addition, reviewers are provided with instructions to reproduce the main experimental results in the \texttt{README.md} file included with the supplementary code.

\item {\bf Experimental setting/details}
    \item[] Question: Does the paper specify all the training and test details (e.g., data splits, hyperparameters, how they were chosen, type of optimizer) necessary to understand the results?
    \item[] Answer: \answerYes{} 
    \item[] Justification: We provide detailed implementation and experimental details, including hyperparameters, the models used, and the analysis procedures, in the main manuscript Sec.~\ref{sec:experimental_settings}, Appendix Sec.~\ref{sec:append_implementation}, and Sec.~\ref{sec:append_details}.

\item {\bf Experiment statistical significance}
    \item[] Question: Does the paper report error bars suitably and correctly defined or other appropriate information about the statistical significance of the experiments?
    \item[] Answer: \answerYes{} 
    \item[] Justification: We report Pearson and Spearman correlation coefficients with their corresponding $p$-values in Fig.~\ref{fig:correlation}. We also include standard deviations in Fig.~\ref{fig:layerwise_alignment} and Fig.~\ref{fig:bef_aft_DCAT}. In Appendix Sec.~\ref{sec:append_statistical}, we conduct statistical significance tests for performance changes across all results presented in the main manuscript and report the complete results.

\item {\bf Experiments compute resources}
    \item[] Question: For each experiment, does the paper provide sufficient information on the computer resources (type of compute workers, memory, time of execution) needed to reproduce the experiments?
    \item[] Answer: \answerYes{} 
    \item[] Justification: We report the compute resources used in our experiments in Appendix Sec.~\ref{sec:append_models_and_computation}. The computational costs of our method and the baselines are provided in Appendix Sec.~\ref{sec:append_efficiency}.
    
\item {\bf Code of ethics}
    \item[] Question: Does the research conducted in the paper conform, in every respect, with the NeurIPS Code of Ethics \url{https://neurips.cc/public/EthicsGuidelines}?
    \item[] Answer: \answerYes{} 
    \item[] Justification: The authors have reviewed the NeurIPS Code of Ethics.

\item {\bf Broader impacts}
    \item[] Question: Does the paper discuss both potential positive societal impacts and negative societal impacts of the work performed?
    \item[] Answer: \answerYes{} 
    \item[] Justification: We provide broader impacts in Appendix Sec.~\ref{sec:append_broader_impacts}.
    
\item {\bf Safeguards}
    \item[] Question: Does the paper describe safeguards that have been put in place for responsible release of data or models that have a high risk for misuse (e.g., pre-trained language models, image generators, or scraped datasets)?
    \item[] Answer: \answerNA{} 
    \item[] Justification: Our experiments are conducted using existing pretrained models and do not aim to introduce new capabilities. Therefore, safeguard design is outside the scope of this work, and we rely on the safeguards already incorporated in the original models.

\item {\bf Licenses for existing assets}
    \item[] Question: Are the creators or original owners of assets (e.g., code, data, models), used in the paper, properly credited and are the license and terms of use explicitly mentioned and properly respected?
    \item[] Answer: \answerYes{} 
    \item[] Justification: We appropriately cite the papers from which the models were obtained, provide citations for all baselines, and credit in-the-wild models by including direct links to their sources.

\item {\bf New assets}
    \item[] Question: Are new assets introduced in the paper well documented and is the documentation provided alongside the assets?
    \item[] Answer: \answerYes{} 
    \item[] Justification: We provide details of the provided code in supplementary \texttt{README.md}.

\item {\bf Crowdsourcing and research with human subjects}
    \item[] Question: For crowdsourcing experiments and research with human subjects, does the paper include the full text of instructions given to participants and screenshots, if applicable, as well as details about compensation (if any)? 
    \item[] Answer: \answerNA{} 
    \item[] Justification: The paper does not involve crowdsourcing nor research with human subjects.

\item {\bf Institutional review board (IRB) approvals or equivalent for research with human subjects}
    \item[] Question: Does the paper describe potential risks incurred by study participants, whether such risks were disclosed to the subjects, and whether Institutional Review Board (IRB) approvals (or an equivalent approval/review based on the requirements of your country or institution) were obtained?
    \item[] Answer: \answerNA{} 
    \item[] Justification: The paper does not involve crowdsourcing nor research with human subjects.

\item {\bf Declaration of LLM usage}
    \item[] Question: Does the paper describe the usage of LLMs if it is an important, original, or non-standard component of the core methods in this research? Note that if the LLM is used only for writing, editing, or formatting purposes and does \emph{not} impact the core methodology, scientific rigor, or originality of the research, declaration is not required.
    \item[] Answer: \answerYes{} 
    \item[] Justification: Our experiments include some benchmarks that use an LLM-as-a-judge and parsing.

\end{enumerate}